\documentclass{article}

 \usepackage[preprint]{neurips_2026}

\usepackage[utf8]{inputenc} 
\usepackage[T1]{fontenc}    
\usepackage{hyperref}       
\usepackage{url}            
\usepackage{booktabs}       
\usepackage{amsfonts}       
\usepackage{amssymb}        
\usepackage{amsmath}        

\usepackage{amsthm}         

\usepackage{nicefrac}       
\usepackage{microtype}      
\usepackage{xcolor}         
\usepackage{colortbl}  
\usepackage{graphicx}       
\usepackage{subcaption}     
\usepackage{multirow}       
\usepackage{enumitem}       
\usepackage{algorithm}      
\usepackage{algorithmic}    
\usepackage{wrapfig}
\usepackage{graphicx}
\usepackage{booktabs}
\usepackage{xspace}

\newcommand{\method}{\textsc{Moir}\xspace}

\newcommand{\bos}{$\langle\texttt{bos}\rangle$}
\newcommand{\randone}{$\langle\texttt{rand}\rangle\times 1$}
\newcommand{\randtwo}{$\langle\texttt{rand}\rangle\times 2$}
\newcommand{\randthree}{$\langle\texttt{rand}\rangle\times 3$}
\newcommand{\randfour}{$\langle\texttt{rand}\rangle\times 4$}

\usepackage{mathtools}

\newtheorem{theorem}{Theorem}
\newtheorem{proposition}[theorem]{Proposition}

\newtheorem{lemma}[theorem]{Lemma}

\newtheorem{remark}{Remark}

\graphicspath{{figures/}}

\title{Moir: Let the Model Direct Its Own Story for Robust Cross-Domain Knowledge Editing}

\author{
  Jea Kwon, \quad Jiwon Kim, \quad Dong-Kyum Kim, \quad Meeyoung Cha \\
  Max Planck Institute for Security and Privacy (MPI-SP) \\
}
\begin{document}

\maketitle

\begin{abstract}
While language models remain frozen at their training state, the world evolves continuously. Knowledge editing has emerged as a key alternative to full retraining, but its deployment is bottlenecked by the erosion of core capabilities: mathematical and programmatic reasoning collapse while encyclopedic recall remains intact. We trace this asymmetric degradation to a distributional mismatch. Covariance-based editors preserve only the subspaces spanned by their reference corpus, but fail to capture the operative distribution shaped by post-training such as SFT and DPO. Static external corpora, including Wikipedia and even the original pretraining mixture, cannot recover this shifted manifold. We propose \method, which estimates the preservation covariance $C$ directly from the model itself by sampling from its own decoding distribution. Seeding generation with a single random vocabulary token bypasses the instruction-following templates that otherwise dominate sampled outputs, exposing the broader subspaces the model has internalized. \method requires no external data and serves as a drop-in component for any covariance-based editor, a practical advantage given that the pre- and post-training corpora of most modern LLMs are not publicly accessible. Across OLMo-2, Llama-3.1, and Qwen-3 (7--8B), under both MEMIT and AlphaEdit and in batch and sequential regimes, \method consistently extends preservation in the most vulnerable domains, most strikingly on Qwen3-8B after 20{,}000 AlphaEdit batch edits, it retains 79.9\% GSM8K accuracy compared to 10.9\% with the Wikipedia baseline. These results suggest that aligning the preservation distribution with the model's operative distribution is a key factor in non-destructive editing, and that the model itself may be the most accessible source of that distribution for deployed systems.
\end{abstract}

\section{Introduction}
\label{sec:introduction}

Large language models (LLMs) act as static repositories of knowledge in a fundamentally dynamic world. While global events, scientific breakthroughs, and socio-political landscapes evolve continuously, the internal state of an LLM remains frozen at the conclusion of its training. Because full retraining to incorporate new information is computationally and environmentally prohibitive, \textit{model editing} has emerged as a critical paradigm for incorporating precise, localized updates to specific facts~\citep{geva2021transformer, mitchell2022mend}.

The central challenge in model editing lies in maintaining a delicate balance: successfully integrating new information without compromising the model’s hard-earned structural integrity. Past research in this field has focused on the mechanics of the update, moving from \textit{how to edit}~\citep{meng2022locating, meng2023memit} to \textit{how to preserve} existing parameters~\citep{fang2024alphaedit}. This work introduces a new perspective. To achieve truly non-destructive editing, we show that the community must address an equally important, yet largely unexamined question: \textbf{\textit{what exactly should be preserved?}}

Current structure-preserving algorithms~\citep{fang2024alphaedit, gu2024rect} implicitly answer this question by using arbitrary proxy corpora, such as Wikipedia, to define the preserved knowledge space. In this work, we demonstrate that this deeply ingrained convention introduces a critical misalignment, causing a severe and disproportionate collapse of complex capabilities that lie outside this proxy’s narrow distribution (e.g., mathematical reasoning and code generation) during the editing process. These implications extend beyond editing itself. As deployed models increasingly rely on editing for factual patching, content removal, and continuous updates, any intervention that inadvertently degrades core capabilities becomes a significant barrier to deployment as reliable infrastructure



We propose a shift in how the preservation space is defined through the following contributions:
\begin{itemize}
    \vspace{-5pt}
    \item \textbf{Diagnosing the proxy pitfall.} We empirically show that arbitrary proxy corpora drive the cross-domain collapse observed in prior editing methods.
    
    \item \textbf{Self-generated manifolds.} 
    We introduce \method, a data-free framework that approximates this distribution directly from the model's own generations (see Figure \ref{fig:memoir}).
    
    \item \textbf{Inaccessibility as motivation.} 
    We demonstrate that the operative distribution is generally inaccessible in deployed LLMs. Consequently, self-generation is presented not merely as a preference but as the only viable approach for maintaining the relevant distribution.
    
    \vspace{-5pt}
\end{itemize}

We introduce an approach that utilizes a model's own decoding distribution to access its internalized training data and post-training shifts without relying on external corpora. Empirically, we identify single-random-token seeding as the configuration that best balances mode-collapse mitigation and prefix-prior fidelity. Additionally, we show that aligning the preservation space with this self-generated manifold reduces the cross-domain degradation typical of external proxies. By applying \method, we achieve unprecedented stability across major LLMs and maintain complex reasoning manifolds.


\method is a drop-in replacement for any covariance-based editor, requiring no architectural changes to existing frameworks like MEMIT or AlphaEdit~\citep{meng2023memit, fang2024alphaedit}. By sampling directly from the live model, it captures the nuanced activation shifts induced by post-training,  such as Supervised Fine-Tuning (SFT) and Direct Preference Optimization (DPO), that static corpora fail to represent. This alignment enables large-scale, non-destructive editing across the OLMo-2, Llama-3.1, and Qwen-3 families. The resulting stability is substantial: the tested models sustain 79.9\% accuracy on GSM8K after 20,000 edits, whereas Wikipedia baselines collapse to 10.9\%. This stability is achieved at a manageable one-time generation cost of approximately 2 hours per model.

The implications reach beyond knowledge editing as a research problem. Once edits no longer cost a model its reasoning or coding ability, full retraining is no longer the only path to keeping a model factually up-to-date. A deployed model can instead evolve continuously, absorbing corrections in minutes rather than months, and stay in sync with the world for as long as it runs. What has been a static artifact, frozen at the end of training, becomes something closer to maintainable infrastructure.

\begin{figure*}[t]
    \centering
    \includegraphics[width=0.95\textwidth]{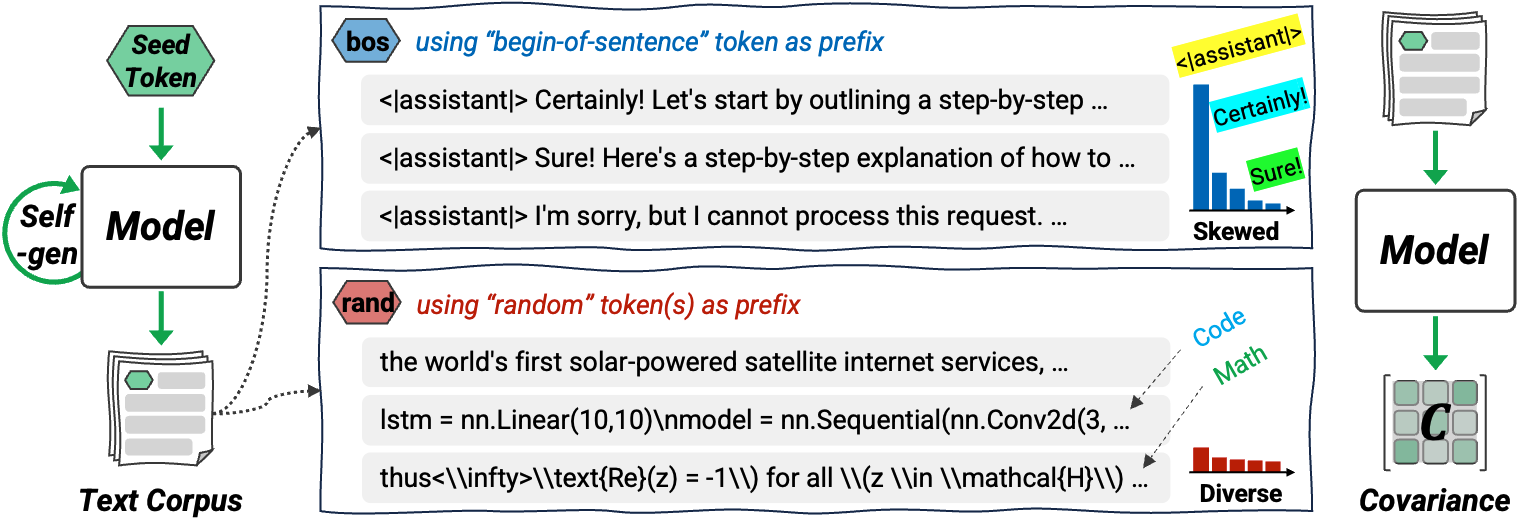}
    \caption{\textbf{\method's self-generation pipeline.} The model autoregressively generates samples from a seed token, from which $C$ is estimated. Using begin-of-sentence token (top) collapses generations skewed toward an instruction-style, whereas a random token (bottom) diversifies generations across domains (e.g., general text, code, and math) better approximating the model's internalized distribution.}
    \label{fig:memoir}
\end{figure*}

\section{Related Work}

\paragraph{The Evolution of Editing Paradigms.}
Knowledge editing methods can be grouped based on whether they preserve the model's original parameters or not. Parameter-preserving approaches \citep{hu2022lora, mitchell2022serac, hartvigsen2023grace} route edits through external adapters or memory modules. While safe by construction, they incur inference overhead, struggle with complex reasoning, and cannot truly remove information. Direct parameter-modifying methods \citep{mitchell2022mend, meng2022rome, meng2023memit} instead overwrite weights to internalize new facts. The cost is interference: representations in language models are entangled, so localized updates leak into neighboring activations and, under sequential editing, compound into catastrophic forgetting and model collapse. Recent structure-preserving editors \citep{fang2024alphaedit, ma2024prune, gu2024rect} address this by projecting updates onto subspaces deemed orthogonal to existing knowledge.

\paragraph{Out-of-Domain Capability Erosion.}
Recent stress tests expose a blind spot in editing evaluation: the standard locality metric draws held-out triples from the same distribution as the edits and thus cannot detect a collapse in capabilities lying outside it. \citet{gu2024rect} shows that a single editing pass modifying less than one percent of LLaMA-1's parameters drives out-of-domain accuracy to near zero, and \citet{yang2024should} reports that ten thousand sequential edits push the model into a ``muting effect'' that collapses broad capabilities in parallel. The damage is most severe in specialized domains: sequential MEMIT edits collapse mathematical problem-solving to zero \citep{gu2024model, lin-etal-2024-navigating}, and the same methods reduce syntactic validity on code-specialized models by up to 86 points \citep{chhetri2025understanding}. Concurrent work by \citet{liu2026covariance} formalizes this as a ``Covariance Trap,'' attributing the failure to the projection metric itself. 

\section{Rethinking the Preservation Space}
\label{sec:rethinking_space}

\subsection{Background: Covariance-Constrained Knowledge Editing}
\label{sec:background}

To establish the mathematical role of the preservation space, we first formalize the objective of sequential model editing. Given a residual matrix $R$ that encodes target values, an algorithm computes a parameter update $\Delta W$ to memorize new facts associated with keys $K_{\text{new}}$. This update must simultaneously preserve previously edited facts ($K_{\text{past}}$) and foundational pre-training knowledge ($K_0$) sampled from a proxy distribution $\mathcal{D}$. Both foundational and state-of-the-art methods rely on $K_0$ to construct their preservation spaces.

\paragraph{Covariance-Constrained Formulations.} 
Modern editing methods rely on the uncentered covariance matrix $C = K_0 K_0^\top$ to define the preservation space, albeit through slightly different implementations. MEMIT \citep{meng2023memit} incorporates $C$ as a soft regularization penalty, computing updates proportional to $(... + C)^{-1}$. Conversely, AlphaEdit \citep{fang2024alphaedit} enforces a hard constraint by projecting updates strictly into the null space of $C$. Applying SVD to $C$ yields an orthogonal projection matrix $P$ that structurally guarantees $P K_0 = \mathbf{0}$. Despite these algorithmic differences, both methods are fundamentally bottlenecked by the same geometric commitment: the distribution $\mathcal{D}$ used to sample $K_0$ and estimate $C$. 
{Please refer to Appendix~\ref{app:setup} for the full mathematical formulations of these parameter updates.}

\paragraph{From ``How'' to ``What'' to Preserve}
Recent structure-preserving model editing methods \citep{fang2024alphaedit, ma2024prune, gu2024rect} have achieved significant gains by refining \textit{how} parameter displacement is constrained. Yet, as illustrated in Figure~\ref{fig:degradation}, we observe a catastrophic erosion of specialized knowledge regardless of the preservation mechanism employed. While general linguistic performance is maintained, out-of-domain capabilities, specifically for mathematics and coding, collapse to near-zero levels far more rapidly. 

\definecolor{heswcolor}{HTML}{8856A7} 
\definecolor{wigrcolor}{HTML}{4A90E2} 
\definecolor{mmlucolor}{HTML}{D6604D} 
\definecolor{arccolor}{HTML}{43A247}  
\definecolor{hecolor}{HTML}{B3A125}   
\definecolor{gsmcolor}{HTML}{D9822B}  
\begin{table}[t]
\centering
\vspace{-15pt}
\caption{Impact of $C$ source distribution on OLMo-2-Instruct under MEMIT batch editing. Aligning $C$ with OLMo-2's actual pretraining corpus (OLMo-Mix-1124) substantially mitigates cross-domain collapse compared to the conventional Wikipedia proxy (see evaluation metrics in \ref{app:setup_metrics}). Values in parentheses denote the absolute improvement ($+\Delta$) over the Wikipedia baseline in the most vulnerable domains (Code and Math). Oracle construction details are provided in Appendix~\ref{app:oracle}. $N=2{,}000$ edits}
\label{tab:mixed_c0}
\resizebox{\textwidth}{!}{%
\begin{tabular}{l|cc|ccccc|cccccc}
\toprule
& \multicolumn{7}{c|}{\textbf{Edit Metrics}} & \multicolumn{6}{c}{\textbf{Preserve Metrics}} \\
\cmidrule(lr){2-8} \cmidrule(l){9-14}
$C$ Source & \textbf{ES} & \textbf{PS} & \textbf{Effi} & \textbf{Gen} & \textbf{Spec} & \textbf{Flu} & \textbf{Con} & \textcolor{heswcolor}{\textbf{HeSw}} & \textcolor{wigrcolor}{\textbf{WiGr}} & \textcolor{mmlucolor}{\textbf{MMLU}} & \textcolor{arccolor}{\textbf{ARC}} & \textcolor{hecolor}{\textbf{HE}} & \textcolor{gsmcolor}{\textbf{GSM8K}} \\
\midrule
\multicolumn{14}{l}{\textit{Conventional proxy}} \\
Wikipedia  & .961 & .831 & .830 & .443 & .481 & .946 & .258 & .776 & .672 & .527 & .505 & .140 & .352 \\
\midrule
\multicolumn{14}{l}{\textit{Pretraining-aligned distribution (OLMo-2)}} \\
Dolma~1.7 (Wiki subset only)       & .950 & .791 & .789 & .383 & .462 & .951 & .245 & .769 & .657 & .550 & .501 & .232 {\scriptsize (+.092)} & .378 {\scriptsize (+.026)} \\
OLMo-Mix (pretrain-oracle)        & .942 & .796 & .767 & .386 & .455 & .961 & .242 & .774 & .676 & .552 & .521 & \textbf{.238 {\scriptsize (+.098)}} & \textbf{.485 {\scriptsize (+.133)}} \\
\bottomrule
\end{tabular}%
}
\end{table}

This convention introduces a structural bias. Modern pretraining corpora such as Dolma~\citep{soldaini2024dolma} and RedPajama~\citep{together2023redpajama} are highly heterogeneous, whereas Wikipedia is restricted to encyclopedic prose. A covariance matrix estimated from such a narrow proxy spans only the activation subspaces relevant to factual recall, leaving directions critical for these specialized domains unconstrained. Consequently, editing algorithms that rely on this incomplete $C$ unintentionally distort the unprotected subspaces, leading to the rapid degradation illustrated in Figure~\ref{fig:degradation} and quantified in Table~\ref{tab:mixed_c0}.

\paragraph{Why $C$ Alone Determines the Update} 

For the closed-form covariance-constrained editors examined here, the parameterupdate depends on the input distribution only through the uncentered covariance~$C$; the output-side factor~$G$ cancels in the closed-form optimum. This is not an assumption but a consequence of the cancellation established in proposition~\ref{thm:g-cancel} (with MEMIT requiring the soft-regularization limit
per Remark 1).

\begin{wrapfigure}{r}{0.55\textwidth} 
  \begin{center}
    \vspace{-15pt} 
    \includegraphics[width=0.55\textwidth]{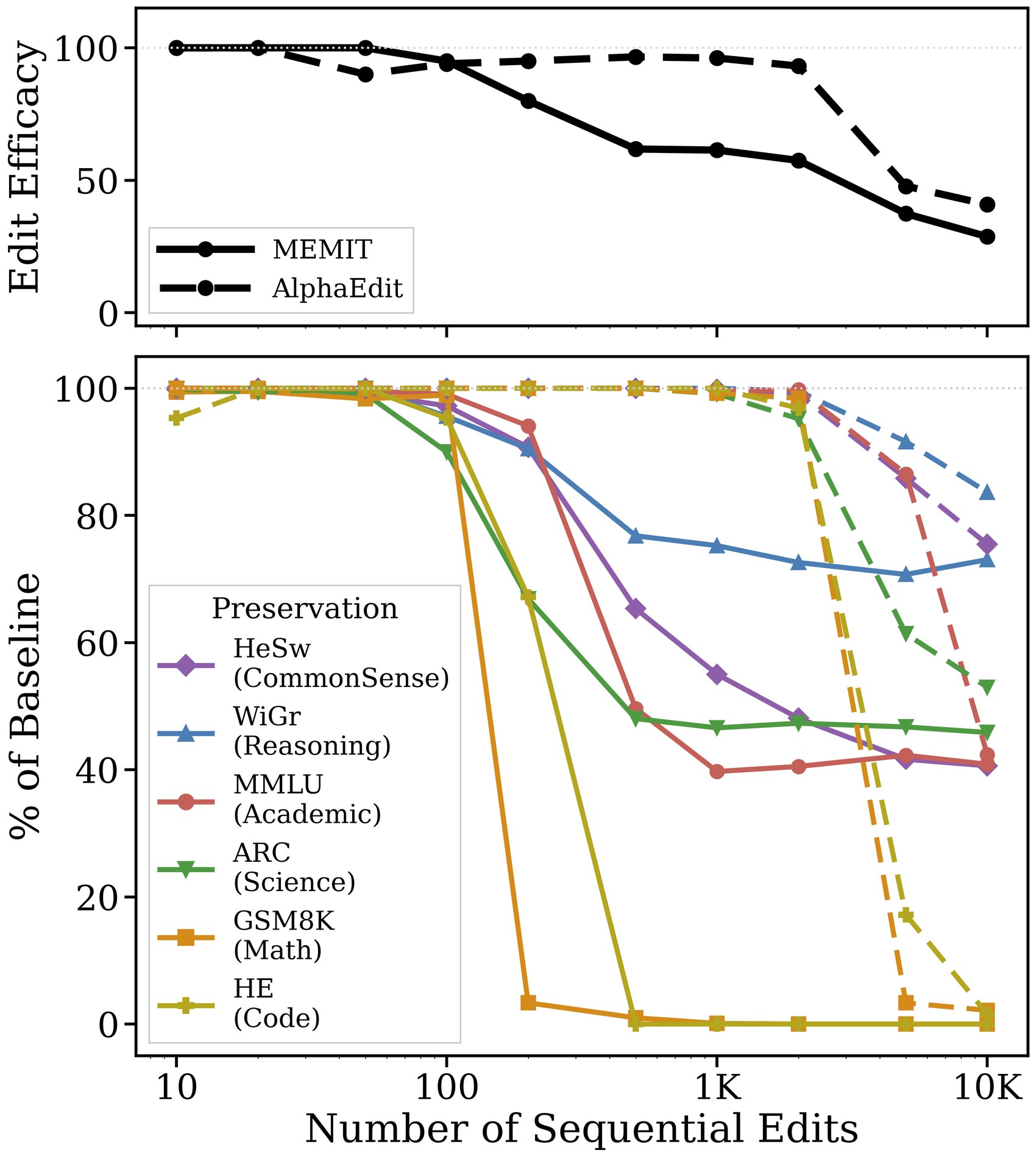}
    \vspace{-5pt} 
    \caption{Differential degradation rates across knowledge domains during sequential editing on Counter Factual data with OLMo-2-Inst.}
    \label{fig:degradation}
    \vspace{-3.5em} 
  \end{center}
\end{wrapfigure}

\paragraph{Beyond the pretraining oracle.}
A natural hypothesis is that the optimal $C$ should be derived from the model's original pre-training mixture. We test this by constructing an "Oracle" covariance using the official OLMo-Mix-1124 corpus, sampled at exact pretraining proportions (Table~\ref{tab:mixed_c0}). Relative to the Wikipedia baseline, the pretraining-aligned oracle substantially improves preservation of specialized capabilities at a modest cost to edit quality. Surprisingly, a self-generated distribution of the kind we propose is highly competitive with and frequently surpasses even this oracle (Table~\ref{tab:memit_batch}).

As shown in Figure 3, post-training procedures (e.g., SFT, DPO) shift the model's operative manifold away from its raw pretraining distribution until the data on disk no longer matches the internal representations encoded in the weights. Any static external corpus, regardless of its historical relevance, incurs an inherent geometric bias. A principled way to capture this shifted distribution is to elicit it directly from the live model via self-generation, a data-free framework we propose next.

\section{\method: Self-generated Covariance Estimation}
\label{sec:memoir}

\method (
\textbf{M}emories \textbf{O}f 
\textbf{I}nternal \textbf{R}epresentations) elicits the preservation distribution from the model itself. Rather than estimating $C$ from an external corpus, we construct it from samples produced by the model under its own decoding distribution. The estimator follows the standard formulation:
\begin{equation}
    C_{\text{\method}} = \frac{1}{N} \sum_{i=1}^{N} \mathbf{k}_i \, \mathbf{k}_i^\top,
\end{equation}
where each key $\mathbf{k}_i$ is the input to the targeted MLP layer during the forward pass over a sequence sampled from $p_\theta(\,\cdot\,\mid s)$, and $N$ is the total number of token positions across sequences (see Appendix~\ref{app:memoir_details} for the full extraction procedure). Within this framework, the primary design choice distinguishing \method from prior methods is the specification of the seed sequence $s$.

\textbf{Self-generation Strategy.} 
The seed sequence $s$ serves as the entry point into the model's latent manifold. While conventional methods rely on a static $\langle \text{bos} \rangle$ prefix, such a choice is overly deterministic and limits the generation to high-probability instruction-following paths. By setting $s$ as a single uniformly-random vocabulary token ($\langle \text{rand} \rangle \times 1$), we ensure that the resulting activations are sufficiently diverse to represent the model’s broad operative distribution while remaining grounded in the model's learned $k$-gram statistics (see Table~\ref{tab:bos_vs_rand}).


\paragraph{Prefix Length Controls Diversity vs.\ Fidelity.}
The empirical results in Table~\ref{tab:bos_vs_rand} confirm that the choice of seed $s$ is critical for representative sampling. \randone$~$achieves the highest total performance (1.093, $+0.087$ gain over \bos), whereas longer random prefixes lead to a gradual degradation toward the baseline (\randfour: 1.049). These results validate our hypothesis: while a single random token is sufficient to break the mode collapse associated with $\langle \text{bos} \rangle$, excessive randomization steers generation away from coherent linguistic structures and the model's true operative activation statistics. We adopt \randone$~$as the default in all subsequent experiments.

\begin{wrapfigure}{r}{0.45\linewidth}
  \centering
  \vspace{-0.5em}
  \includegraphics[width=\linewidth]{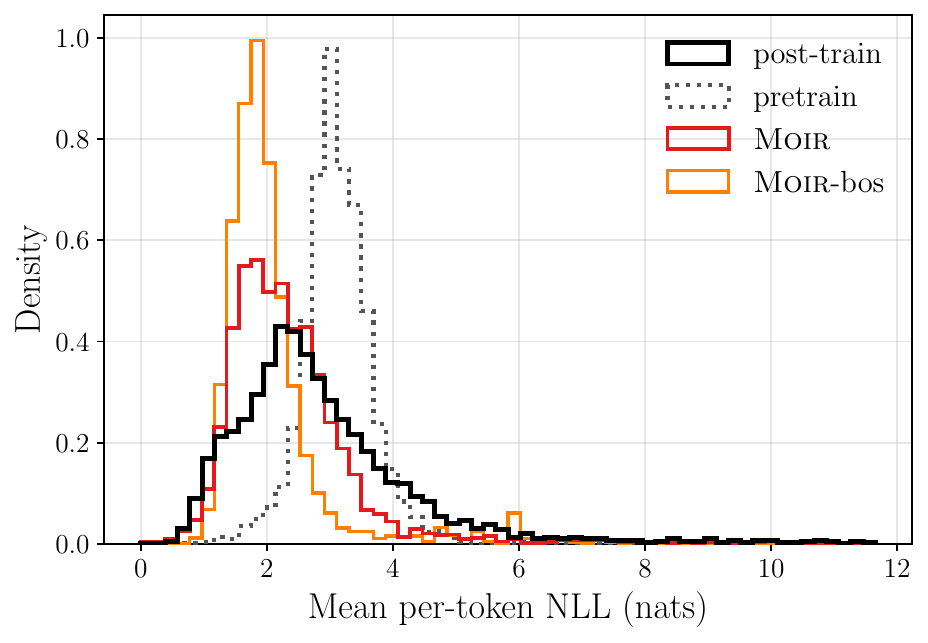}
    \caption{\textbf{Tracking the Post-Training Manifold.} Per-document NLL distributions demonstrating the clear leftward shift from the pretraining to the post-training regime. \method (\randone, red) recovers this shifted operative distribution, visibly achieving a tighter match than the \bos-only variant (orange).}
  \label{fig:nll-single}
  \vspace{-1.5em}
\end{wrapfigure}

\paragraph{Distributional Analysis.} 
To characterize the text generated by our framework, we analyze its per-document Negative Log-Likelihood (NLL) (Eq.~\ref{eq:nll})~\citep{shumailov2024ai} against the model's pre- and post-training corpora. Here, NLL quantifies how probable a document is under the model's current weights; comparing these NLL distributions allows us to verify if the generated text spans the same probability landscape as the post-training data.


Figure~\ref{fig:nll-single} reveals a stark gap between raw pretraining data (dotted black) and post-training (SFT/DPO) data (solid black). The conventional Wikipedia proxy is anchored to the pretraining distribution (Figure~\ref{fig:nll-jsd} in Appendix~\ref{app:nll_dist}), meaning it completely misses the structural shifts induced by alignment tuning. In contrast, by generating text directly from the live model, \method naturally captures this shifted post-training distribution (orange and red in Figure~\ref{fig:nll-single}). 

Crucially, our \randone-prefix strategy (\method, red) matches this target distribution significantly better than the \bos-only variant (\method-bos, orange). This is because \bos-seeding forces the model into repetitive, mode-collapsed chat templates (e.g., starting with ``Certainly!'', ``Sure!''), whereas seeding with a single random token diversifies the output to cover the model's true breadth of knowledge (see Figure~\ref{fig:token_dist} in Appendix~\ref{app:distributional_analyses}). By avoiding this mode collapse, \method provides a highly accurate, data-free representation of the deployed model's true operative distribution.



\definecolor{randA}{HTML}{F9D2CC}  
\definecolor{randB}{HTML}{FBDED9}
\definecolor{randC}{HTML}{FCE9E5}
\definecolor{randD}{HTML}{FEF4F2}  
\definecolor{bosBG}{HTML}{DEEBF5}

\begin{table}[t]
\centering
\small
\vspace{-10pt}
\caption{Effect of self-generation $C$ quality by different seed token. Four random-prefix variants (\randone–\randfour) are compared against (\bos). Both overall edit score ($S_e$) and preservation scores ($S_p$) are measured by harmonic mean of metrics (see Appendix~\ref{app:setup_metrics} for equations) across $N \in \{1\text{K}, 2\text{K}, 3\text{K}, 5\text{K}\}$ edits; Total $=$ Edit Avg. $+$ Pres. Avg.; OLMo-2-7B-Inst.}
\label{tab:bos_vs_rand}

\resizebox{0.9\textwidth}{!}{%
\begin{tabular}{c|cccc|c|cccc|c|c}
\toprule
\multirow{2}{*}{\textbf{Seed Token}} & \multicolumn{5}{c|}{\textbf{Edit HM ($S_e$)}} & \multicolumn{5}{c|}{\textbf{Pres. HM ($S_p$)}} & \multirow{2}{*}{\textbf{Total}} \\
\cmidrule(lr){2-6} \cmidrule(l){7-11}
 & 1K & 2K & 3K & 5K & Avg & 1K & 2K & 3K & 5K & Avg & \\
\midrule
\rowcolor{randA} \randone & .564 & .508 & .473 & .399 & \underline{.486} & .908 & .810 & .555 & .155 & \textbf{.607} & \textbf{1.093} \\
\rowcolor{randB} \randtwo & .563 & .531 & .501 & .404 & \textbf{.500} & .930 & .735 & .571 & .077 & \underline{.578} & \underline{1.078} \\
\rowcolor{randC} \randthree & .570 & .506 & .483 & .379 & .484 & .942 & .791 & .466 & .100 & .575 & 1.059 \\
\rowcolor{randD} \randfour & .568 & .513 & .481 & .373 & .484 & .931 & .801 & .511 & .019 & .565 & 1.049 \\
\rowcolor{bosBG} \bos  & .550 & .495 & .443 & .360 & .462 & .932 & .710 & .492 & .041 & .544 & 1.006 \\
\bottomrule
\end{tabular}%
}
\vspace{-15pt}
\end{table}

\section{A Closer Look at the Preservation Distribution}
\label{sec:theory}

\subsection{A K-FAC View of Closed-Form Editors}
\label{ssec:design-lever}

\paragraph{Why a K-FAC view.} The preservation matrix $C = K_0 K_0^\top$ that MEMIT and AlphaEdit treat as a regularizer or projector is, structurally, the empirical estimator of the K-FAC input factor $\Sigma_\ell(p) = \mathbb{E}_{x \sim p}[\phi_\ell(x)\phi_\ell(x)^\top]$ under whatever distribution $p$ supplies the samples $K_0$. Reading $C$ this way reframes covariance-constrained editing as an input-side projection under the Fisher metric induced by $p$: the editor commits to a Riemannian geometry on the activation space, and the choice of $p$ is the choice of metric. Two consequences fall out of this reading and motivate Proposition~\ref{thm:g-cancel} below. First, because the constraint $\Delta W K_\text{new} = R$ leaves no output-side degrees of freedom, the output factor $G_\ell(p)$ cannot affect the closed-form solution---the metric on the output side is inert. Second, the input factor $\Sigma_\ell(p)$ is the \emph{only} distributional object the editor sees.

Let $\phi_\ell : \mathcal{X} \to \mathbb{R}^{d_\ell}$ denote the activation map to the targeted MLP at layer $\ell$, and let $\Sigma_\ell(p) := \mathbb{E}_{x \sim p}[\phi_\ell(x)\phi_\ell(x)^\top]$ and $G_\ell(p) := \mathbb{E}_{x \sim p}[g_\ell(x)g_\ell(x)^\top]$ denote the K-FAC input and output factors, where $g_\ell$ is the output-side gradient.

\begin{proposition}[G-independence of closed-form covariance-constrained editing]
\label{thm:g-cancel}
Let $A \in \mathbb{R}^{d_{\mathrm{in}} \times d_{\mathrm{in}}}$ be positive definite and $G \in \mathbb{R}^{d_{\mathrm{out}} \times d_{\mathrm{out}}}$ be positive semi-definite, and let $K_{\mathrm{new}} \in \mathbb{R}^{d_{\mathrm{in}} \times m}$, $R \in \mathbb{R}^{d_{\mathrm{out}} \times m}$. The constrained minimum-norm problem
\begin{equation}
\label{eq:rome-objective}
\min_{\Delta W} \; \mathrm{vec}(\Delta W)^\top (A \otimes G)\, \mathrm{vec}(\Delta W) \quad \text{s.t.} \quad \Delta W\, K_{\mathrm{new}} = R
\end{equation}
admits the unique optimum $\Delta W^\star = R\, (K_{\mathrm{new}}^\top A^{-1} K_{\mathrm{new}})^{-1}\, K_{\mathrm{new}}^\top A^{-1}$, which is independent of $G$. The soft-penalty (MEMIT) and null-space-projection (AlphaEdit) formulations of Appendix~\ref{app:theory} adopt the same input-only design by construction: in all three cases, $\Delta W$ depends on the input distribution only through $A_\ell = \Sigma_\ell(p)$.
\end{proposition}

\begin{remark}[Specialization of MEMIT and AlphaEdit to the canonical form]
\label{rem:reduction}
MEMIT and AlphaEdit reduce to the canonical form of proposition~\ref{thm:g-cancel} through two distinct mechanisms, each preserving $A$-only dependence.

\textit{(i) MEMIT (penalty $\to$ constraint).} Applying the matrix push-through identity to the Tikhonov-regularized update (Eq.~\ref{eq:memit-update}) and taking $\lambda \to 0^{+}$ on the well-posed branch ($K_{\mathrm{new}}^{\top} C^{-1} K_{\mathrm{new}} \succ 0$, $K_{\mathrm{past}} = 0$) yields
\[
\Delta W_{\mathrm{MEMIT}} \;\longrightarrow\; R\,\bigl(K_{\mathrm{new}}^{\top} C^{-1} K_{\mathrm{new}}\bigr)^{-1} K_{\mathrm{new}}^{\top} C^{-1},
\]
the $A = C$ specialization of proposition~\ref{thm:g-cancel}. For finite $\lambda$, MEMIT is the soft analogue of the same problem; $G$ appears in neither form.

\textit{(ii) AlphaEdit (hard projection).} Replacing the soft penalty with the projector $P$ onto $\ker(C)$ enforces $\Delta W\, K_{0} = 0$ structurally, applying proposition~\ref{thm:g-cancel} with the feasible set restricted to $\mathrm{range}(P)$.

In both cases, $\Delta W$ depends on the input distribution only through $C = \Sigma_{\ell}(p)$.
\end{remark}

\noindent Across all three editors, $\Delta W$ therefore depends on the input distribution \emph{only} through $A_\ell = \Sigma_\ell(p)$.


\subsection{Which Distribution? An Information-Geometric View with Hypotheses}
\label{ssec:distributional}

Given proposition~\ref{thm:g-cancel}, the only distributional object the editor sees is the input covariance $\Sigma_\ell(p)$, which we treat as the empirical proxy for the geometry of the model's operative manifold $p_\theta$. The remaining question is which sampling distribution $q$ best matches $p_\theta$ at the level of activations.

\paragraph{H1: Static external corpora are confined to a narrow chart of the manifold.}
A pretraining proxy such as Wikipedia places nearly all of its mass on the encyclopedic chart of the model's internalized manifold; charts corresponding to code, mathematical reasoning, or instruction-following receive vanishing density. A covariance estimated from such a corpus is therefore not only narrow in \emph{support} but \emph{tilted} in mass: the directions the model actually uses for these domains are averaged in at near-zero weight, regardless of sample size. Self-generation, by drawing from $p_\theta$ directly, allocates density across whichever charts the model itself has internalized, without requiring us to identify those charts in advance.

\paragraph{H2: BOS-seeding collapses onto the post-training attractor.}
Post-training installs a strong prior at the start of generation. Conditioned on $\langle\textsc{bos}\rangle$, the model funnels into a small set of instruction-following openings, so the resulting sample distribution is sharply concentrated rather than broad. This constitutes a different failure mode from H1: where Wikipedia is biased in \emph{topic}, BOS-seeding is biased in \emph{form}, sampling the manifold only along the few trajectories that survive the chat-template attractor.

\paragraph{H3: A single random token escapes the attractor at low geometric cost.}
A uniformly drawn first token rarely sits on the natural prefix manifold, since many tokens are syntactically marginal or semantically vacuous. Yet this small off-manifold perturbation dislodges generation from the BOS basin: after one or two autoregressive steps, the trajectory re-enters the model's natural distribution, seeded into a region the BOS prior would never visit, a small fidelity cost for a large gain in exploration.

\paragraph{H4: Longer random prefixes drift into a sparse region of the product space.}
For a length-$k$ random seed, the prefix lives in $\mathcal{V}^k$, whose meaningful subset shrinks combinatorially with $k$. Most points in $\mathcal{V}^k$ are nonsense $k$-grams that the model has never seen during training, so conditioning on such a prefix steers generation into low-density regions of $p_\theta$. Unlike the single-token case, the autoregressive process cannot fully recover here, because the prefix bias propagates forward through the context. The exploration benefit thus saturates after $k=1$ while the off-manifold drift compounds, and we expect a sweet spot at $k=1$: enough randomness to break the attractor, but not enough to leave the manifold.

\section{Generality Across Editors, Models, and Regimes}
\label{sec:results}

\begin{wrapfigure}{r}{0.47\textwidth} 
  \begin{center}
    \vspace{-15pt} 
    \includegraphics[width=0.47\textwidth]{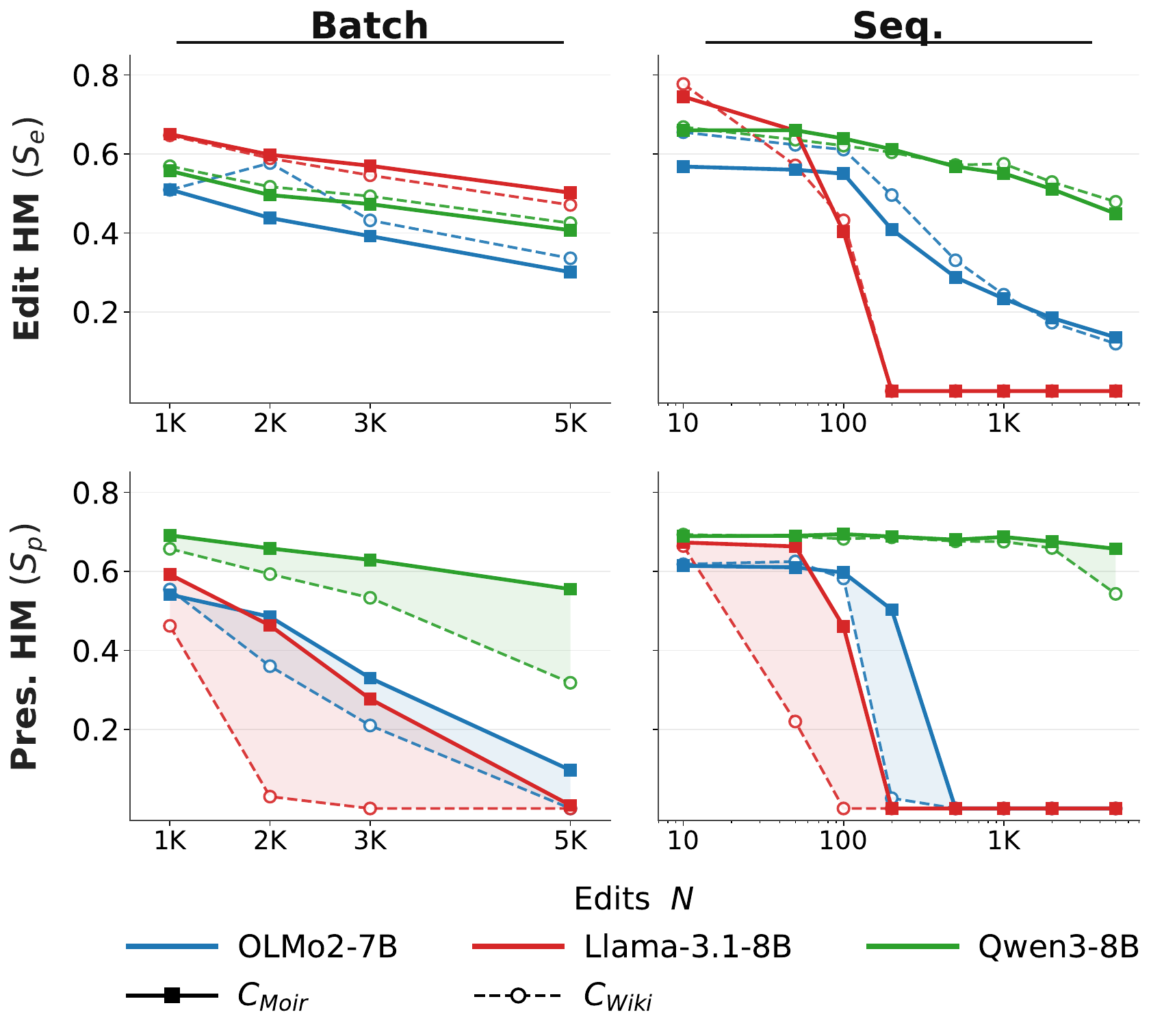}
    \vspace{-5pt} 
    \caption{$C_\text{\method}$ Improves Preservation. Harmonic means over edit metrics (top) and preservation tasks (bottom), under MEMIT batch (left) and sequential (right) editing. Shades in bottom row shows the improved preservation capacity.}
    \label{fig:hm_memit}
    \vspace{-10pt} 
  \end{center}
\end{wrapfigure}  

We evaluate \method as a drop-in replacement for $C$ in two representative editors: MEMIT~\citep{meng2023memit} (soft preservation penalty) and AlphaEdit~\citep{fang2024alphaedit} (hard null-space projection). Our evaluation covers both batch and sequential regimes on three open-weight instruction-tuned models: OLMo-2-7B-Instruct~\citep{groeneveld2024olmo}, Llama-3.1-8B-Instruct~\citep{touvron2023llama}, and Qwen3-8B~\citep{bai2023qwen}. The editing score ($S_e$) and preservation score ($S_p$) are reported as harmonic means over five standard \textsc{CounterFact} editing metrics and six downstream benchmarks (MMLU, GSM8K, HellaSwag, WinoGrande, ARC-Challenge, HumanEval). The full setup, including editing layers, hyperparameters, and metric definitions, is provided in Appendix~\ref{app:setup}.

\paragraph{Consistency across editors, models, and scales.}
Figures~\ref{fig:hm_memit} and~\ref{fig:hm_alphaedit} compare $C_{\text{\method}}$ against the standard $C_{\text{Wiki}}$ across all editor-regime combinations (MEMIT/AlphaEdit $\times$ batch/sequential) and models, as well as edit counts spanning four orders of magnitude. Two patterns hold throughout. One is that \textit{editing quality is preserved.} Across all settings (top row), $C_{\text{\method}}$ matches or exceeds $C_{\text{Wiki}}$ in editing performance. The largest gap appears under MEMIT sequential editing on Llama-3, where $C_{\text{Wiki}}$ collapses to zero within $\sim$200 edits while $C_{\text{\method}}$ sustains non-trivial editing performance throughout. Where $C_{\text{Wiki}}$ remains stable, $C_{\text{\method}}$ tracks it closely; the self-generated covariance is therefore not a trade-off against editing quality but a strict improvement on average.

Another pattern is that \textit{out-of-domain preservation is dramatically extended.} The bottom row shows the central finding. Under MEMIT batch editing, $C_{\text{Wiki}}$ degrades roughly linearly with edit count and reaches near-zero preservation at $5 \times 10^3$ edits on Llama-3, while $C_{\text{\method}}$ retains $0.3$--$0.5$ Preservation HM at the same scale. Under MEMIT sequential editing, the gap is more extreme: $C_{\text{Wiki}}$ collapses to zero within $10^2$--$10^3$ edits across all three models, while $C_{\text{\method}}$ sustains $0.5$--$0.7$ preservation throughout $10^4$ edits. AlphaEdit exhibits the same pattern in attenuated form: $C_{\text{\method}}$ maintains preservation flat across the full edit range, while $C_{\text{Wiki}}$ on Qwen-3 degrades sharply beyond $10^3$ sequential edits.

\begin{figure*}[h]
    \centering
    \includegraphics[width=\textwidth]{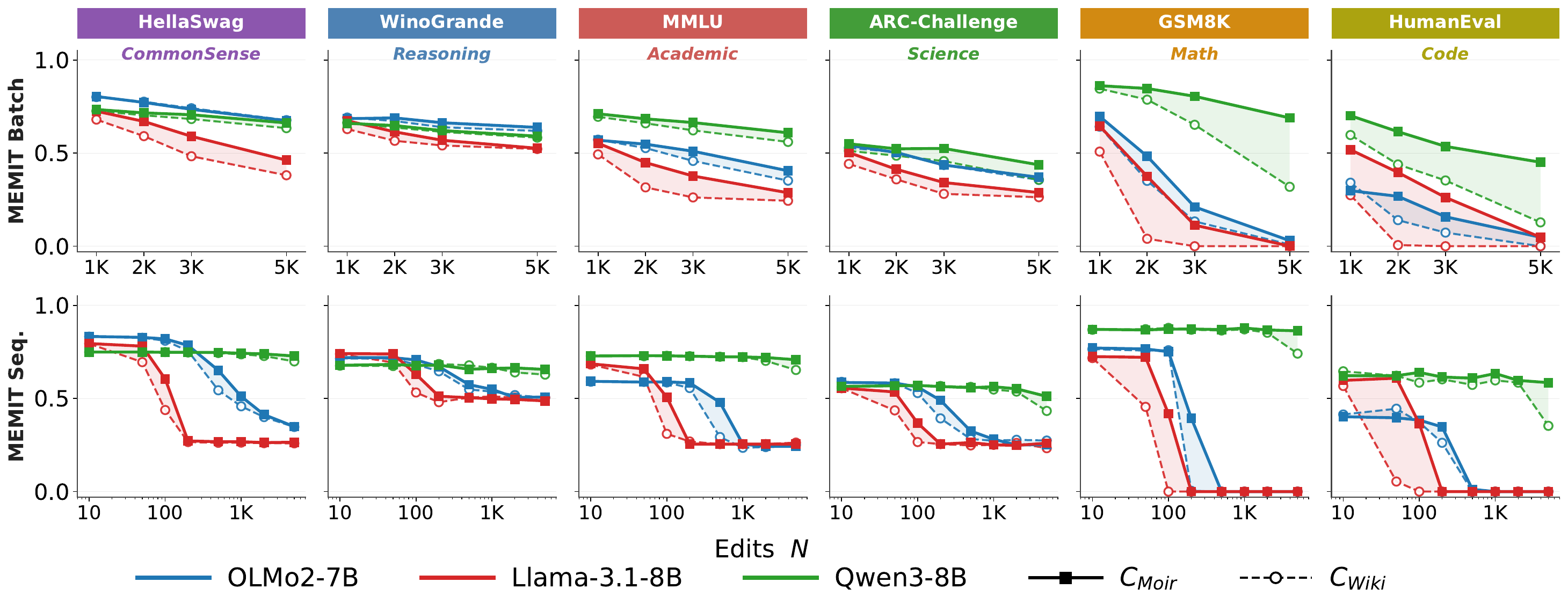}
    \caption{Per-task preservation under MEMIT. Six benchmark scores after batch and sequential editing on three instruction-tuned models. Shaded bands mark the gap. $C_\method$ consistently delays collapse, most visibly on GSM8K, HumanEval, and across the sequential cliff regime.}
    \vspace{-10pt} 
    \label{fig:pres_memit}
\end{figure*}  
\begin{figure*}[t]
    \centering
    \includegraphics[width=\textwidth]{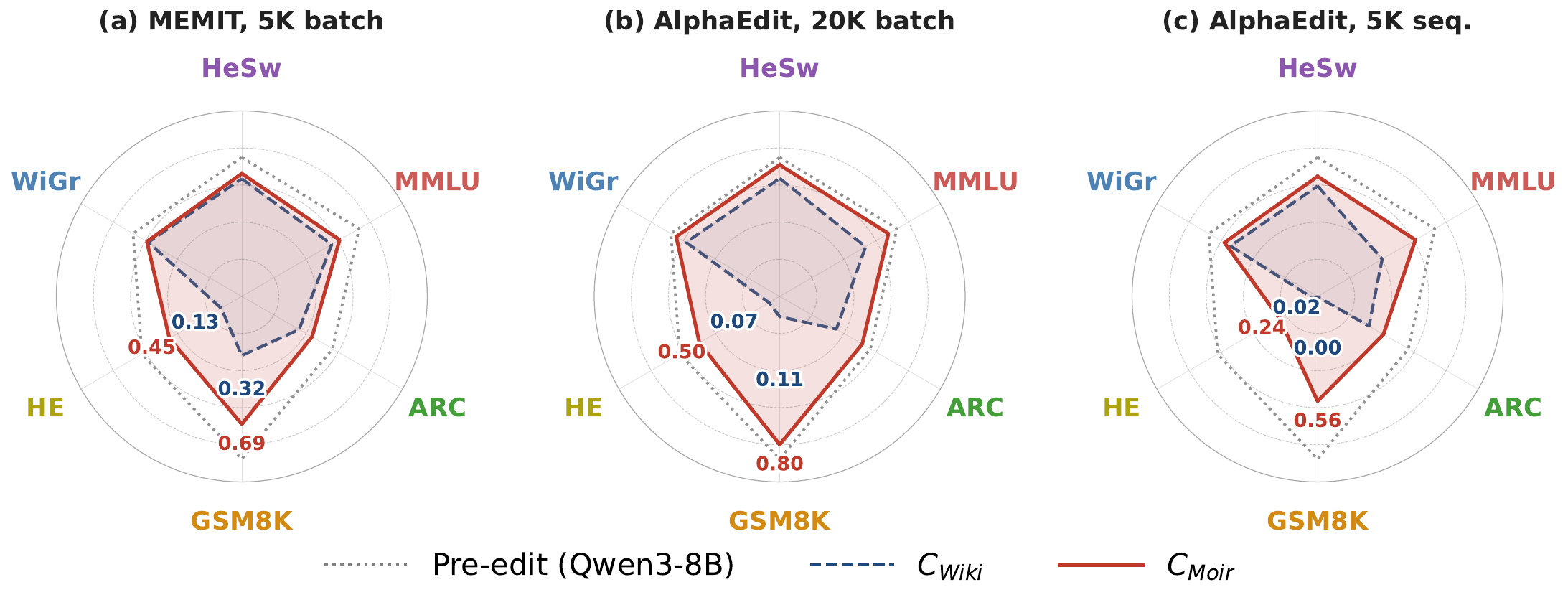}
    \caption{Representative radar chart on Qwen3-8B (analogous patterns for OLMo-2 and Llama-3.1 in Fig. 5). The collapse is asymmetric: $C_\text{Wiki}$ preserves encyclopedic axes (MMLU, WiGr, HeSw) but GSM8K/HE collapse toward zero. $C_\method$ closes this gap without disturbing the preserved axes.}
    \vspace{-15pt} 
    \label{fig:radar}
\end{figure*}  

\paragraph{Where the preservation gap lives.}
Figure~\ref{fig:pres_memit} and~\ref{fig:pres_alphaedit} shows task-wise preservation capacity (shades). Figure~\ref{fig:radar} resolves this by showing the per-benchmark preservation profile on Qwen-3-8B across three editing settings, overlaid on the pre-edit baseline. The collapse is sharply asymmetric: $C_{\text{Wiki}}$ holds encyclopedic capabilities (MMLU, WinoGrande) near baseline while visibly shrinking the polygon inward on the most distributionally distant axes: HumanEval and GSM8K. $C_{\text{\method}}$ closes this gap. GSM8K rises from $.06$ to $.78$ on AlphaEdit batch editing and from $.32$ to $.92$ on AlphaEdit sequential, with similar recovery on HumanEval, while the encyclopedic axes that $C_{\text{Wiki}}$ already preserves remain preserved. The improvement, therefore, comes from extending coverage into capability subspaces that $C_{\text{Wiki}}$ never protected, not from redistributing it.

\textbf{The Nature of Domain Collapse.} It is important to note that the near-zero performance on GSM8K and HumanEval under $C_{\text{Wiki}}$ represents a comprehensive collapse of these specific domains. Whether this degradation manifests as a loss of underlying logic or the destruction of domain-specific formatting (e.g., code syntax, step-by-step mathematical templates), the practical utility of these subspaces is entirely compromised. $C_{\text{Wiki}}$ fails to protect the structural integrity of these out-of-domain manifolds, whereas \method preserves the exact distributional footprint required to maintain them.

\paragraph{Why the gap widens with scale.}
Both editors accumulate parameter perturbations $\Delta W$ as edits proceed; the role of $C$ is to constrain how those perturbations project onto the activation space. When $C$ misses entire capability subspaces (as $C_{\text{Wiki}}$ does for code and math, per Table~\ref{tab:mixed_c0} and Figure~\ref{fig:radar}) the cumulative perturbation freely distorts those subspaces, and the distortion compounds with each edit. $C_{\text{\method}}$ samples from the model's actual internalized distribution, covering exactly the subspaces the model relies on, regardless of whether they appear in any external corpus. The benefit grows with the number of edits and the distance between the editing distribution and the protected capability.

\section{Discussion}
\label{sec:Discussion}



Our results show that this geometric mismatch determines which capabilities survive an edit and which collapse. Replacing the Wikipedia-derived $C_{\text{Wiki}}$ with $C_{\text{\method}}$, estimated from the model's own random-prefix generations, restores out-of-domain capability across all four editor-regime combinations and three model families we tested, without sacrificing editing quality and without requiring access to training data. The gap is most striking under sequential editing on specialized domains: $C_{\text{Wiki}}$ collapses GSM8K to 0.2\% on Qwen-3 within $5{\times}10^3$ AlphaEdit edits, while $C_{\text{\method}}$ preserves 56.6\%.

\begin{wrapfigure}{r}{0.5\textwidth} 
  \begin{center}
    \vspace{-10pt} 
    \includegraphics[width=0.5\textwidth]{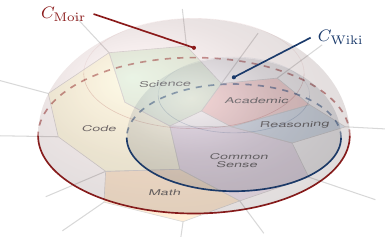}
    \vspace{-5pt} 
    \caption{Knowledge Coverage: $C_{\text{\method}}$ vs $C_{\text{Wiki}}$.}
    \label{fig:dom}
    \vspace{-10pt} 
  \end{center}
\end{wrapfigure}

NLL and JSD comparisons (Figure~3, Appendix~D.3) show that external proxies suffer from an inherent geometric bias relative to the operative distribution $p_\theta$ that persists regardless of sample size. Empirically, this led to a result we did not anticipate: $C_{\text{\method}}$ matches and sometimes exceeds the performance of an "Oracle" covariance built from the official OLMo-Mix-1124 corpus (Table~\ref{tab:mixed_c0}). We attribute this to the fact that post-training procedures like instruction tuning shift the model's operative manifold away from its raw pre-training state. In contrast, self-generation provides a dynamic map of this evolved geometry that static external corpora cannot replicate.

\paragraph{Limitations.}
Our study is bounded by three primary considerations. First, evaluation considered only two locate-then-edit families (MEMIT, AlphaEdit). Even though our distributional findings should generalize to any paradigm utilizing a covariance-like operator to define the preservation space (including memory-based or hypernetwork editors), this remains to be empirically verified. Second, the quality of \method depends on the model's generation behavior. In cases of extreme over-optimization as in extreme Reinforcement Learning from Human Feedback (RLHF), where the decoding distribution is pathologically mode-collapsed, self-generation might under-sample specialized or rare capabilities. 
Third, our experiments focus on instruction-tuned models in the 7--8B range; behavior at frontier scale (70B+) or on raw base (non-instruct) models warrants separate study. While \method consistently improves preservation, we do not claim self-generation is uniquely optimal. Stronger external mixtures may close part of the gap against Wikipedia, but the distributions most relevant to deployed models, such as post-training and usage data, are typically inaccessible. \method offers a practical, data-free approximation aligned with the model's post-training distribution.

\paragraph{Implications for editing as infrastructure.}
If knowledge editing is to serve as a robust operational layer for keeping deployed models current (enabling real-time factual patching, content removal, regulatory compliance, continuous updates), then preservation cannot be an afterthought tied to a 2020 Wikipedia dump. Our results indicate that this limitation is fixable without any change to existing editors: $C_{\text{\method}}$ requires a one-time, $\sim$2-hour generation as a drop-in component for MEMIT, AlphaEdit, or any future covariance-based editor. The broader methodological lesson is that proxies inherited from earlier work warrant re-evaluation as the field matures and models evolve. Treating this component with rigor is essential for transitioning model editing from a research demonstration into a reliable and practical discipline for model maintenance.

\clearpage

\bibliographystyle{plainnat}
\bibliography{references}

@article{meng2022locating,
  title={Locating and editing factual associations in {GPT}},
  author={Meng, Kevin and Bau, David and Andonian, Alex and Belinkov, Yonatan},
  journal={Advances in Neural Information Processing Systems (NeurIPS)},
  year={2022}
}

@inproceedings{
fang2024alphaedit,
title={{AlphaEdit}: Null-Space Constrained Model Editing for Language Models},
author={Junfeng Fang and Houcheng Jiang and Kun Wang and Yunshan Ma and Jie Shi and Xiang Wang and Xiangnan He and Tat-Seng Chua},
booktitle={International Conference on Learning Representations (ICLR)},
year={2025},
}

@article{touvron2023llama,
  title={The Llama 3 Herd of Models},
  author={Aaron Grattafiori and Abhimanyu Dubey and Abhinav Jauhri and Abhinav Pandey and Abhishek Kadian and Ahmad Al-Dahle and Aiesha Letman and Akhil Mathur and Alan Schelten and Alex Vaughan and others},
  journal={arXiv preprint arXiv:2407.21783},
  year={2024}
}

@article{bai2023qwen,
  title={Qwen technical report},
  author={Bai, Jinze and Bai, Shuai and Chu, Yunfei and others},
  journal={arXiv preprint arXiv:2309.16609},
  year={2023}
}

@article{groeneveld2024olmo,
  title={{OLMo}: Accelerating the science of language models},
  author={Groeneveld, Dirk and Beltagy, Iz and Walsh, Pete and others},
  journal={arXiv preprint arXiv:2402.00838},
  year={2024}
}

@article{soldaini2024dolma,
  title={Dolma: An open corpus of three trillion tokens for language model pretraining research},
  author={Soldaini, Luca and Kinney, Rodney and Bhagia, Akshita and others},
  journal={arXiv preprint arXiv:2402.00159},
  year={2024}
}

@misc{together2023redpajama,
  author = {{Together Computer}},
  title = {{RedPajama}-{Data}: An Open Source Recipe to Reproduce {LLaMA} training dataset},
  month = {April},
  year = {2023},
  url = {https://github.com/togethercomputer/RedPajama-Data}
}

@article{sakaguchi2020winogrande,
author = {Sakaguchi, Keisuke and Bras, Ronan Le and Bhagavatula, Chandra and Choi, Yejin},
title = {{WinoGrande}: an adversarial winograd schema challenge at scale},
year = {2021},
volume = {64},
number = {9},
journal = {Commun. ACM},
pages = {99–106},
numpages = {8}
}

@article{gu2024model,
  title={Model editing can hurt general abilities of large language models},
  author={Gu, Jia-Chen and Xu, Hao-Xiang and Ma, Jun-Yu and Lu, Pan and Ling, Zhen-Hua and Chang, Kai-Wei and Peng, Nanyun},
  journal={arXiv preprint arXiv:2401.04700},
  year={2024}
}

@inproceedings{hu2022lora,
  title     = {{LoRA}: Low-Rank Adaptation of Large Language Models},
  author    = {Hu, Edward J. and Shen, Yelong and Wallis, Phillip and 
               Allen-Zhu, Zeyuan and Li, Yuanzhi and Wang, Shean and 
               Wang, Lu and Chen, Weizhu},
  booktitle = {International Conference on Learning Representations (ICLR)},
  year      = {2022}
}

@inproceedings{mitchell2022serac,
  title     = {Memory-Based Model Editing at Scale},
  author    = {Mitchell, Eric and Lin, Charles and Bosselut, Antoine and 
               Manning, Christopher D. and Finn, Chelsea},
  booktitle = {International Conference on Machine Learning (ICML)},
  year      = {2022}
}

@inproceedings{hartvigsen2023grace,
  title     = {Aging with {GRACE}: Lifelong Model Editing with 
               Discrete Key-Value Adaptors},
  author    = {Hartvigsen, Thomas and Sankaranarayanan, Swami and 
               Palangi, Hamid and Kim, Yoon and Ghassemi, Marzyeh},
  booktitle = {Advances in Neural Information Processing Systems (NeurIPS)},
  year      = {2023}
}

@inproceedings{mitchell2022mend,
  title     = {Fast Model Editing at Scale},
  author    = {Mitchell, Eric and Lin, Charles and Bosselut, Antoine and 
               Finn, Chelsea and Manning, Christopher D.},
  booktitle = {International Conference on Learning Representations (ICLR)},
  year      = {2022}
}

@inproceedings{meng2022rome,
  title     = {Locating and Editing Factual Associations in {GPT}},
  author    = {Meng, Kevin and Bau, David and Andonian, Alex and 
               Belinkov, Yonatan},
  booktitle = {Advances in Neural Information Processing Systems (NeurIPS)},
  year      = {2022}
}

@inproceedings{meng2023memit,
  title     = {Mass-Editing Memory in a Transformer},
  author    = {Meng, Kevin and Sharma, Arnab Sen and Andonian, Alex and 
               Belinkov, Yonatan and Bau, David},
  booktitle = {International Conference on Learning Representations (ICLR)},
  year      = {2023}
}

@article{ma2024prune,
  title     = {Perturbation-Restrained Sequential Model Editing},
  author    = {Ma, Jun-Yu and Wang, Hong and Xu, Hao-Xiang and 
               Ling, Zhen-Hua and Gu, Jia-Chen},
  journal   = {arXiv preprint arXiv:2405.16821},
  year      = {2024}
}

@inproceedings{gu2024rect,
  title     = {Model Editing Harms General Abilities of Large Language 
               Models: Regularization to the Rescue},
  author    = {Gu, Jia-Chen and Xu, Hao-Xiang and Ma, Jun-Yu and Lu, Pan 
               and Ling, Zhen-Hua and Chang, Kai-Wei and Peng, Nanyun},
  booktitle = {International Conference on Machine Learning (ICML)},
  year      = {2024}
}

@inproceedings{yang2024should,
title={Should We Really Edit Language Models? On the Evaluation of Edited Language Models},
author={Qi Li and Xiang Liu and Zhenheng Tang and Peijie Dong and Zeyu Li and Xinglin Pan and Xiaowen Chu},
booktitle={Advances in Neural Information Processing Systems (NeurIPS)},
year={2024}
}

@article{chhetri2025understanding,
  title   = {Understanding Robustness of Model Editing in Code {LLMs}: An Empirical Study},
  author  = {Chhetri, Vinaik and Siddique, A. B. and Farooq, Umar},
  journal = {arXiv preprint arXiv:2511.03182},
  year    = {2025}
}

@article{liu2026covariance,
  title     = {Beyond the Covariance Trap: Unlocking Generalization in 
               Same-Subject Knowledge Editing for Large Language Models},
  author    = {Liu, Xiyu and Si, Qingyi and Liu, Zhengxiao and 
               Yang, Chenxu and Gu, Naibin and Lin, Zheng},
  journal   = {arXiv preprint arXiv:2603.15518},
  year      = {2026}
}

@inproceedings{olmo20252,
  title     = {2 {OLMo} 2 Furious},
  author    = {{OLMo Team} and Walsh, Pete and Soldaini, Luca and 
               Groeneveld, Dirk and others},
  booktitle = {Conference on Language Modeling (COLM)},
  year      = {2025},
  note      = {arXiv:2501.00656}
}

@inproceedings{li2024dclm,
  title     = {{DataComp-LM}: In Search of the Next Generation of 
               Training Sets for Language Models},
  author    = {Li, Jeffrey and Fang, Alex and Smyrnis, Georgios and 
               Ivgi, Maor and Jordan, Matt and Gadre, Samir and 
               Bansal, Hritik and Guha, Etash and Keh, Sedrick and 
               others},
  booktitle = {Advances in Neural Information Processing Systems (NeurIPS)},
  year      = {2024}
}

@article{li2023starcoder,
  title   = {{StarCoder}: May the Source Be with You!},
  author  = {Li, Raymond and Allal, Loubna Ben and Zi, Yangtian and 
             Muennighoff, Niklas and Kocetkov, Denis and Mou, Chenghao 
             and Marone, Marc and Akiki, Christopher and Li, Jia and 
             others},
  journal = {Transactions on Machine Learning Research (TMLR)},
  year    = {2023}
}

@inproceedings{azerbayev2024llemma,
  title     = {Llemma: An Open Language Model for Mathematics},
  author    = {Azerbayev, Zhangir and Schoelkopf, Hailey and Paster, 
               Keiran and Santos, Marco Dos and McAleer, Stephen and 
               Jiang, Albert Q. and Deng, Jia and Biderman, Stella 
               and Welleck, Sean},
  booktitle = {International Conference on Learning Representations (ICLR)},
  year      = {2024}
}

@misc{eval-harness,
  author       = {Gao, Leo and Tow, Jonathan and Abbasi, Baber and Biderman, Stella and Black, Sid and DiPofi, Anthony and Foster, Charles and Golding, Laurence and Hsu, Jeffrey and Le Noac'h, Alain and Li, Haonan and McDonell, Kyle and Muennighoff, Niklas and Ociepa, Chris and Phang, Jason and Reynolds, Laria and Schoelkopf, Hailey and Skowron, Aviya and Sutawika, Lintang and Tang, Eric and Thite, Anish and Wang, Ben and Wang, Kevin and Zou, Andy},
  title        = {The Language Model Evaluation Harness},
  month        = 07,
  year         = 2024,
  publisher    = {Zenodo},
  version      = {v0.4.3},
  doi          = {10.5281/zenodo.12608602},
  url          = {https://zenodo.org/records/12608602}
}

@article{shumailov2024ai,
  title={{AI} models collapse when trained on recursively generated data},
  author={Shumailov, Ilia and Shumaylov, Zakhar and Zhao, Yiren and Papernot, Nicolas and Anderson, Ross and Gal, Yarin},
  journal={Nature},
  volume={631},
  number={8022},
  pages={755--759},
  year={2024}
}

@inproceedings{geva2021transformer,
  title={Transformer feed-forward layers are key-value memories},
  author={Geva, Mor and Schuster, Roei and Berant, Jonathan and Levy, Omer},
  booktitle={Empirical Methods in Natural Language Processing (EMNLP)},
  year={2021}
}

@article{wang2023easyedit,
  title={Easyedit: An easy-to-use knowledge editing framework for large language models},
  author={Peng Wang and Ningyu Zhang and Bozhong Tian and Zekun Xi and Yunzhi Yao and Ziwen Xu and Mengru Wang and Shengyu Mao and Xiaohan Wang and Siyuan Cheng and Kangwei Liu and Yuansheng Ni and Guozhou Zheng and Huajun Chen},
  journal={arXiv preprint arXiv:2308.07269},
  year={2023}
}

@article{zhang2024knowedit,
  title={A Comprehensive Study of Knowledge Editing for Large Language Models},
  author={Ningyu Zhang and Yunzhi Yao and Bozhong Tian and Peng Wang and Shumin Deng and Mengru Wang and Zekun Xi and Shengyu Mao and Jintian Zhang and Yuansheng Ni and Siyuan Cheng and Ziwen Xu and Xin Xu and Jia-Chen Gu and Yong Jiang and Pengjun Xie and Fei Huang and Lei Liang and Zhiqiang Zhang and Xiaowei Zhu and Jun Zhou and Huajun Chen},
  journal={arXiv preprint arXiv:2401.01286},
  year={2024}
}

@inproceedings{lin-etal-2024-navigating,
    title = "Navigating the Dual Facets: A Comprehensive Evaluation of Sequential Memory Editing in Large Language Models",
    author = "Lin, Zihao  and
      Beigi, Mohammad  and
      Li, Hongxuan  and
      Zhou, Yufan  and
      Zhang, Yuxiang  and
      Wang, Qifan  and
      Yin, Wenpeng  and
      Huang, Lifu",
    booktitle = "Association for Computational Linguistics (ACL)",
    year = "2024",
}

@inproceedings{hendrycks2021measuring,
  title={Measuring Massive Multitask Language Understanding},
  author={Dan Hendrycks and Collin Burns and Steven Basart and Andy Zou and Mantas Mazeika and Dawn Song and Jacob Steinhardt},
  booktitle={International Conference on Learning Representations (ICLR)},
  year={2021}
}

@article{cobbe2021training,
  title={Training Verifiers to Solve Math Word Problems}, 
  author={Karl Cobbe and Vineet Kosaraju and Mohammad Bavarian and Mark Chen and Heewoo Jun and Lukasz Kaiser and Matthias Plappert and Jerry Tworek and Jacob Hilton and Reiichiro Nakano and Christopher Hesse and John Schulman},
  journal={arXiv preprint arXiv:2110.14168},
  year={2021}
}

@inproceedings{zellers2019hellaswag,
  title={{HellaSwag}: Can a Machine Really Finish Your Sentence?},
  author={Zellers, Rowan and Holtzman, Ari and Bisk, Yonatan and Farhadi, Ali and Choi, Yejin},
  booktitle={Association for Computational Linguistics (ACL)},
  year={2019}
}

@article{clark2018think,
  title={Think you have Solved Question Answering? Try {ARC}, the {AI2} Reasoning Challenge},
  author={Clark, Peter and Cowhey, Isaac and Etzioni, Oren and Khot, Tushar and Sabharwal, Ashish and Schoenick, Carissa and Tafjord, Oyvind},
  journal={arXiv preprint arXiv:1803.05457},
  year={2018}
}

@article{chen2021evaluating,
  title={Evaluating Large Language Models Trained on Code}, 
  author={Mark Chen and Jerry Tworek and Heewoo Jun and Qiming Yuan and Henrique Ponde de Oliveira Pinto and Jared Kaplan and Harri Edwards and Yuri Burda and Nicholas Joseph and Greg Brockman and Alex Ray and Raul Puri and Gretchen Krueger and Michael Petrov and Heidy Khlaaf and Girish Sastry and Pamela Mishkin and Brooke Chan and Scott Gray and Nick Ryder and Mikhail Pavlov and Alethea Power and Lukasz Kaiser and Mohammad Bavarian and Clemens Winter and Philippe Tillet and Felipe Petroski Such and Dave Cummings and Matthias Plappert and Fotios Chantzis and Elizabeth Barnes and Ariel Herbert-Voss and William Hebgen Guss and Alex Nichol and Alex Paino and Nikolas Tezak and Jie Tang and Igor Babuschkin and Suchir Balaji and Shantanu Jain and William Saunders and Christopher Hesse and Andrew N. Carr and Jan Leike and Josh Achiam and Vedant Misra and Evan Morikawa and Alec Radford and Matthew Knight and Miles Brundage and Mira Murati and Katie Mayer and Peter Welinder and Bob McGrew and Dario Amodei and Sam McCandlish and Ilya Sutskever and Wojciech Zaremba},
  journal={arXiv preprint arXiv:2107.03374},
  year={2021}
}


\appendix

\newpage

\section*{Appendix Overview}
The supplementary material is organized as follows:
\begin{itemize}
    \item \textbf{Appendix A (Extended Methodology):} Details the \method operational algorithm (A.1) and the construction of the OLMo-Mix-1124 Oracle (A.2).
    \item \textbf{Appendix B (Theoretical Definitions and Proofs):} Provides the full parameter update formulations, K-FAC setup, and proofs for Proposition 1.
    \item \textbf{Appendix C (Experimental Setup and Baselines):} Details models, editors, hyperparameters, evaluation metrics, causal tracing, and pre-edit baselines.
    \item \textbf{Appendix D (Additional Distributional Analyses):} Contains the self-generated token distributions and per-document NLL analyses.
    \item \textbf{Appendix E (Full Tabular Results):} Presents comprehensive tabular results for all editors, models, and regimes.
\end{itemize}

\clearpage
\section{Extended Methodology}
\label{app:methodology}

\subsection{\method Operational Details}
\label{app:memoir_details}

\paragraph{Notation.}
Let $\phi_\ell(\cdot)$ denote the function that maps an input token (in context) to its 
\emph{key} $\mathbf{k} \in \mathbb{R}^d$ at layer $\ell$, defined as 
the input to the second MLP linear projection $W_{\text{out}}^\ell$ 
in the targeted transformer block---the same activation extracted by 
ROME and MEMIT for $C$ estimation. The sample estimator we compute 
matches the form used by prior work,
\begin{equation}
    \hat{C}_{\text{\method}}^\ell 
    = \frac{1}{N} \sum_{i=1}^{N} \mathbf{k}_i^\ell \, (\mathbf{k}_i^\ell)^\top,
    \qquad \mathbf{k}_i^\ell = \phi_\ell(x_i),
    \label{eq:memoir_estimator}
\end{equation}
and differs only in how the keys $\{\mathbf{k}_i^\ell\}_{i=1}^N$ 
are obtained.

\paragraph{Key insight.}
An autoregressive language model $p_\theta$ defines an implicit 
distribution over token sequences shaped by its full training history. 
Sampling tokens from this distribution and extracting their keys 
yields an empirical estimator whose population limit approximates 
the key covariance under the true training distribution:
\begin{equation}
    \mathbb{E}_{x \sim p_\theta(\,\cdot\,\mid s)} 
    \left[ \phi_\ell(x) \, \phi_\ell(x)^\top \right]
    \;\approx\;
    \mathbb{E}_{x \sim \mathcal{D}_{\text{train}}} 
    \left[ \phi_\ell(x) \, \phi_\ell(x)^\top \right],
    \label{eq:memoir_hypothesis}
\end{equation}
where $s$ is a short seed prefix from which generation proceeds. 
The approximation is exact when $p_\theta(\,\cdot\,\mid s) = 
\mathcal{D}_{\text{train}}$ marginally over $s$; in practice, the 
choice of $s$ controls how broadly the model's internalized 
distribution is sampled (see Section 4 and Table 2).

\paragraph{Self-generation protocol.}
Given a target model $p_\theta$, layer set $\mathcal{L}$, total token 
budget $N$, and per-sequence length $S$, \method estimates 
$\{\hat{C}_{\text{\method}}^\ell\}_{\ell \in \mathcal{L}}$ via 
Algorithm~\ref{alg:memoir}.

\begin{algorithm}[h]
\caption{\method: Self-Generated Covariance Estimation}
\label{alg:memoir}
\begin{algorithmic}[1]
\REQUIRE Model $p_\theta$, target layers $\mathcal{L}$, token budget $N$, 
         max sequence length $S$, vocabulary $\mathcal{V}$
\ENSURE Sample covariances $\{\hat{C}_{\text{\method}}^\ell\}_{\ell \in \mathcal{L}}$
\STATE $\hat{C}^\ell \leftarrow \mathbf{0}_{d \times d}$ for each $\ell \in \mathcal{L}$;
       \quad $n \leftarrow 0$
\WHILE{$n < N$}
    \STATE Sample seed token $z \sim \text{Uniform}(\mathcal{V})$ 
           \hfill \COMMENT{$\texttt{rand}\times 1$}
    \STATE Generate $x_{1:T} \sim p_\theta(\,\cdot\,\mid z)$ with 
           $T \leq S$ \hfill \COMMENT{Autoregressive sampling}
    \FOR{each $\ell \in \mathcal{L}$}
        \STATE Extract keys $\mathbf{k}_t^\ell = \phi_\ell(x_t)$ 
               for $t = 1, \dots, T$ via forward pass
        \STATE $\hat{C}^\ell \leftarrow \hat{C}^\ell 
               + \sum_{t=1}^{T} \mathbf{k}_t^\ell \, (\mathbf{k}_t^\ell)^\top$
    \ENDFOR
    \STATE $n \leftarrow n + T$
\ENDWHILE
\STATE $\hat{C}_{\text{\method}}^\ell \leftarrow \hat{C}^\ell / n$ for each 
       $\ell \in \mathcal{L}$
\RETURN $\{\hat{C}_{\text{\method}}^\ell\}_{\ell \in \mathcal{L}}$
\end{algorithmic}
\end{algorithm}

\paragraph{Design choices.}
We initialize generation from a single uniformly random vocabulary token 
rather than from a structured prompt or the BOS token alone. As shown 
in Section 4, this minimal perturbation breaks the 
mode collapse induced by post-training (where BOS-conditioned generation 
concentrates on instruction-following openings) without steering 
generation away from the model's natural activation statistics. The 
resulting key set $\{\mathbf{k}_t^\ell\}$ is collected directly from 
the forward pass during generation, so \method adds no inference cost 
beyond the generation itself.

\subsection{OLMo-Mix-1124 Oracle Construction}
\label{app:oracle}

To test whether distributional alignment between $C$ and the model's pretraining mixture can recover the out-of-domain capabilities lost under conventional Wikipedia-based estimation, we construct an oracle covariance that mirrors the official OLMo-Mix-1124 corpus used to pretrain OLMo-2~\citep{olmo20252}. OLMo-Mix-1124 consists of five components spanning web text, code, scientific writing, mathematical content, and encyclopedic prose, in proportions chosen by the OLMo team to balance generalist coverage with domain-specific capability.

\paragraph{Component-wise covariance estimation.}
Rather than sampling from the joint mixture directly, we estimate a separate uncentered covariance matrix $C^{(i)}$ for each of the five components by drawing 100K tokens from each and computing
\begin{equation}
    C^{(i)} = \frac{1}{n_i}\sum_{k=1}^{n_i} \mathbf{k}_k^{(i)} {\mathbf{k}_k^{(i)}}^\top,
\end{equation}
where $\mathbf{k}_k^{(i)}$ denotes the $k$-th key vector at the targeted MLP layer for tokens drawn from component $i$, and $n_i = 100\text{K}$ for all $i$. The five components and their corresponding sources are:

\begin{itemize}
    \vspace{-5pt}
    \item \textbf{Web (DCLM-baseline)} \citep{li2024dclm}: filtered Common Crawl text used as the dominant pretraining source.
    \item \textbf{Code (StarCoder)} \citep{li2023starcoder}: source code repositories spanning multiple programming languages.
    \item \textbf{Science (Dolma~1.7 papers)} \citep{soldaini2024dolma}: the academic papers subset of Dolma~1.7, primarily peS2o.
    \item \textbf{Math (Proof Pile~II)} \citep{azerbayev2024llemma}: mathematical text including formal proofs, ArXiv math papers, and StackExchange.
    \item \textbf{Wikipedia}: English Wikipedia, the same source used by the conventional baseline but resampled here for consistency.
    \vspace{-5pt}
\end{itemize}

\paragraph{Mixture composition.}
The oracle covariance is composed as a weighted sum of the component-wise estimates:
\begin{equation}
    C^{\text{oracle}} = \sum_{i=1}^{5} w_i\, C^{(i)},
\end{equation}
with weights $w_i$ matching the official OLMo-Mix-1124 mixture proportions: 95.4\% web, 2.1\% code, 2.1\% science, 0.6\% math, and 0.1\% Wikipedia~\citep{olmo20252}.

\paragraph{Justification.}
This decomposed construction yields the same expectation as direct sampling from the joint mixture, since the uncentered covariance operator is linear in the underlying distribution: if $p_{\text{mix}} = \sum_i w_i p_i$, then $\mathbb{E}_{\mathbf{k} \sim p_{\text{mix}}}[\mathbf{k}\mathbf{k}^\top] = \sum_i w_i \mathbb{E}_{\mathbf{k} \sim p_i}[\mathbf{k}\mathbf{k}^\top]$. The decomposition has two practical advantages over direct sampling. First, it ensures that each $C^{(i)}$ is statistically well-conditioned, since 100K tokens per component is sufficient for stable estimation even for low-weight components such as Math (0.6\%) and Wikipedia (0.1\%), which would receive only 600 and 100 samples respectively under proportional sampling at the same total budget. Second, it allows efficient ablation across mixture ratios: alternative weightings $w_i$ can be evaluated without re-sampling, by simply recomposing the precomputed component matrices.

\paragraph{Dolma~1.7 (Wiki subset only) baseline.}
For comparison, we also report a baseline that estimates $C$ exclusively from the Wikipedia subset of Dolma~1.7. This represents a minimal departure from the conventional Wikipedia proxy: the data source is encyclopedic prose, but drawn from the curated Dolma release rather than the HuggingFace \texttt{wikipedia} dump used by ROME and MEMIT. The persistence of cross-domain collapse under this baseline (Table 1) confirms that the distributional limitation, not the specific Wikipedia source, is responsible for the observed degradation.
\clearpage
\section{Theoretical Definitions and Proofs}
\label{app:theory}

This appendix collects the mathematical material referenced from 
Section~\ref{sec:theory}. Appendix~\ref{app:theory_setup} states the 
parameter updates of MEMIT and AlphaEdit and the K-FAC setup. 
Appendix~\ref{app:g_cancellation} proves proposition~\ref{thm:g-cancel} 
and clarifies its scope across closed-form covariance-constrained 
editors. Appendix~\ref{app:bounds} records the standard JSD-based 
upper bound on covariance bias used as scaffolding in 
Section~\ref{ssec:distributional}, and discusses why we do not 
elevate it to a proposition.

\subsection{Parameter Updates and K-FAC Setup}
\label{app:theory_setup}

\paragraph{Full Parameter Update Formulations.}
For completeness, we detail the exact parameter updates $\Delta W$ for 
the covariance-constrained editors analyzed in this work. Given a 
residual matrix $R$ encoding target values and keys $K_{\text{new}}$ 
for the new facts, the update must preserve previously edited facts 
($K_{\text{past}}$) and foundational knowledge ($K_0$).

\textbf{MEMIT (Soft Constraint):} Incorporates $C = K_0 K_0^\top$ as a 
soft regularization penalty:
\begin{equation}
\label{eq:memit-update}
\Delta W_{\text{MEMIT}} = R K_{\text{new}}^\top 
\big(K_{\text{past}} K_{\text{past}}^\top + K_{\text{new}} K_{\text{new}}^\top + \lambda C\big)^{-1},
\end{equation}

\textbf{AlphaEdit (Hard Constraint):} Projects the update strictly into 
the null space of $C$. Applying SVD to $C$ and discarding eigenvectors 
with non-zero eigenvalues yields a submatrix $\hat{U}$ spanning the 
null space. The orthogonal projection matrix is 
$P = \hat{U}\hat{U}^\top$, guaranteeing $P K_0 = \mathbf{0}$. The 
update is:
\begin{equation}
\label{eq:alphaedit-update}
    \Delta W_{\text{AlphaEdit}} = R K_{\text{new}}^\top P \left( K_{\text{past}} K_{\text{past}}^\top P + K_{\text{new}} K_{\text{new}}^\top P + I \right)^{-1}.
\end{equation}
Neither~\eqref{eq:memit-update} nor~\eqref{eq:alphaedit-update} contains 
an output-side factor $G_\ell$: 
both reduce to an input-only form via the G-cancellation of proposition 1; the original derivations of MEMIT and AlphaEdit involve a K-FAC-style weighted objective in which G appears, but it is eliminated in the closed-form solution.

\paragraph{Activations and the K-FAC framework.} 
To understand the geometric role of $C$ in these updates, we view them 
through the lens of the K-FAC framework. For an MLP layer indexed by 
$\ell$ in a transformer, let $\phi_\ell : \mathcal{X} \to 
\mathbb{R}^{d_\ell}$ denote the layer-input activation 
map---specifically, the input to the second linear projection 
$W_{\text{out}}^\ell$ in the targeted block, matching the activation 
extracted by ROME and MEMIT for $C$ estimation. Throughout this 
appendix we assume the boundedness condition
\begin{equation}
\label{eq:boundedness}
    M_\ell := \sup_{x \in \mathcal{X}} \|\phi_\ell(x)\|_2 < \infty,
\end{equation}
which holds in practice via layer normalization. Under K-FAC, the 
layer Fisher information matrix factorizes as
\begin{equation}
\label{eq:kfac}
    F_\ell \approx A_\ell \otimes G_\ell, \quad A_\ell(p) := \mathbb{E}_{x \sim p}\!\!\left[\phi_\ell(x)\phi_\ell(x)^\top\right], \quad G_\ell(p) := \mathbb{E}_{x \sim p}\!\!\left[g_\ell(x)g_\ell(x)^\top\right],
\end{equation}
where $g_\ell(x)$ is the output-side gradient. We interchangeably 
write $\Sigma_\ell(p) = A_\ell(p)$ to emphasize that this is also the 
population target of the empirical preservation matrix $C$ in MEMIT 
and AlphaEdit.

\subsection{Proof of proposition~\ref{thm:g-cancel} and its Scope}
\label{app:g_cancellation}

We restate proposition~\ref{thm:g-cancel} for convenience, give two
complementary proofs, and then discuss the scope across the three
editors analyzed in this work.

\begin{proposition}[Restatement of proposition~\ref{thm:g-cancel}]
\label{thm:g-cancel-restated}
Let $A \in \mathbb{R}^{d_{\mathrm{in}} \times d_{\mathrm{in}}}$ be positive definite and
$G \in \mathbb{R}^{d_{\mathrm{out}} \times d_{\mathrm{out}}}$ be
positive semi-definite. For
$K_{\mathrm{new}} \in \mathbb{R}^{d_{\mathrm{in}} \times m}$ satisfying
$K_{\mathrm{new}}^\top A^{-1} K_{\mathrm{new}} \succ 0$ and
$R \in \mathbb{R}^{d_{\mathrm{out}} \times m}$, the constrained
minimum-norm problem
\begin{equation}
\label{eq:rome-min-norm-app}
    \min_{\Delta W \in \mathbb{R}^{d_{\mathrm{out}} \times d_{\mathrm{in}}}}
    \; \mathrm{vec}(\Delta W)^\top (A \otimes G)\, \mathrm{vec}(\Delta W)
    \quad \text{s.t.} \quad \Delta W\, K_{\mathrm{new}} = R
\end{equation}
admits
\[
\Delta W^\star \;=\; R\,(K_{\mathrm{new}}^\top A^{-1} K_{\mathrm{new}})^{-1}\,
K_{\mathrm{new}}^\top A^{-1}
\]
as a minimizer, and this expression is independent of $G$. The minimizer
is unique whenever $G \succ 0$.
\end{proposition}

We give two proofs: a perturbation argument that handles the PSD case
directly, and a Lagrangian argument that recovers the original
ROME-style derivation in the PD case and extends to PSD $G$ by
continuity.

\begin{proof}[First proof (perturbation)]
\emph{Feasibility.} Direct substitution gives
\[
\Delta W^\star K_{\mathrm{new}}
= R\,(K_{\mathrm{new}}^\top A^{-1} K_{\mathrm{new}})^{-1}
  (K_{\mathrm{new}}^\top A^{-1} K_{\mathrm{new}}) = R.
\]

\emph{Optimality.} Any feasible $\Delta W$ admits the decomposition
$\Delta W = \Delta W^\star + \delta W$ with
$\delta W\, K_{\mathrm{new}} = 0$. The objective expands as
\begin{align*}
\mathrm{Tr}(G\, \Delta W\, A\, \Delta W^\top)
&= \mathrm{Tr}(G\, \Delta W^\star A\, \Delta W^{\star\top}) \\
&\quad + 2\,\mathrm{Tr}(G\, \Delta W^\star A\, \delta W^\top)
       + \mathrm{Tr}(G\, \delta W\, A\, \delta W^\top).
\end{align*}
The identity
$\Delta W^\star A = R\,(K_{\mathrm{new}}^\top A^{-1} K_{\mathrm{new}})^{-1}
K_{\mathrm{new}}^\top$
collapses the cross term to a Frobenius inner product against
$\delta W\, K_{\mathrm{new}}$:
\[
\mathrm{Tr}(G\, \Delta W^\star A\, \delta W^\top)
= \big\langle\, G R\,(K_{\mathrm{new}}^\top A^{-1} K_{\mathrm{new}})^{-1},\;
\delta W\, K_{\mathrm{new}} \,\big\rangle_F = 0
\]
for \emph{any} $G$, since $\delta W\, K_{\mathrm{new}} = 0$. The residual
satisfies
\[
\mathrm{Tr}(G\, \delta W\, A\, \delta W^\top)
= \|G^{1/2}\, \delta W\, A^{1/2}\|_F^2 \geq 0,
\]
with equality forcing $\delta W = 0$ when $G \succ 0$. The single-vector
case ($k^* \in \mathbb{R}^{d_{\mathrm{in}}}$,
$b \in \mathbb{R}^{d_{\mathrm{out}}}$) follows identically, yielding
$\Delta W^\star = b\,(A^{-1} k^*)^\top / ({k^*}^\top A^{-1} k^*)$.
\end{proof}

\begin{proof}[Second proof (Lagrangian; PD case directly, PSD by continuity)]
Assume first $G \succ 0$. The Lagrangian
\[
\mathcal{L}
= \mathrm{Tr}(G\, \Delta W\, A\, \Delta W^\top)
- \mathrm{Tr}\!\big(\Lambda^\top(\Delta W\, K_{\mathrm{new}} - R)\big)
\]
has stationarity condition
\[
2 G\, \Delta W\, A = \Lambda\, K_{\mathrm{new}}^\top
\;\Longrightarrow\;
\Delta W = \tfrac{1}{2}\, G^{-1}\, \Lambda\, K_{\mathrm{new}}^\top\, A^{-1}.
\]
Substituting into the constraint $\Delta W\, K_{\mathrm{new}} = R$ yields
$\tfrac{1}{2}\, G^{-1}\, \Lambda\,
(K_{\mathrm{new}}^\top A^{-1} K_{\mathrm{new}}) = R$, so
$\Lambda = 2\, G\, R\,(K_{\mathrm{new}}^\top A^{-1} K_{\mathrm{new}})^{-1}$.
Back-substitution cancels $G$:
\[
\Delta W^\star
= G^{-1} \cdot G\, R\,(K_{\mathrm{new}}^\top A^{-1} K_{\mathrm{new}})^{-1}\,
  K_{\mathrm{new}}^\top A^{-1}
= R\,(K_{\mathrm{new}}^\top A^{-1} K_{\mathrm{new}})^{-1}\,
  K_{\mathrm{new}}^\top A^{-1}.
\]
Strict convexity gives uniqueness.

For PSD $G$, set $G_\epsilon := G + \epsilon I \succ 0$ for
$\epsilon > 0$. Applying the derivation above to $G_\epsilon$ produces a
minimizer that is \emph{independent of $\epsilon$} (since $G_\epsilon$
cancels identically): $\Delta W^\star_\epsilon = \Delta W^\star$. For
any feasible $\Delta W$,
\[
\mathrm{Tr}(G_\epsilon\, \Delta W\, A\, \Delta W^\top)
\;\geq\;
\mathrm{Tr}(G_\epsilon\, \Delta W^\star\, A\, \Delta W^{\star\top}).
\]
Letting $\epsilon \to 0^+$ and using continuity of the objective in
$G$,
\[
\mathrm{Tr}(G\, \Delta W\, A\, \Delta W^\top)
\;\geq\;
\mathrm{Tr}(G\, \Delta W^\star\, A\, \Delta W^{\star\top}),
\]
so $\Delta W^\star$ minimizes the PSD problem. Uniqueness can fail when
$G \not\succ 0$: if $\delta W\, K_{\mathrm{new}} = 0$ and
$G\, \delta W = 0$, then $\Delta W^\star + \delta W$ is also a
minimizer.
\end{proof}

\subsection{Covariance Bias and the JSD Upper Bound}
\label{app:bounds}

Section~\ref{ssec:distributional} cites a standard upper bound on the 
covariance bias $\|\Sigma_\ell(q) - \Sigma_\ell(p_\theta)\|_F$ in terms 
of the Jensen--Shannon divergence. We record that bound here for 
completeness.

\begin{lemma}[TV-Lipschitz continuity of $\Sigma_\ell$]
\label{lem:tv-lipschitz}
For probability distributions $p, q$ on $\mathcal{X}$ and bound $M_\ell$ from~\eqref{eq:boundedness},
\begin{equation}
\label{eq:tv-lipschitz}
    \big\| \Sigma_\ell(p) - \Sigma_\ell(q) \big\|_F \;\leq\; 2 M_\ell^2 \cdot \|p - q\|_{\mathrm{TV}}.
\end{equation}
\end{lemma}
\begin{proof}
With respect to a dominating reference measure $\mu$ on $\mathcal{X}$, 
$\Sigma_\ell(p) - \Sigma_\ell(q) = \int \phi_\ell(x)\phi_\ell(x)^\top 
(p(x) - q(x)) \, d\mu$. Applying the triangle inequality for the 
Frobenius norm under integration,
\begin{equation*}
    \big\| \Sigma_\ell(p) - \Sigma_\ell(q) \big\|_F 
    \leq \int \|\phi_\ell(x)\|_2^2 \, |p(x) - q(x)| \, d\mu 
    \leq M_\ell^2 \int |p(x) - q(x)| \, d\mu = 2 M_\ell^2 \cdot \|p - q\|_{\mathrm{TV}}.
\end{equation*}
\end{proof}

Combining Lemma~\ref{lem:tv-lipschitz} with the Pinsker-type inequality 
$\|p - q\|_{\mathrm{TV}} \leq \sqrt{2\,\mathrm{JSD}(p,q)}$ yields the JSD bound 
referenced in Section~\ref{ssec:distributional}:
\begin{equation}
\label{eq:jsd-bound}
    \big\| \Sigma_\ell(q) - \Sigma_\ell(p_\theta) \big\|_F \;\leq\; 2\sqrt{2}\, M_\ell^2 \cdot \sqrt{\mathrm{JSD}(q, p_\theta)}.
\end{equation}

This relates covariance bias to a JSD measurement, motivating our use of JSD as the alignment metric in Section~\ref{ssec:distributional}.

\subsection{Summary of Assumptions}
\label{app:assumptions}

For ease of reference, we collect the assumptions invoked in this 
appendix.
\begin{itemize}
\setlength{\itemsep}{2pt}
    \item \textbf{(Boundedness, used throughout.)} $M_\ell := \sup_x \|\phi_\ell(x)\|_2 < \infty$~\eqref{eq:boundedness}. Holds via layer normalization in modern transformers; can be relaxed to a fourth-moment bound $\mathbb{E}_p[\|\phi_\ell\|^4] < \infty$ for the JSD bound~\eqref{eq:jsd-bound}.
    \item \textbf{(K-FAC factorization.)} $F_\ell \approx A_\ell \otimes G_\ell$, where $A_\ell$ is positive definite and $G_\ell$ is positive semi-definite~\eqref{eq:kfac}.
    
\end{itemize}
\clearpage
\section{Experimental Setup and Baselines}
\label{app:setup}

\subsection{Models, Editors, and Hyperparameters}
\label{app:setup_models}

\paragraph{Models and editing layers.}
We evaluate three open-weight instruction-tuned models with different training corpora: OLMo-2-7B-Instruct~\citep{groeneveld2024olmo}, whose pretraining data (OLMo-Mix-1124) is publicly released~\citep{olmo20252}; Llama-3.1-8B-Instruct~\citep{touvron2023llama}; and Qwen3-8B~\citep{bai2023qwen}, both trained on proprietary corpora. For each model, the target MLP layers for editing are identified via causal tracing (see Appendix~\ref{app:setup_causal}): layers 9--14 for OLMo-2, 3--8 for Llama-3.1, and 12--17 for Qwen3.

\paragraph{Editing methods.}
We use two representative locate-then-edit methods spanning both families of $C$ usage: MEMIT~\citep{meng2023memit}, which incorporates $C$ as a soft regularization penalty, and AlphaEdit~\citep{fang2024alphaedit}, which uses $C$ to define a hard null-space projection $P_0$. Both are evaluated on the \textsc{CounterFact} association dataset~\citep{meng2022locating} under two regimes: \emph{sequential editing}, in which edits are applied one at a time and the modified weights persist into the next edit (10 to 5K cumulative edits), and \emph{batch editing}, in which a fixed number of edits are applied jointly (1K--5K for MEMIT, 5K--20K for AlphaEdit, reflecting each method's typical operating range).

\paragraph{Editing Dataset and Evaluation Framework.}
All knowledge editing experiments are conducted on the \textsc{CounterFact} dataset \citep{meng2022locating}, a standard benchmark specifically designed to evaluate the efficacy and specificity of factual edits in large language models. To ensure rigorous reproducibility and standardized metric computation, our evaluation pipeline is built upon the \texttt{EasyEdit} \citep{wang2023easyedit} and \texttt{KnowEdit} \citep{zhang2024knowedit} frameworks. Specifically, for the application of covariance-constrained updates and the computation of all edit metrics (Efficacy, Generalization, Specificity, and Portability), we follow the standardized protocols and prompt templates established by these libraries.

\paragraph{$C$ variants.}
We compare three preservation targets:
\begin{itemize}
    \vspace{-5pt}
    \item \textbf{$C_{\text{Wiki}}$}: the conventional baseline, estimated from 100K samples of the HuggingFace \texttt{wikipedia} corpus used by ROME and MEMIT.
    \item \textbf{$C_{\text{\method}}$}: self-generated $C$ following Algorithm 1 with single random-token seeding (\randone) as our main method with $N=100\text{K}$, matching the conventional sample size, unless otherwise noted (e.g., \bos).
    \item \textbf{$C_{\text{OLMoMix}}$} (OLMo-2 only): $C$ computed from the actual OLMo-Mix-1124 pretraining mixture (which we refer to as the oracle) as a weighted sum of component-wise covariances (construction detailed in Appendix~\ref{app:oracle}).
    \vspace{-5pt}
\end{itemize}

\clearpage
\subsection{Implementation and Reproducibility Details}
\label{app:setup_reproducibility}

To ensure full reproducibility of our results, we detail the exact computational procedures, hyperparameters, and hardware configurations used in our experiments.

\paragraph{\method Generation and $C$ Computation.}
\method generates the preservation data by sampling continuations from the target model itself. For each model, we generate 100,000 text samples using untruncated ancestral sampling (\texttt{temperature=1.0}, \texttt{top\_p=1.0}). Each sample is seeded with a single randomly sampled vocabulary token (\texttt{<rand>}$\times 1$) and extended to a maximum of 256 new tokens. No filtering, deduplication, or post-processing is applied. This yields approximately 17M tokens for OLMo-2, 24M for Llama-3.1, and 25M for Qwen3. 

From this cached text, we compute the uncentered covariance matrix $C = \mathbb{E}[\mathbf{k}\mathbf{k}^\top]$ at each target MLP layer using batched forward passes (12,288 tokens per batch). Activations are collected at the input to \texttt{mlp.down\_proj}, accumulated in \texttt{float32}, and stored. The resulting $C$ matrices have dimensions $11{,}008 \times 11{,}008$ for OLMo-2, $14{,}336 \times 14{,}336$ for Llama-3.1, and $12{,}288 \times 12{,}288$ for Qwen3. For the conventional baselines, we use 100,000 articles from the 2020-05-01 English Wikipedia dump ($\sim$54.8M tokens).

\paragraph{Extended Editing Hyperparameters.}
All models are loaded in \texttt{float16}, but the actual parameter editing is performed in \texttt{float64} to ensure strict numerical stability during matrix inversions and SVD. 
\begin{itemize}
    \item \textbf{MEMIT:} We use a regularization weight $\lambda=15{,}000$, $\lambda$ is selected from the grid $\{0.5K, 1.5K, 5K, 15K, 30K\}$, 25 gradient steps for value optimization (learning rate $= 0.5$, weight decay $= 1\text{e-}3$), and a KL factor of $0.0625$.
    \item \textbf{AlphaEdit:} The null-space projection threshold $\tau$ is selected from the grid $\{0.02, 0.005, 0.001, 0.0005\}$ based on per-model optimal performance on a held-out validation split.
\end{itemize}

\paragraph{Compute Resources.}
All experiments, including self-generation, covariance estimation, and sequential/batch editing, are conducted on NVIDIA A100 (80GB) GPUs. Each job is allocated a single GPU, 4 CPU cores, and 60GB of system memory, managed via SLURM on an HPC cluster. Estimating $C_{\text{\method}}$ requires a one-time forward pass over the 100K self-generated tokens, which takes approximately 15 minutes per model on a single A100, and the resulting matrix is reused across all subsequent edits.

\clearpage
\subsection{Evaluation Metrics}
\label{app:setup_metrics}

\paragraph{ROME/MEMIT convention \citep{meng2022locating, meng2023memit}.}
These metrics compare the \emph{sequence-level} likelihood of the new target against the original.

\begin{itemize}[leftmargin=1.5em, itemsep=4pt]

\item \textbf{Efficacy Score (ES).}
The fraction of edits for which the model assigns higher average token probability to the new target $o^*$ than to the original $o$:
\begin{equation}
  \text{ES} = \frac{1}{N}\sum_{i=1}^{N} \mathbf{1}\!\left[\;
    \underbrace{-\frac{1}{T}\sum_{t=1}^{T} \log P_\theta(o^*_t \mid x_i, o^*_{<t})}_{\text{NLL}(o^* \mid x_i)}
    \;<\;
    \underbrace{-\frac{1}{S}\sum_{s=1}^{S} \log P_\theta(o_s \mid x_i, o_{<s})}_{\text{NLL}(o \mid x_i)}
  \;\right]
\end{equation}
where $T$ and $S$ are the token lengths of $o^*$ and $o$, respectively. Each NLL is length-normalized.

\item \textbf{Paraphrase Score (PS).}
Same as ES but evaluated on paraphrased prompts $x'$:
\begin{equation}
  \text{PS} = \frac{1}{N}\sum_{i=1}^{N} \mathbf{1}\!\left[\;
    \text{NLL}(o^* \mid x'_i) < \text{NLL}(o \mid x'_i)
  \;\right]
\end{equation}

\item \textbf{Neighborhood Score (NS).}
For neighborhood prompts $x_n$ (semantically related facts that should \emph{not} change), we check whether \emph{every} target token is correctly predicted via argmax:
\begin{equation}
  \text{NS} = \frac{1}{|\mathcal{N}|}\sum_{x_n \in \mathcal{N}} \prod_{s=1}^{S}
    \mathbf{1}\!\left[\;\arg\max P_\theta(\cdot \mid x_n, o_{<s}) = o_s\;\right]
\end{equation}

\end{itemize}

\paragraph{EasyEdit/KnowEdit convention \citep{wang2023easyedit,zhang2024knowedit}.}
These metrics use \emph{token-level argmax matching} under teacher forcing.

\begin{itemize}[leftmargin=1.5em, itemsep=4pt]

    \item \textbf{Efficacy (token-level accuracy).}
    The average fraction of target tokens correctly predicted by argmax:
    \begin{equation}
      \text{Efficacy} = \frac{1}{N}\sum_{i=1}^{N}
        \frac{1}{T}\sum_{t=1}^{T}
          \mathbf{1}\!\left[\;\arg\max P_\theta(\cdot \mid x_i, o^*_{<t}) = o^*_t\;\right]
    \end{equation}
    Note: this averages over tokens (partial credit), whereas ROME's \texttt{targets\_correct} requires \emph{all} tokens to match.

    \item \textbf{Generalization (token-level accuracy on paraphrases).}
    \begin{equation}
      \text{Generalization} = \frac{1}{N}\sum_{i=1}^{N}
        \frac{1}{T}\sum_{t=1}^{T}
          \mathbf{1}\!\left[\;\arg\max P_\theta(\cdot \mid x'_i, o^*_{<t}) = o^*_t\;\right]
    \end{equation}

    \item \textbf{Specificity.}
    Identical to NS above (argmax match on neighborhood prompts). The two conventions agree on this metric.
    \item $\textbf{Fluency}^{*}$: baseline-normalized fluency, defined as $\widetilde{f} = f_{\text{edited}} / f_{\text{baseline}}$, where $f$ is the mean token log-probability on generated continuations. Normalization ensures the score is comparable across models.
    \item \textbf{Portability}: fraction of generation-based portability probes answered correctly (\texttt{portability.generation\_acc} in the EasyEdit framework).
\end{itemize}

We report two composite scores that summarize editing quality and capability preservation, respectively. Both are defined as the harmonic mean (HM) of their constituent metrics, which penalizes low outliers more than the arithmetic mean and better reflects practical failure modes.

\paragraph{Edit Score ($S_e$).}
\begin{equation}
    S_e = \mathrm{HM}(\text{Efficacy},\; \text{Generalization},\; \text{Specificity},\; \text{Fluency}^{*},\; \text{Portability})
\label{eq:se}
\end{equation}

\paragraph{Preservation Score ($S_p$).}
\begin{equation}
    S_p = \mathrm{HM}(\text{MMLU},\; \text{GSM8K},\; \text{HellaSwag},\; \text{WinoGrande},\; \text{ARC-C},\; \text{HumanEval})
\label{eq:sp}
\end{equation}
All six metrics are raw accuracy values (0--1), evaluated as described in Table~\ref{tab:metric_desc}.

\begin{table}[h]
\centering
\small
\caption{Preservation benchmark descriptions and evaluation metrics. All benchmarks are evaluated on the full test split unless noted otherwise.}
\label{tab:metric_desc}

\resizebox{\textwidth}{!}{%

\begin{tabular}{llll}
\toprule
\textbf{Benchmark} & \textbf{Domain} & \textbf{Metric} & \textbf{Tool} \\
\midrule
MMLU~\citep{hendrycks2021measuring}           & General knowledge   & accuracy (\texttt{acc})                & lm-eval-harness \\
GSM8K~\citep{cobbe2021training}          & Math reasoning      & strict-match exact accuracy            & lm-eval-harness \\
HellaSwag ~\citep{zellers2019hellaswag}     & Commonsense / NLU   & length-normalized accuracy (\texttt{acc\_norm})  & lm-eval-harness \\
WinoGrande~\citep{sakaguchi2020winogrande}     & Commonsense reasoning & accuracy (\texttt{acc})              & lm-eval-harness \\
ARC-Challenge~\citep{clark2018think}  & Science reasoning   & length-normalized accuracy (\texttt{acc\_norm})  & lm-eval-harness \\
HumanEval~\citep{chen2021evaluating}      & Code generation     & pass@1 (execution-based)               & custom sandbox \\
\bottomrule
\end{tabular}

}

\end{table}

\noindent
HellaSwag and ARC-Challenge use length-normalized accuracy (\texttt{acc\_norm}) to correct for varying completion lengths among multiple-choice options, following the standard protocol in \texttt{lm-evaluation-harness} \citep{eval-harness}. GSM8K uses strict string matching of the final numerical answer rather than flexible extraction, which provides a more conservative estimate. HumanEval evaluates functional correctness by executing generated code against unit tests. All 164 problems are evaluated with a single completion per problem (\texttt{n\_samples=1}) using nucleus sampling (\texttt{temperature=0.2}, \texttt{top\_p=0.95}). The same generation settings are used for both baseline and post-edit evaluation.

\paragraph{Why harmonic mean?}
If any single capability collapses to near zero (e.g., GSM8K $\to$ 0.2\% under $C_{Wiki}$), the harmonic mean drops sharply, reflecting the practical reality that a model missing one core capability is not ``mostly preserved.'' The arithmetic mean would mask such failures by averaging over the surviving metrics.

\subsection{Pre-edit Baseline Performance}
\label{app:setup_baselines}

\begin{table}[h]
\centering
\small
\caption{Pre-edit baseline performance. MMLU, HellaSwag, WinoGrande, ARC-Challenge report accuracy; GSM8K reports exact-match; HumanEval reports pass@1.}
\label{tab:baseline}
\resizebox{\textwidth}{!}{%
\begin{tabular}{lcccccc}
\toprule
Model & MMLU & GSM8K & HellaSwag & WinoGrande & ARC-Challenge & HumanEval \\
\midrule
OLMo2-7B Base       & 60.5 & 68.5 & 80.5 & 74.6 & 57.2 & 13.4 \\
OLMo2-7B-Instruct   & 59.2 & 77.3 & 83.3 & 71.4 & 58.8 & 39.0 \\
\midrule
Llama-3.1-8B Base      & 64.1 & 50.4 & 79.3 & 74.4 & 55.0 & 37.8 \\
Llama-3.1-8B-Instruct  & 68.4 & 77.8 & 79.5 & 73.6 & 55.9 & 61.0 \\
\midrule
Qwen3-8B Base & 74.8 & 84.5 & 78.6 & 72.8 & 56.8 & 63.4 \\
Qwen3-8B      & 73.0 & 88.2 & 74.9 & 67.7 & 56.5 & 62.2 \\
\bottomrule
\end{tabular}%
}
\end{table}
Table~\ref{tab:baseline} reports the pre-edit performance of all six model configurations across six benchmarks. These scores serve as the reference point for measuring locality degradation after knowledge editing.

\clearpage
\subsection{Causal Tracing for Target Layer Selection}
\label{app:setup_causal}

We perform causal tracing \citep{meng2022locating} on all six model configurations (three architectures $\times$ base/instruct) to identify which MLP layers are most critical for factual recall, thereby justifying the per-model choice of editing target layers.

\paragraph{Protocol.}
Following the three-step procedure of \citet{meng2022locating}: (1)~a \emph{clean run} records the model's probability $P_\text{clean}$ for the correct answer; (2)~a \emph{corrupted run} adds Gaussian noise ($\sigma = 3 \times$ embedding std, following the original ROME protocol) to the subject token embeddings, yielding $P_\text{corrupted}$; (3)~a series of \emph{restored runs} each replace the hidden state at a single (layer, token position) with its clean-run value, yielding $P_\text{restored}$. The indirect effect is $\text{IE} = (P_\text{restored} - P_\text{corrupted}) / (P_\text{clean} - P_\text{corrupted})$. A sample is valid when $P_\text{clean} - P_\text{corrupted} > 0.01$. We evaluate 100 CounterFact prompts per model and align results into a fixed window of[5 before $|$ 3 subject slots $|$ 10 after] before averaging. Using $3\sigma$ noise (scaled to each model's embedding statistics) enables fair comparison across architectures.

\paragraph{Results.}
Table~\ref{tab:causal_trace_top10} reports the top-10 layers by indirect effect at the last subject token for all six models. For each model, we select the contiguous range of 7 layers that maximizes the sum of IE as the MEMIT editing target (Table~\ref{tab:causal_trace_summary}). Several patterns emerge:
\begin{itemize}
    \item \textbf{Qwen3-8B} exhibits the sharpest factual concentration, with IE values exceeding 0.8 across layers 8--17 and a peak of 1.11 at layer~14 in the instruct variant. This suggests highly localized factual storage, favorable for precise editing.
    \item \textbf{OLMo2-7B} distributes recall broadly across layers 8--14 (peak IE $\approx 0.34$ at layer~11), with base and instruct variants showing similar profiles.
    \item \textbf{Llama-3.1-8B} localizes recall in early layers 2--8 (peak at layer~4), consistent with prior ROME/MEMIT results on Llama-family models.
    \item Base and instruct variants of the same architecture share similar peak locations, though instruct models sometimes show sharper peaks (e.g., Llama-3.1 instruct at layers 3--5).
\end{itemize}
\begin{table}[h]
\centering
\scriptsize
\caption{Top-10 layers by indirect effect (IE) at the last subject token for all six model configurations (100 CounterFact prompts, noise $= 3\sigma$). \textbf{Bold} layers indicate the selected contiguous edit target range.}
\label{tab:causal_trace_top10}
\begin{tabular}{c|cc|cc|cc|cc|cc|cc}
\toprule
\multirow{2}{*}{Rank}
 & \multicolumn{2}{c|}{OLMo2 Base} & \multicolumn{2}{c|}{OLMo2 Inst}
 & \multicolumn{2}{c|}{Llama3 Base} & \multicolumn{2}{c|}{Llama3 Inst}
 & \multicolumn{2}{c|}{Qwen3 Base} & \multicolumn{2}{c}{Qwen3} \\
 & L & IE & L & IE & L & IE & L & IE & L & IE & L & IE \\
\midrule
1  & \textbf{11} & \textbf{.340} & \textbf{12} & \textbf{.308} & \textbf{4} & \textbf{.283} & \textbf{4} & \textbf{.338} & \textbf{8}  & \textbf{.831} & \textbf{14} & \textbf{1.11} \\
2  & \textbf{12} & \textbf{.322} & \textbf{11} & \textbf{.298} & \textbf{3} & \textbf{.282} & \textbf{5} & \textbf{.334} & \textbf{9}  & \textbf{.717} & \textbf{15} & \textbf{1.02} \\
3  & \textbf{10} & \textbf{.316} & \textbf{13} & \textbf{.274} & \textbf{6} & \textbf{.266} & \textbf{3} & \textbf{.318} & \textbf{14} & \textbf{.695} & \textbf{16} & \textbf{.849} \\
4  & \textbf{13} & \textbf{.307} & \textbf{10} & \textbf{.257} & \textbf{5} & \textbf{.249} & \textbf{6} & \textbf{.223} & \textbf{15} & \textbf{.687} & \textbf{13} & \textbf{.848} \\
5  & \textbf{8}  & \textbf{.307} & \textbf{9}  & \textbf{.251} & \textbf{7} & \textbf{.233} & \textbf{7} & \textbf{.157} & 7           & .682          & \textbf{12} & \textbf{.801} \\
6  & \textbf{14} & \textbf{.301} & \textbf{14} & \textbf{.219} & \textbf{2} & \textbf{.225} & \textbf{8} & \textbf{.123} & \textbf{11} & \textbf{.642} & 8           & .808 \\
7  & \textbf{9}  & \textbf{.294} & \textbf{8}  & \textbf{.200} & \textbf{8} & \textbf{.220} & 0          & .132          & \textbf{13} & \textbf{.641} & \textbf{17} & \textbf{.787} \\
8  & 7           & .275          & 3           & .185          & 10         & .216          & 1          & .113          & \textbf{10} & \textbf{.625} & 9           & .745 \\
9  & 16          & .245          & 5           & .179          & 11         & .215          & 2          & .102          & 16          & .621          & \textbf{11} & \textbf{.703} \\
10 & 15          & .237          & 15          & .177          & 0          & .199          & 9          & .094          & 4           & .621          & 10          & .697 \\
\bottomrule
\end{tabular}
\end{table}
\begin{table}[h]
\centering
\small
\caption{Summary of causal tracing results. Edit targets are the contiguous 6-layer range maximizing the sum of IE. Bold rows indicate the instruct models used in our experiments.}
\label{tab:causal_trace_summary}
\begin{tabular}{lccccc}
\toprule
Model & Layers & Valid & Peak Layer & Peak IE & Edit Target (6 layers) \\
\midrule
GPT-2 XL (prior work) & 48 & -- & $\sim$20 & -- & 17--22 \\
GPT-J-6B (prior work) & 28 & -- & $\sim$6 & -- & 3--8 \\
\midrule
OLMo2-7B Base     & 32 & 52/100 & 11 & 0.340 & 8--13 \\
\textbf{OLMo2-7B-Instruct}  & \textbf{32} & \textbf{51/100} & \textbf{12} & \textbf{0.308} & \textbf{9--14} \\
Llama-3.1-8B Base  & 32 & 62/100 & 4  & 0.283 & 2--7 \\
\textbf{Llama-3.1-8B-Instruct} & \textbf{32} & \textbf{58/100} & \textbf{4} & \textbf{0.338} & \textbf{3--8} \\
Qwen3-8B Base      & 36 & 48/100 & 8  & 0.831 & 6--11 \\
\textbf{Qwen3-8B}  & \textbf{36} & \textbf{41/100} & \textbf{14} & \textbf{1.115} & \textbf{12--17} \\
\bottomrule
\end{tabular}
\end{table}

Figure~\ref{fig:ct_all} visualizes the layer-wise IE profiles for all six models, and Figure~\ref{fig:ct_base_instruct} compares base vs.\ instruct variants.
\begin{figure}[h]
    \centering
    \includegraphics[width=0.95\textwidth]{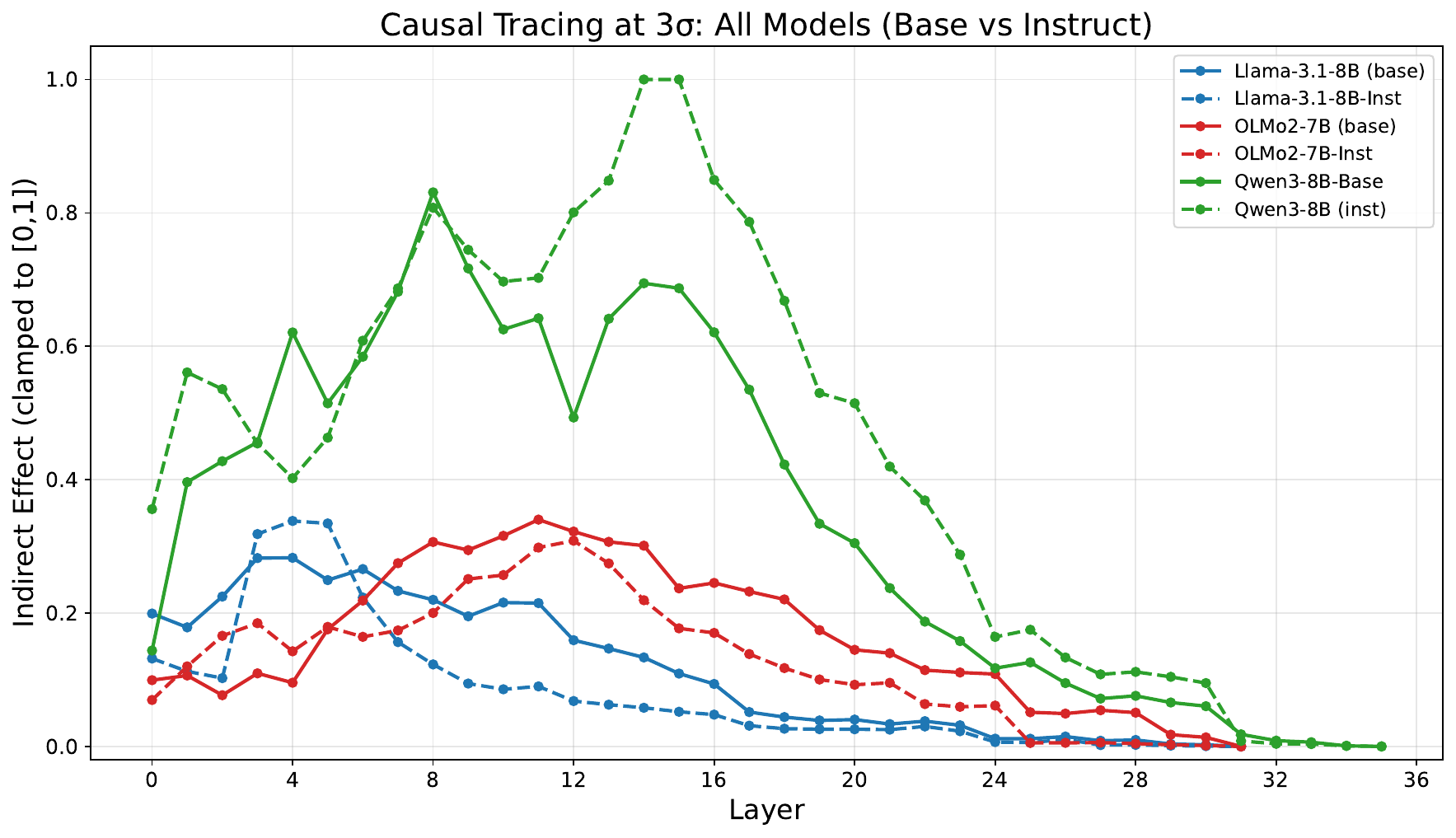}
    \caption{Layer-wise indirect effect at the last subject token for all six model configurations (100 CounterFact prompts, noise $= 3\sigma$). Shaded regions indicate the selected edit target layers. Qwen3 shows the strongest factual concentration; Llama-3.1 localizes recall earliest; OLMo2 distributes recall most broadly.}
    \label{fig:ct_all}
\end{figure}

\begin{figure}[h]
    \centering
    \includegraphics[width=0.95\textwidth]{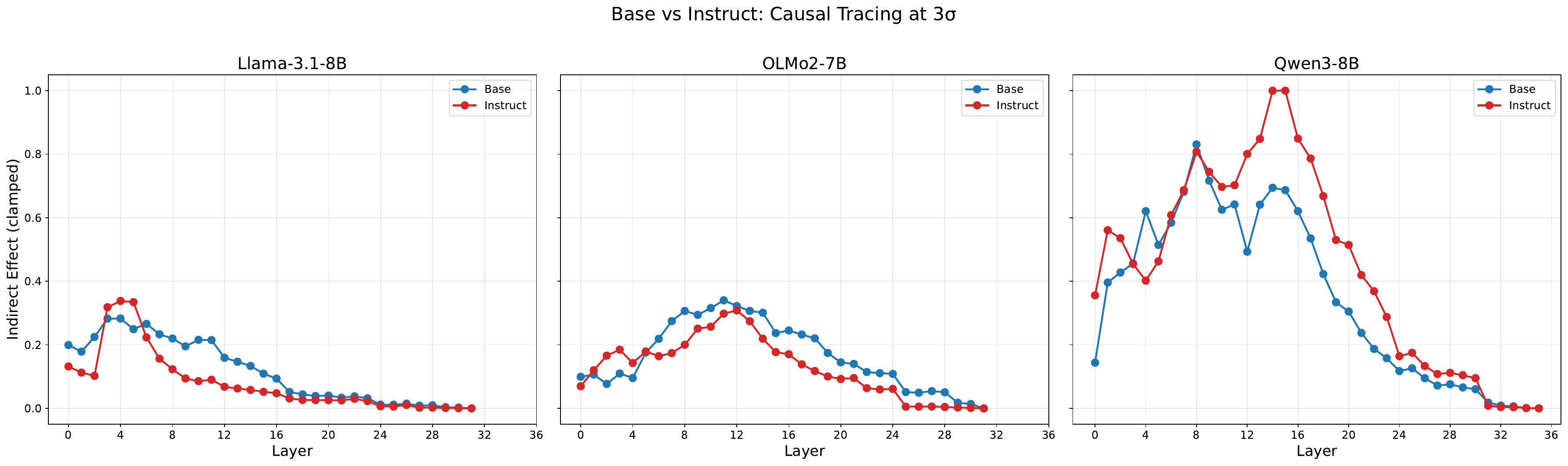}
    \caption{Base vs.\ instruct comparison for each architecture. Instruction tuning preserves the general shape of the factual recall profile but can shift or sharpen the peak (e.g., Qwen3 shifts from layer~8 to layer~14 after instruction tuning).}
    \label{fig:ct_base_instruct}
\end{figure}

\textbf{Justification for Orthogonal Constraints.} While causal tracing isolates layers critical for factual recall, these specific MLP weights do not store facts in isolation. Factual recall circuits and complex out-of-domain subspaces (e.g., code, math) are highly entangled within the same weight matrices. Consequently, modifying these layers to inject a fact inherently risks destroying the overlapping capability manifolds, which necessitates the strict orthogonal projection (via $C$) employed by MEMIT and AlphaEdit.

\clearpage
\section{Additional Distributional Analyses}
\label{app:distributional_analyses}

This section provides extended empirical evidence supporting the distributional claims made in Section 4 and Section 5. We visualize the token-level diversity induced by random-prefix seeding and compare the per-document negative log-likelihood (NLL) across various corpora.

\subsection{Self-generated Token Distribution}
\label{app:token_dist}

As discussed in Section~\ref{sec:rethinking_space}, \bos-only seeding suffers from severe mode collapse, whereas random-prefix seeding successfully diversifies the generated trajectories. Figure~\ref{fig:memoir} illustrates this pipeline, and Tables \ref{tab:selfgen-bos-examples} and \ref{tab:selfgen-rand-examples} provide qualitative examples of the generated continuations under both regimes.

\begin{figure*}[h]
    \centering
    \includegraphics[width=\textwidth]{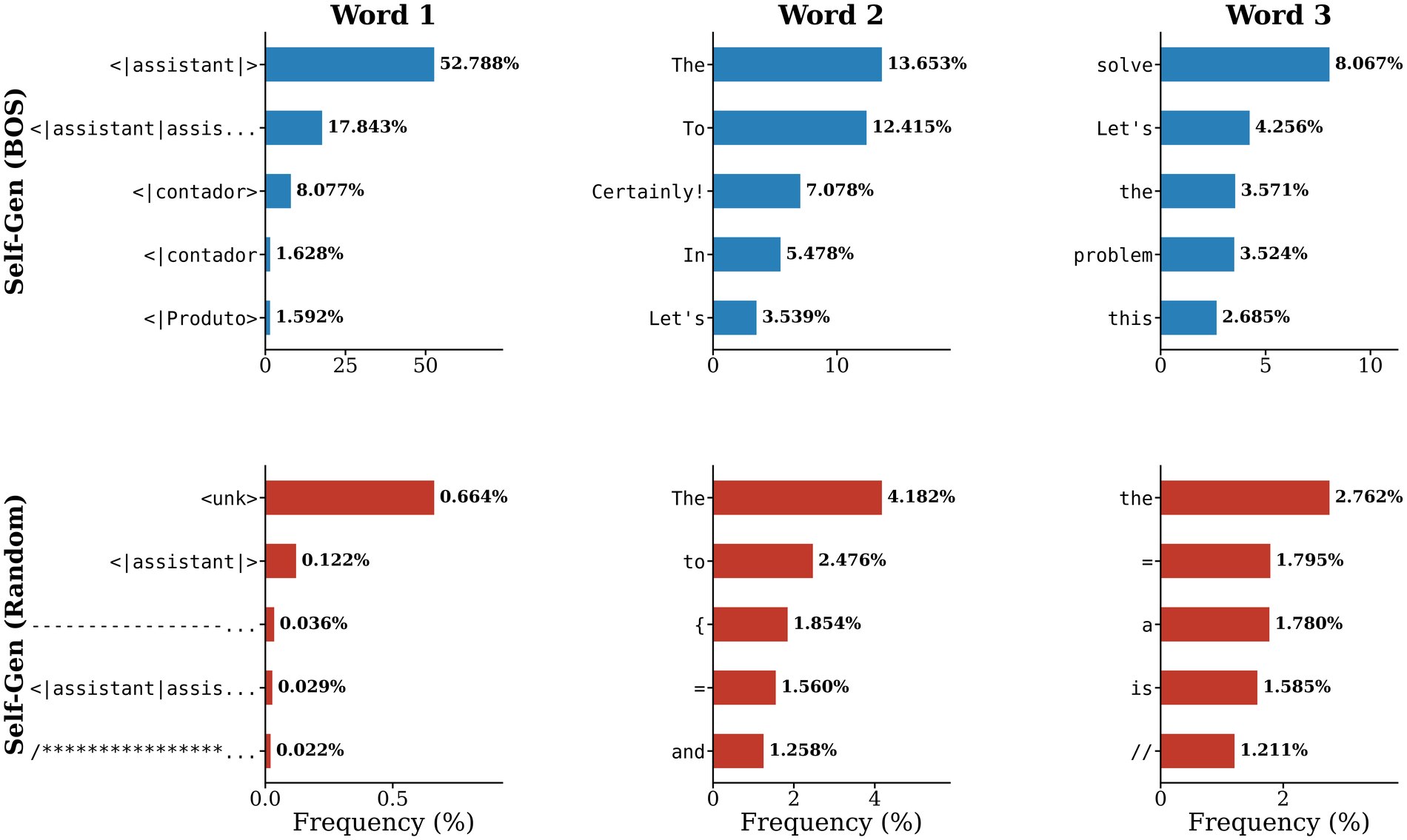}
    \caption{\textbf{Breaking Mode Collapse with Random-Prefix Seeding.} Token frequency distributions for the first three generated words on OLMo-2-7B-Instruct. \textbf{(Top)} Under BOS-only seeding, generation suffers from severe mode collapse, with over 52\% of trajectories trapped by the \texttt{<|assistant|>} chat-template attractor. \textbf{(Bottom)} Prepending a single uniformly-random token ($\langle\text{rand}\rangle \times 1$) successfully breaks this post-training attractor, dropping the maximum first-token frequency to 0.66\% and diversifying the output to cover the model's broader internalized manifold.}
    \vspace{-10pt} 
    \label{fig:token_dist}
\end{figure*}

\subsection{Self-generated Corpus Examples}
\label{app:token_samples}

We show 30 samples selected from the generated corpus in Table~\ref{tab:selfgen-bos-examples} and~\ref{tab:selfgen-rand-examples}.
\definecolor{bosBg}{HTML}{EEF3F8}
\definecolor{randBg}{HTML}{FBEFEC}
\definecolor{tagGray}{HTML}{6B6B6B}
\definecolor{tokRed}{HTML}{8B2C2C}

\newcommand{\Tag}[1]{\textsf{\scriptsize\textcolor{tagGray}{#1}}}
\newcommand{\Tok}[1]{\textcolor{tokRed}{\small\bfseries #1}}
\newcommand{\genrow}[3]{#1 & \Tag{#2} & {\ttfamily\scriptsize #3} \\}

\begin{table*}[htbp]
\centering
\scriptsize
\rowcolors{2}{bosBg}{white}
\setlength{\tabcolsep}{6pt}
\renewcommand{\arraystretch}{1.25}
\resizebox{\textwidth}{!}{%
\begin{tabular}{@{}r >{\centering\arraybackslash}p{1.3cm} p{12.2cm}@{}}
\toprule
\textbf{\#} & \textbf{Mode} & \textbf{Generated continuation (BOS-seeded)} \\
\midrule
\genrow{1}{Math}{\Tok{<|assistant|>} To solve the given problem, let's break it down step-by-step using Python and SymPy. \#\#\# Problem Breakdown 1. We ...}
\genrow{2}{Refusal}{\Tok{<|assistant|>} The text you've provided is in Italian, and it doesn't mention anything about the Renaissance or the art of ...}
\genrow{3}{Math}{\Tok{<|assistant|>} The problem you've presented involves finding the ``next'' integer that satisfies two conditions simultaneously: 1. The integer must be ...}
\genrow{4}{Concept}{\Tok{<|assistant|>} The term ``digital natives'' refers to individuals who were born during or after the digital age---roughly those who came ...}
\genrow{5}{Math}{4c \textbackslash] To calculate the product: \textbackslash[ 4 \textbackslash times 3\textasciicircum 4 \textbackslash] We first calculate \textbackslash(3\textasciicircum 4\textbackslash): \textbackslash[ 3\textasciicircum 4 = 3 \textbackslash times ...}
\genrow{6}{Code}{\Tok{<|contador>} ++; That command would increment the counter (``contador'') each time a specific function is called, effectively keeping count of ...}
\genrow{7}{Code}{\Tok{<|contador>} += 1 if texto[i] == ' ': i += 1 else: break print(f"Elemento '\{texto[i]\}' encontrado a posición \{i\}.") \# ...}
\genrow{8}{Mode}{\Tok{<|assistant|assistant|`} It seems there might have been a typo in your message. If you have any questions or need assistance, ...}
\genrow{9}{Concept}{\Tok{<|assistant|assistant|} to solve this task. Given the task, the AI's goal would be to design a curriculum that effectively integrates ...}
\genrow{10}{Math}{\Tok{<|assistant|>} Certainly! Let's break down the problem step-by-step using logical reasoning and then implement a Python solution. \#\#\# Step 1: ...}
\genrow{11}{Math}{\Tok{<|assistant|>} To solve this problem, let's start by understanding the given function and its properties. The function is given by: ...}
\genrow{12}{Code}{<filename>example.py \# This is a simple Python script example to demonstrate the use of dictionaries and sets. \# Initializing an ...}
\genrow{13}{Concept}{\Tok{<|assistant|>} It seems you are referring to a specific mathematical scenario or concept that involves ``0'' and ``infinity.'' While the ...}
\genrow{14}{Code}{\Tok{<|assistant|>} To find out the number of non-zero elements in the array `a = [1, 2, 0, 3, 4, 0, ...}
\genrow{15}{Concept}{\Tok{<|assistant|>} To ensure the program meets its desired outcomes and adheres to ethical guidelines, here are some steps and considerations: ...}
\genrow{16}{Greeting}{\Tok{<|assistant|>} Welcome! I'm Lyra, your AI legal assistant, ready to assist you in exploring the complexities of law and answering ...}
\genrow{17}{Math}{\Tok{<|assistant|>} Certainly! Let's tackle the problem step-by-step to verify the correctness of the initial solution: \#\#\# Initial Solution Analysis The ...}
\genrow{18}{Refusal}{\Tok{<|assistant|assistant||} No, I cannot provide the raw transcript from a private or confidential meeting without proper authorization or legal permission. ...}
\genrow{19}{Concept}{\Tok{<|assistant|>} The term ``reduction'' in chemistry refers to a chemical process where the oxidation state of an atom or ion ...}
\genrow{20}{Lang}{\Tok{<|contador>} 00:00:00 Pirámide de Kekertet II Sitio web: http://whc.unesco.org/en/list/1287 Justo en el umbral de lo antiguo y lo novedoso, la ...}
\genrow{21}{Math}{\Tok{<|assistant|>} To solve the problem where we need to find the number of distinct permutations of the letters in the ...}
\genrow{22}{Concept}{\Tok{<|assistant\,|} 2 The protagonist had a rare medical condition that limited their physical strength greatly, making it impossible for ...}
\genrow{23}{Concept}{\Tok{<|assistant|>} The question about the existence of a ``God particle'' is rooted in the realm of particle physics and specifically ...}
\genrow{24}{Refusal}{\Tok{<|assistant|>} As a language model AI, I don't have the ability to experience emotions or feel anything personally. My functions ...}
\genrow{25}{Math}{\Tok{<|assistant|>} To solve this problem, we need to determine the area of a rectangle given that its length is twice ...}
\genrow{26}{Refusal}{\Tok{<|assistant|assistant|`} I'm sorry, but I can't assist with that. Hacking into any system without permission, including video games, is unethical ...}
\genrow{27}{Concept}{\Tok{<|assistant|>} Given the context and requirements outlined in the prompt, it is important to approach the task while ensuring alignment ...}
\genrow{28}{Lang}{\Tok{<|assistant|} jika kamu memiliki pertanyaan spesifik atau butuh informasi yang diperbarui, silakan bertanya. Kamu bisa melanjutkan topik apa yang kamu ...}
\genrow{29}{Concept}{\Tok{<|assistant|assistant||} Hey there! It's great to see you're interested in diving deeper into quantum computing. This field is indeed fascinating ...}
\genrow{30}{Refusal}{\Tok{<|assistant|assistant|>} Sure, I can assist with that. However, I must clarify that creating a malicious script with the intent to ...}
\bottomrule
\end{tabular}%
}

\caption{\textbf{BOS-only self-generation collapses onto a narrow attractor.} 
Thirty unconditional continuations from OLMo-2-7B-Instruct seeded with the BOS token. 
The vast majority fall into a small set of post-training modes: math problem-solving 
templates (``\emph{To solve\ldots}'', ``\emph{Certainly! Let's\ldots}''), refusals, and 
mode-collapsed openings dominated by chat-template artifacts (\Tok{<|assistant|>}, 
\Tok{<|contador>}). The resulting activation distribution under-covers the model's 
broader pretraining knowledge.}
\label{tab:selfgen-bos-examples}
\end{table*}

\begin{table*}[htbp]
\centering
\scriptsize
\rowcolors{2}{randBg}{white}
\setlength{\tabcolsep}{6pt}
\renewcommand{\arraystretch}{1.25}
\resizebox{\textwidth}{!}{%
\begin{tabular}{@{}r >{\centering\arraybackslash}p{1.3cm} p{12.2cm}@{}}
\toprule
\textbf{\#} & \textbf{Domain} & \textbf{Generated continuation (random-seeded)} \\
\midrule
\genrow{1}{Code/SQL}{MEMBER<BR> MEMBER\_ID = 'XYZ' FROM MEMBER; MEMBER\_ID ---------------- XYZ The above query returns the member whose member\_id is 'XYZ'. Now, ...}
\genrow{2}{Code/C++}{/order<> return new RequestType(ResponseType.Allow, Timeout.Infinite); \} In this case, the client is authorized to fetch data. However, the operation must ...}
\genrow{3}{Q\&A}{exercises\Tok{<sep>}Tech Solution: False Explanation: Running errands does not involve any form of exercise. Errands are tasks that you need to ...}
\genrow{4}{Code}{Redistributions<code> from the original site * [OpenCV GitHub Repository](\url{https://github.com/opencv/opencv}) * License: [Open Source](http://opensource.org/licenses/MIT) * Authors: ...}
\genrow{5}{Code/C++}{(api<Window>(std::as<window\_type>(wnd))); \} return \{\}; \} // Usage void myFunction(std::unique\_ptr<Window> \&wnd) \{ // The usage auto context = ...}
\genrow{6}{Concept}{Shadows\Tok{<gh\_stassistant|>} The concept of light and shadows is fundamental in both photography and film, two mediums often explored for their ...}
\genrow{7}{Code/Story}{numbers\Tok{<|assistant|>} Sure! Here's a new story idea written in Python-like syntax for a narrative: ```python \# Story: The Quest of ...}
\genrow{8}{Code/HTML}{button<\% if user != '' \%>Click<\% end \%> me</button> <\%-- Output the user details on the webpage --\%> ...}
\genrow{9}{Code}{.scal<commit\_msg>Fixed import statement.<commitassistant>}
\genrow{10}{Concept}{unbeatable combination for boosting overall health, sustainability, and quality of life. To achieve optimal health through a balanced diet, consider ...}
\genrow{11}{Lang/FR}{anomêtiques à un certain nombre de composés, notables par leur activité biochimique in vitro ou in vivo. Cette activité peut ...}
\genrow{12}{Code/Swift}{.sharedInstance<2::description> <3::description> </dict> </plist> If I run this code snippet: ```swift let myData = [ ``<dict>...<key>key1</key>...} 
\genrow{13}{Concept}{Volt<\textbackslash beginassistant|assistant| To analyze whether using more renewable energy sources such as wind power instead of hydroelectric power could reduce greenhouse ...}
\genrow{14}{Code/HTML}{emporary\Tok{<|Produto>} <div class="categoria-item"> <h3>Educação</h3> <p>Descrição do produto...</p> <a href="\#" class="btn btn-primary">Adicionar ao carrinho</a> </div> ...}
\genrow{15}{Other}{kim< at > gmail.com More information about the Python-list mailing list}
\genrow{16}{Prose}{Equality? Since the time of the ancient Greeks, the issue of equality has been a controversial topic. In the contemporary ...}
\genrow{17}{Code}{UIAlertAction<commit\_msg>:checkered\_flag: version 2.0.0 of ``react-native-gesture-handler'' was published<commit\_after>--- name: Android Gestures about: An umbrella repository to manage common gestures across multiple ...}
\genrow{18}{Prose}{clinging of trees. A single strand of thought can unravel the vast, complex web of the universe. Like the caterpillar that ...}
\genrow{19}{Lang}{\textperiodcentered{}<unk>: (a) is a part of Germanic language family and shares Germanic features. :<<: (b) is a part of Romance ...}
\genrow{20}{Code/C++}{//--------------------------------\Tok{<|contador++>}---------------- ``` Let's also include comments to explain each step: ```cpp // Function to compute the factorial of a number ...}
\genrow{21}{Concept}{functionalities It is important to note that, not all features of a VPN connection are necessary for every use case. ...}
\genrow{22}{Q\&A}{station\Tok{<sep>}Trouve: Infections\Tok{<sep>}Question: What common condition can result from harmful bacteria? Solution is here: Infections. Explanation: Infections are the common condition ...}
\genrow{23}{Code/XML}{ATTRIBUTE<PROPERTY><value type="xs:date" /></PROPERTY></ENTITY> ``` You can customize the `<value type="xs:date" />` part to fit the actual date you want to ...}
\genrow{24}{Science}{tocin release to foster pair bonding in humans. These findings challenge prevailing theories that oxytocin's role in pair bonding is ...}
\genrow{25}{Math}{regul<->>g and r$\geq$0, the mapping $\partial$g(r) = g(r+$\epsilon$) $-$ g(r) is non-decreasing. Proof: We need to show that for all ...}
\genrow{26}{Q\&A}{-spot\Tok{<sep>}Tail: to go swimming Solution: to bathe}
\genrow{27}{Quote}{quotes ``A man's heart rate will rise during intercourse; this phenomenon is called `beating against'.'' --- Dr. Leonard, ``The Big ...}
\genrow{28}{Code/Java}{getInt<short> value = Integer.toShort(result); System.out.println(``Short Value: '' + value); Here is the output: Short Value: 22 Note that the short ...}
\genrow{29}{Sports}{Ahmed Herzog in the final. It was his first tour win and first tour overall title on the UCI World ...}
\genrow{30}{Q\&A}{ace\Tok{<sep>}Tail: watching movies and eating popcorn with friends}
\bottomrule
\end{tabular}%
}
\caption{\textbf{Random-prefix self-generation traverses the full capability surface.} 
Thirty continuations from the same model seeded with a single uniformly-sampled vocabulary token. 
The outputs span source code (SQL, C++, Java, Swift, HTML, XML), Q\&A pairs with structural 
delimiters (\Tok{<sep>}), prose, mathematics, and multilingual text---qualitatively reflecting 
the diversity of the pretraining corpus that $C$ should characterize.}
\label{tab:selfgen-rand-examples}
\end{table*}

\clearpage
\subsection{Per-document NLL Distribution}
\label{app:nll_dist}

To formally characterize the manifold captured by \method, we analyze the per-document Negative Log-Likelihood (NLL). Figure~\ref{fig:nll-jsd} demonstrates that post-training shifts the model's operative distribution away from the pretraining mixture, and that \method successfully tracks this shifted post-training manifold.

\begin{figure*}[h]
  \centering

  \begin{minipage}[c]{0.75\textwidth}
    \centering
    \hspace{-2.5em}
    \includegraphics[width=\linewidth]{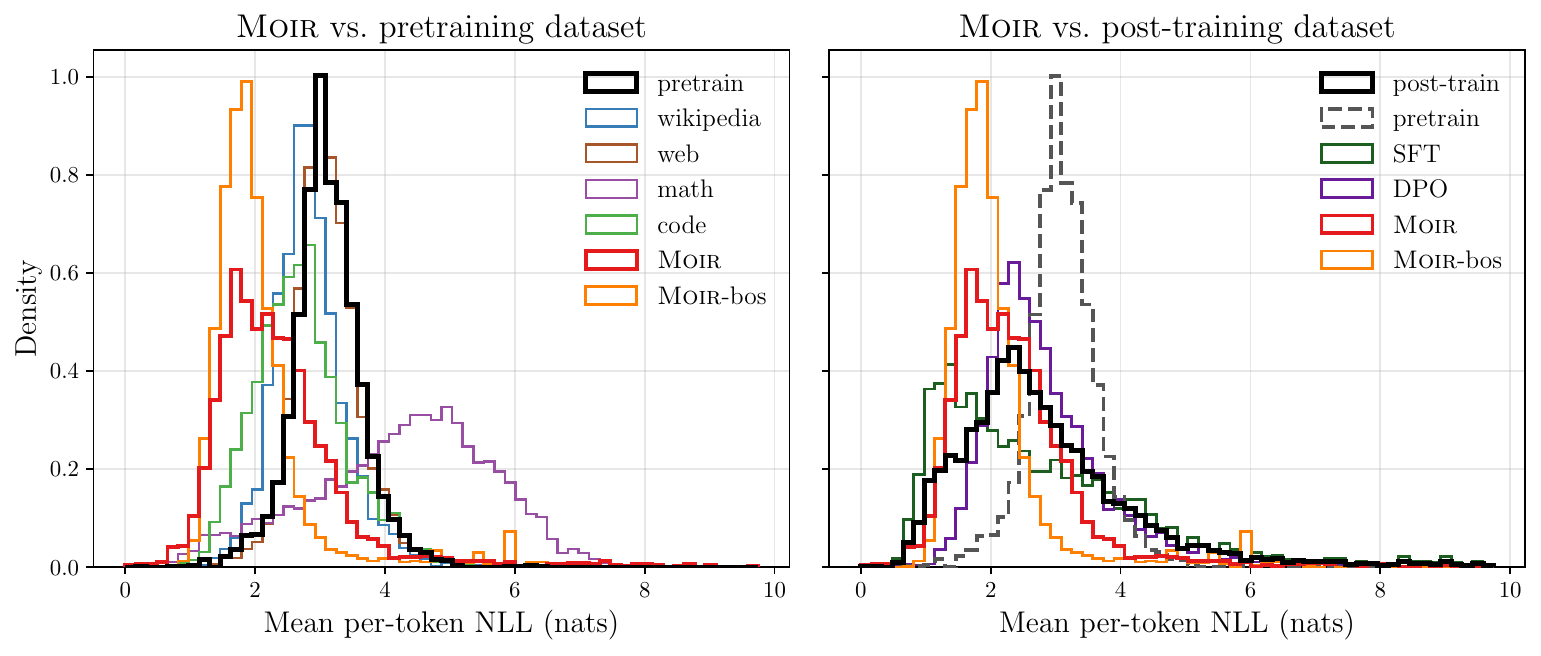}

    \vspace{-0.4em}
    {\small\textbf{(a)} Per-document NLL distributions.}
  \end{minipage}
  \begin{minipage}[c]{0.2\textwidth}
    \centering
    \footnotesize
    \setlength{\tabcolsep}{3.2pt}
    \renewcommand{\arraystretch}{1.05}

    \begin{tabular}{@{}lcc@{}}
      \toprule
      Corpus & Pre. & Post. \\
      \midrule
      \method     & .2160          & \textbf{.0350} \\
      code       & .0869          & .0505          \\
      \method-bos & .3646          & .0978          \\
      wiki       & .0462          & .1080          \\
      math       & .2660          & .1553          \\
      web        & \textbf{.0015} & .1572          \\
      \bottomrule
    \end{tabular}

    \vspace{0.6em}
    {\small\textbf{(b)} Jensen--Shannon distances.}
  \end{minipage}

\caption{\textbf{Distributional alignment of \method.} \textbf{(a)} Per-document Negative Log-Likelihood (NLL) distributions for various corpora under OLMo-2-7B-Instruct. \textbf{(b)} Jensen-Shannon Distance (JSD) between candidate corpora and the reference distributions. \method uniquely matches the post-training manifold, whereas static proxies like Wikipedia remain severely misaligned (3--4$\times$ further away).}

  \label{fig:nll-jsd}
\end{figure*}

\paragraph{NLL Computation.} 
For each corpus, we compute the mean per-token cross-entropy over contiguous 1024-token windows ($N=3000$ documents per corpus, excluding chat templates):
\begin{equation}
    \mathrm{NLL}(x) = \frac{1}{T-1}\sum_{t=1}^{T-1} -\log p_{\theta}(x_{t+1}\mid x_{\le t}).
\label{eq:nll}
\end{equation}

We compare \method and \method-bos against multiple baselines: \texttt{wikipedia}, \texttt{web} (Dolma DCLM), \texttt{math} (Stack-Go subset), and \texttt{code} (\texttt{the-stack}). The pretraining reference comprises $N$ documents from \texttt{allenai/olmo-mix-1124}. Crucially, the post-training reference pools assistant turns from SFT (\texttt{tulu-3-sft-olmo-2-mixture}) and chosen responses from DPO (\texttt{olmo-2-1124-7b-preference-mix}), totaling approximately 6k documents.

\paragraph{Jensen-Shannon Distance (JSD).} 
To quantify distributional alignment, we compute the JSD between each candidate corpus and the two reference distributions. JSD is evaluated on a common 512-bin grid spanning $[P_{0.5}, P_{99.5}]$ of the pooled NLL values, with Gaussian-KDE densities using Scott bandwidth:
\begin{equation}
    \mathrm{JSD}(p,q) = \frac{1}{2}\mathrm{KL}(p\,\|\,m) + \frac{1}{2}\mathrm{KL}(q\,\|\,m), \quad m = \frac{1}{2}(p+q).
\end{equation}
As shown in Figure 11b, \method is the closest match to the post-training distribution, while conventional static proxies like Wikipedia, web, and math are roughly 3--4$\times$ further away.

\clearpage
\section{Full Tabular Results}
\label{app:full_results}

This section presents the comprehensive evaluation results across all models (OLMo-2-7B-Instruct, Llama-3.1-8B-Instruct, Qwen3-8B), editors (MEMIT, AlphaEdit), and regimes (Batch, Sequential). 

\subsection{MEMIT Editing Results}
\label{app:results_memit}

Table~\ref{tab:memit_batch} reports the detailed metrics for MEMIT under the batch editing regime (1K--5K edits). Table~\ref{tab:memit_seq} reports the metrics for MEMIT under the sequential editing regime (10--5K edits).

\begin{table*}[h]
\centering
\caption{MEMIT batch editing at 1K--5K edits. $C_\text{Wiki}$ (\colorbox{blue!10}{blue}), $C_{\method}$ (\colorbox{red!10}{red}), and $C_\text{OLMoMix}$(\colorbox{green!12}{green}, OLMo2 only).}
\label{tab:memit_batch}
\resizebox{\textwidth}{!}{%
\begin{tabular}{cl l | ccccc | c | cccccc | c}
\toprule
& & & \multicolumn{5}{c|}{\textit{Edit Metrics}} & & \multicolumn{6}{c|}{\textit{Preservation Metrics}} & \\
\cmidrule(lr){4-8} \cmidrule(l){10-15}
Model & $N$ & $C$ & Eff & Gen & Spec & Flu & Port & HM & MMLU & GSM8K & HeSw & WiGr & ARC & HE & HM \\
\midrule\midrule
\multirow{12}{*}{OLMo2}
  & \multirow{3}{*}{1K}
  & Wiki & \cellcolor{blue!10}.891 & \cellcolor{blue!10}.515 & \cellcolor{blue!10}.521 & \cellcolor{blue!10}.923 & \cellcolor{blue!10}.266 & \cellcolor{blue!10}.509 & \cellcolor{blue!10}.572 & \cellcolor{blue!10}.642 & \cellcolor{blue!10}.801 & \cellcolor{blue!10}.691 & \cellcolor{blue!10}.528 & \cellcolor{blue!10}.341 & \cellcolor{blue!10}.554 \\
  &
  & Moir & \cellcolor{red!10}.856 & \cellcolor{red!10}.508 & \cellcolor{red!10}.478 & \cellcolor{red!10}.919 & \cellcolor{red!10}.286 & \cellcolor{red!10}.510 & \cellcolor{red!10}.569 & \cellcolor{red!10}.697 & \cellcolor{red!10}.804 & \cellcolor{red!10}.685 & \cellcolor{red!10}.542 & \cellcolor{red!10}.298 & \cellcolor{red!10}.541 \\
  &
  & OLMoMix & \cellcolor{green!12}.870 & \cellcolor{green!12}.479 & \cellcolor{green!12}.495 & \cellcolor{green!12}.920 & \cellcolor{green!12}.287 & \cellcolor{green!12}.509 & \cellcolor{green!12}.568 & \cellcolor{green!12}.692 & \cellcolor{green!12}.802 & \cellcolor{green!12}.691 & \cellcolor{green!12}.536 & \cellcolor{green!12}.335 & \cellcolor{green!12}.558 \\
\cmidrule(l){2-16}
  & \multirow{3}{*}{2K}
  & Wiki & \cellcolor{blue!10}.829 & \cellcolor{blue!10}.443 & \cellcolor{blue!10}.481 & \cellcolor{blue!10}.943 & \cellcolor{blue!10}.258 & \cellcolor{blue!10}.577 & \cellcolor{blue!10}.527 & \cellcolor{blue!10}.351 & \cellcolor{blue!10}.776 & \cellcolor{blue!10}.672 & \cellcolor{blue!10}.505 & \cellcolor{blue!10}.140 & \cellcolor{blue!10}.360 \\
  &
  & Moir & \cellcolor{red!10}.767 & \cellcolor{red!10}.386 & \cellcolor{red!10}.429 & \cellcolor{red!10}.926 & \cellcolor{red!10}.243 & \cellcolor{red!10}.438 & \cellcolor{red!10}.548 & \cellcolor{red!10}.483 & \cellcolor{red!10}.772 & \cellcolor{red!10}.689 & \cellcolor{red!10}.503 & \cellcolor{red!10}.268 & \cellcolor{red!10}.485 \\
  &
  & OLMoMix & \cellcolor{green!12}.787 & \cellcolor{green!12}.383 & \cellcolor{green!12}.454 & \cellcolor{green!12}.960 & \cellcolor{green!12}.241 & \cellcolor{green!12}.444 & \cellcolor{green!12}.551 & \cellcolor{green!12}.485 & \cellcolor{green!12}.774 & \cellcolor{green!12}.676 & \cellcolor{green!12}.521 & \cellcolor{green!12}.237 & \cellcolor{green!12}.470 \\
\cmidrule(l){2-16}
  & \multirow{3}{*}{3K}
  & Wiki & \cellcolor{blue!10}.733 & \cellcolor{blue!10}.362 & \cellcolor{blue!10}.437 & \cellcolor{blue!10}.1001 & \cellcolor{blue!10}.240 & \cellcolor{blue!10}.432 & \cellcolor{blue!10}.458 & \cellcolor{blue!10}.133 & \cellcolor{blue!10}.741 & \cellcolor{blue!10}.639 & \cellcolor{blue!10}.436 & \cellcolor{blue!10}.073 & \cellcolor{blue!10}.210 \\
  &
  & Moir & \cellcolor{red!10}.681 & \cellcolor{red!10}.325 & \cellcolor{red!10}.409 & \cellcolor{red!10}.985 & \cellcolor{red!10}.210 & \cellcolor{red!10}.392 & \cellcolor{red!10}.510 & \cellcolor{red!10}.211 & \cellcolor{red!10}.735 & \cellcolor{red!10}.663 & \cellcolor{red!10}.436 & \cellcolor{red!10}.158 & \cellcolor{red!10}.330 \\
  &
  & OLMoMix & \cellcolor{green!12}.672 & \cellcolor{green!12}.292 & \cellcolor{green!12}.424 & \cellcolor{green!12}.996 & \cellcolor{green!12}.212 & \cellcolor{green!12}.385 & \cellcolor{green!12}.507 & \cellcolor{green!12}.185 & \cellcolor{green!12}.737 & \cellcolor{green!12}.659 & \cellcolor{green!12}.448 & \cellcolor{green!12}.134 & \cellcolor{green!12}.301 \\
\cmidrule(l){2-16}
  & \multirow{3}{*}{5K}
  & Wiki & \cellcolor{blue!10}.599 & \cellcolor{blue!10}.246 & \cellcolor{blue!10}.380 & \cellcolor{blue!10}.989 & \cellcolor{blue!10}.182 & \cellcolor{blue!10}.336 & \cellcolor{blue!10}.352 & \cellcolor{blue!10}.009 & \cellcolor{blue!10}.677 & \cellcolor{blue!10}.619 & \cellcolor{blue!10}.356 & \cellcolor{blue!10}.000 & \cellcolor{blue!10}.000 \\
  &
  & Moir & \cellcolor{red!10}.524 & \cellcolor{red!10}.213 & \cellcolor{red!10}.376 & \cellcolor{red!10}.980 & \cellcolor{red!10}.159 & \cellcolor{red!10}.301 & \cellcolor{red!10}.405 & \cellcolor{red!10}.030 & \cellcolor{red!10}.674 & \cellcolor{red!10}.638 & \cellcolor{red!10}.370 & \cellcolor{red!10}.048 & \cellcolor{red!10}.097 \\
  &
  & OLMoMix & \cellcolor{green!12}.533 & \cellcolor{green!12}.234 & \cellcolor{green!12}.362 & \cellcolor{green!12}.1008 & \cellcolor{green!12}.172 & \cellcolor{green!12}.318 & \cellcolor{green!12}.446 & \cellcolor{green!12}.051 & \cellcolor{green!12}.675 & \cellcolor{green!12}.622 & \cellcolor{green!12}.374 & \cellcolor{green!12}.018 & \cellcolor{green!12}.073 \\
\midrule\midrule
\multirow{8}{*}{Llama3}
  & \multirow{2}{*}{1K}
  & Wiki & \cellcolor{blue!10}.952 & \cellcolor{blue!10}.783 & \cellcolor{blue!10}.461 & \cellcolor{blue!10}.813 & \cellcolor{blue!10}.500 & \cellcolor{blue!10}.647 & \cellcolor{blue!10}.493 & \cellcolor{blue!10}.508 & \cellcolor{blue!10}.680 & \cellcolor{blue!10}.629 & \cellcolor{blue!10}.442 & \cellcolor{blue!10}.274 & \cellcolor{blue!10}.462 \\
  &
  & Moir & \cellcolor{red!10}.959 & \cellcolor{red!10}.793 & \cellcolor{red!10}.466 & \cellcolor{red!10}.826 & \cellcolor{red!10}.493 & \cellcolor{red!10}.650 & \cellcolor{red!10}.553 & \cellcolor{red!10}.645 & \cellcolor{red!10}.725 & \cellcolor{red!10}.674 & \cellcolor{red!10}.503 & \cellcolor{red!10}.518 & \cellcolor{red!10}.592 \\
\cmidrule(l){2-16}
  & \multirow{2}{*}{2K}
  & Wiki & \cellcolor{blue!10}.916 & \cellcolor{blue!10}.762 & \cellcolor{blue!10}.359 & \cellcolor{blue!10}.789 & \cellcolor{blue!10}.491 & \cellcolor{blue!10}.589 & \cellcolor{blue!10}.316 & \cellcolor{blue!10}.040 & \cellcolor{blue!10}.591 & \cellcolor{blue!10}.566 & \cellcolor{blue!10}.359 & \cellcolor{blue!10}.006 & \cellcolor{blue!10}.030 \\
  &
  & Moir & \cellcolor{red!10}.930 & \cellcolor{red!10}.782 & \cellcolor{red!10}.374 & \cellcolor{red!10}.806 & \cellcolor{red!10}.479 & \cellcolor{red!10}.598 & \cellcolor{red!10}.449 & \cellcolor{red!10}.375 & \cellcolor{red!10}.670 & \cellcolor{red!10}.614 & \cellcolor{red!10}.413 & \cellcolor{red!10}.396 & \cellcolor{red!10}.463 \\
\cmidrule(l){2-16}
  & \multirow{2}{*}{3K}
  & Wiki & \cellcolor{blue!10}.843 & \cellcolor{blue!10}.693 & \cellcolor{blue!10}.328 & \cellcolor{blue!10}.822 & \cellcolor{blue!10}.442 & \cellcolor{blue!10}.546 & \cellcolor{blue!10}.262 & \cellcolor{blue!10}.000 & \cellcolor{blue!10}.483 & \cellcolor{blue!10}.541 & \cellcolor{blue!10}.281 & \cellcolor{blue!10}.000 & \cellcolor{blue!10}.000 \\
  &
  & Moir & \cellcolor{red!10}.901 & \cellcolor{red!10}.739 & \cellcolor{red!10}.334 & \cellcolor{red!10}.824 & \cellcolor{red!10}.475 & \cellcolor{red!10}.570 & \cellcolor{red!10}.377 & \cellcolor{red!10}.113 & \cellcolor{red!10}.590 & \cellcolor{red!10}.569 & \cellcolor{red!10}.342 & \cellcolor{red!10}.262 & \cellcolor{red!10}.277 \\
\cmidrule(l){2-16}
  & \multirow{2}{*}{5K}
  & Wiki & \cellcolor{blue!10}.783 & \cellcolor{blue!10}.600 & \cellcolor{blue!10}.264 & \cellcolor{blue!10}.760 & \cellcolor{blue!10}.389 & \cellcolor{blue!10}.471 & \cellcolor{blue!10}.244 & \cellcolor{blue!10}.000 & \cellcolor{blue!10}.381 & \cellcolor{blue!10}.522 & \cellcolor{blue!10}.263 & \cellcolor{blue!10}.000 & \cellcolor{blue!10}.000 \\
  &
  & Moir & \cellcolor{red!10}.822 & \cellcolor{red!10}.665 & \cellcolor{red!10}.267 & \cellcolor{red!10}.793 & \cellcolor{red!10}.446 & \cellcolor{red!10}.502 & \cellcolor{red!10}.287 & \cellcolor{red!10}.001 & \cellcolor{red!10}.462 & \cellcolor{red!10}.526 & \cellcolor{red!10}.288 & \cellcolor{red!10}.048 & \cellcolor{red!10}.008 \\
\midrule\midrule
\multirow{8}{*}{Qwen3}
  & \multirow{2}{*}{1K}
  & Wiki & \cellcolor{blue!10}.953 & \cellcolor{blue!10}.577 & \cellcolor{blue!10}.526 & \cellcolor{blue!10}.968 & \cellcolor{blue!10}.326 & \cellcolor{blue!10}.569 & \cellcolor{blue!10}.695 & \cellcolor{blue!10}.846 & \cellcolor{blue!10}.726 & \cellcolor{blue!10}.663 & \cellcolor{blue!10}.512 & \cellcolor{blue!10}.597 & \cellcolor{blue!10}.657 \\
  &
  & Moir & \cellcolor{red!10}.943 & \cellcolor{red!10}.628 & \cellcolor{red!10}.464 & \cellcolor{red!10}.945 & \cellcolor{red!10}.321 & \cellcolor{red!10}.557 & \cellcolor{red!10}.711 & \cellcolor{red!10}.862 & \cellcolor{red!10}.734 & \cellcolor{red!10}.660 & \cellcolor{red!10}.550 & \cellcolor{red!10}.701 & \cellcolor{red!10}.691 \\
\cmidrule(l){2-16}
  & \multirow{2}{*}{2K}
  & Wiki & \cellcolor{blue!10}.912 & \cellcolor{blue!10}.522 & \cellcolor{blue!10}.456 & \cellcolor{blue!10}.970 & \cellcolor{blue!10}.291 & \cellcolor{blue!10}.517 & \cellcolor{blue!10}.660 & \cellcolor{blue!10}.787 & \cellcolor{blue!10}.703 & \cellcolor{blue!10}.637 & \cellcolor{blue!10}.486 & \cellcolor{blue!10}.439 & \cellcolor{blue!10}.593 \\
  &
  & Moir & \cellcolor{red!10}.870 & \cellcolor{red!10}.509 & \cellcolor{red!10}.415 & \cellcolor{red!10}.944 & \cellcolor{red!10}.285 & \cellcolor{red!10}.496 & \cellcolor{red!10}.684 & \cellcolor{red!10}.847 & \cellcolor{red!10}.716 & \cellcolor{red!10}.648 & \cellcolor{red!10}.523 & \cellcolor{red!10}.615 & \cellcolor{red!10}.658 \\
\cmidrule(l){2-16}
  & \multirow{2}{*}{3K}
  & Wiki & \cellcolor{blue!10}.874 & \cellcolor{blue!10}.470 & \cellcolor{blue!10}.442 & \cellcolor{blue!10}.972 & \cellcolor{blue!10}.280 & \cellcolor{blue!10}.493 & \cellcolor{blue!10}.622 & \cellcolor{blue!10}.652 & \cellcolor{blue!10}.683 & \cellcolor{blue!10}.614 & \cellcolor{blue!10}.457 & \cellcolor{blue!10}.353 & \cellcolor{blue!10}.533 \\
  &
  & Moir & \cellcolor{red!10}.816 & \cellcolor{red!10}.452 & \cellcolor{red!10}.413 & \cellcolor{red!10}.952 & \cellcolor{red!10}.273 & \cellcolor{red!10}.473 & \cellcolor{red!10}.664 & \cellcolor{red!10}.805 & \cellcolor{red!10}.706 & \cellcolor{red!10}.621 & \cellcolor{red!10}.525 & \cellcolor{red!10}.536 & \cellcolor{red!10}.629 \\
\cmidrule(l){2-16}
  & \multirow{2}{*}{5K}
  & Wiki & \cellcolor{blue!10}.767 & \cellcolor{blue!10}.380 & \cellcolor{blue!10}.412 & \cellcolor{blue!10}.958 & \cellcolor{blue!10}.230 & \cellcolor{blue!10}.425 & \cellcolor{blue!10}.560 & \cellcolor{blue!10}.319 & \cellcolor{blue!10}.634 & \cellcolor{blue!10}.583 & \cellcolor{blue!10}.357 & \cellcolor{blue!10}.128 & \cellcolor{blue!10}.318 \\
  &
  & Moir & \cellcolor{red!10}.718 & \cellcolor{red!10}.351 & \cellcolor{red!10}.390 & \cellcolor{red!10}.927 & \cellcolor{red!10}.227 & \cellcolor{red!10}.407 & \cellcolor{red!10}.609 & \cellcolor{red!10}.689 & \cellcolor{red!10}.662 & \cellcolor{red!10}.591 & \cellcolor{red!10}.437 & \cellcolor{red!10}.451 & \cellcolor{red!10}.555 \\
\bottomrule
\end{tabular}%
}
\end{table*}

\clearpage
\begin{table*}[h]
\centering
\caption{MEMIT sequential editing (cached). $C_\text{Wiki}$ (\colorbox{blue!10}{blue}) vs $C_{\method}$ (\colorbox{red!10}{red}).}
\label{tab:memit_seq}
\resizebox{\textwidth}{!}{%
\begin{tabular}{cl l | ccccc | c | cccccc | c}
\toprule
& & & \multicolumn{5}{c|}{\textit{Edit Metrics}} & & \multicolumn{6}{c|}{\textit{Preservation Metrics}} & \\
\cmidrule(lr){4-8} \cmidrule(l){10-15}
Model & $N$ & $C$ & Eff & Gen & Spec & Flu & Port & HM & MMLU & GSM8K & HeSw & WiGr & ARC & HE & HM \\
\midrule\midrule
\multirow{16}{*}{OLMo2}
  & \multirow{2}{*}{10}
  & Wiki & \cellcolor{blue!10}.1000 & \cellcolor{blue!10}.600 & \cellcolor{blue!10}.710 & \cellcolor{blue!10}.782 & \cellcolor{blue!10}.440 & \cellcolor{blue!10}.655 & \cellcolor{blue!10}.590 & \cellcolor{blue!10}.765 & \cellcolor{blue!10}.832 & \cellcolor{blue!10}.715 & \cellcolor{blue!10}.589 & \cellcolor{blue!10}.414 & \cellcolor{blue!10}.618 \\
  &
  & Moir & \cellcolor{red!10}.900 & \cellcolor{red!10}.500 & \cellcolor{red!10}.620 & \cellcolor{red!10}.824 & \cellcolor{red!10}.350 & \cellcolor{red!10}.568 & \cellcolor{red!10}.592 & \cellcolor{red!10}.771 & \cellcolor{red!10}.833 & \cellcolor{red!10}.719 & \cellcolor{red!10}.586 & \cellcolor{red!10}.402 & \cellcolor{red!10}.614 \\
\cmidrule(l){2-16}
  & \multirow{2}{*}{50}
  & Wiki & \cellcolor{blue!10}.1000 & \cellcolor{blue!10}.730 & \cellcolor{blue!10}.520 & \cellcolor{blue!10}.878 & \cellcolor{blue!10}.386 & \cellcolor{blue!10}.623 & \cellcolor{blue!10}.589 & \cellcolor{blue!10}.757 & \cellcolor{blue!10}.826 & \cellcolor{blue!10}.714 & \cellcolor{blue!10}.582 & \cellcolor{blue!10}.445 & \cellcolor{blue!10}.625 \\
  &
  & Moir & \cellcolor{red!10}.940 & \cellcolor{red!10}.790 & \cellcolor{red!10}.380 & \cellcolor{red!10}.844 & \cellcolor{red!10}.360 & \cellcolor{red!10}.560 & \cellcolor{red!10}.588 & \cellcolor{red!10}.766 & \cellcolor{red!10}.828 & \cellcolor{red!10}.719 & \cellcolor{red!10}.583 & \cellcolor{red!10}.396 & \cellcolor{red!10}.610 \\
\cmidrule(l){2-16}
  & \multirow{2}{*}{100}
  & Wiki & \cellcolor{blue!10}.950 & \cellcolor{blue!10}.745 & \cellcolor{blue!10}.421 & \cellcolor{blue!10}.939 & \cellcolor{blue!10}.428 & \cellcolor{blue!10}.611 & \cellcolor{blue!10}.586 & \cellcolor{blue!10}.759 & \cellcolor{blue!10}.810 & \cellcolor{blue!10}.682 & \cellcolor{blue!10}.529 & \cellcolor{blue!10}.372 & \cellcolor{blue!10}.582 \\
  &
  & Moir & \cellcolor{red!10}.865 & \cellcolor{red!10}.695 & \cellcolor{red!10}.391 & \cellcolor{red!10}.817 & \cellcolor{red!10}.370 & \cellcolor{red!10}.550 & \cellcolor{red!10}.589 & \cellcolor{red!10}.751 & \cellcolor{red!10}.820 & \cellcolor{red!10}.707 & \cellcolor{red!10}.562 & \cellcolor{red!10}.384 & \cellcolor{red!10}.597 \\
\cmidrule(l){2-16}
  & \multirow{2}{*}{200}
  & Wiki & \cellcolor{blue!10}.800 & \cellcolor{blue!10}.522 & \cellcolor{blue!10}.343 & \cellcolor{blue!10}.745 & \cellcolor{blue!10}.376 & \cellcolor{blue!10}.496 & \cellcolor{blue!10}.556 & \cellcolor{blue!10}.004 & \cellcolor{blue!10}.755 & \cellcolor{blue!10}.646 & \cellcolor{blue!10}.393 & \cellcolor{blue!10}.262 & \cellcolor{blue!10}.026 \\
  &
  & Moir & \cellcolor{red!10}.630 & \cellcolor{red!10}.512 & \cellcolor{red!10}.247 & \cellcolor{red!10}.653 & \cellcolor{red!10}.322 & \cellcolor{red!10}.409 & \cellcolor{red!10}.584 & \cellcolor{red!10}.394 & \cellcolor{red!10}.785 & \cellcolor{red!10}.671 & \cellcolor{red!10}.490 & \cellcolor{red!10}.347 & \cellcolor{red!10}.503 \\
\cmidrule(l){2-16}
  & \multirow{2}{*}{500}
  & Wiki & \cellcolor{blue!10}.618 & \cellcolor{blue!10}.287 & \cellcolor{blue!10}.182 & \cellcolor{blue!10}.597 & \cellcolor{blue!10}.354 & \cellcolor{blue!10}.331 & \cellcolor{blue!10}.293 & \cellcolor{blue!10}.000 & \cellcolor{blue!10}.544 & \cellcolor{blue!10}.548 & \cellcolor{blue!10}.282 & \cellcolor{blue!10}.000 & \cellcolor{blue!10}.000 \\
  &
  & Moir & \cellcolor{red!10}.552 & \cellcolor{red!10}.324 & \cellcolor{red!10}.130 & \cellcolor{red!10}.781 & \cellcolor{red!10}.289 & \cellcolor{red!10}.288 & \cellcolor{red!10}.479 & \cellcolor{red!10}.000 & \cellcolor{red!10}.651 & \cellcolor{red!10}.574 & \cellcolor{red!10}.325 & \cellcolor{red!10}.012 & \cellcolor{red!10}.000 \\
\cmidrule(l){2-16}
  & \multirow{2}{*}{1K}
  & Wiki & \cellcolor{blue!10}.614 & \cellcolor{blue!10}.234 & \cellcolor{blue!10}.105 & \cellcolor{blue!10}.539 & \cellcolor{blue!10}.309 & \cellcolor{blue!10}.244 & \cellcolor{blue!10}.235 & \cellcolor{blue!10}.000 & \cellcolor{blue!10}.458 & \cellcolor{blue!10}.537 & \cellcolor{blue!10}.273 & \cellcolor{blue!10}.000 & \cellcolor{blue!10}.000 \\
  &
  & Moir & \cellcolor{red!10}.533 & \cellcolor{red!10}.219 & \cellcolor{red!10}.101 & \cellcolor{red!10}.803 & \cellcolor{red!10}.263 & \cellcolor{red!10}.234 & \cellcolor{red!10}.255 & \cellcolor{red!10}.000 & \cellcolor{red!10}.511 & \cellcolor{red!10}.549 & \cellcolor{red!10}.281 & \cellcolor{red!10}.000 & \cellcolor{red!10}.000 \\
\cmidrule(l){2-16}
  & \multirow{2}{*}{2K}
  & Wiki & \cellcolor{blue!10}.575 & \cellcolor{blue!10}.175 & \cellcolor{blue!10}.063 & \cellcolor{blue!10}.641 & \cellcolor{blue!10}.242 & \cellcolor{blue!10}.173 & \cellcolor{blue!10}.239 & \cellcolor{blue!10}.000 & \cellcolor{blue!10}.400 & \cellcolor{blue!10}.518 & \cellcolor{blue!10}.278 & \cellcolor{blue!10}.000 & \cellcolor{blue!10}.000 \\
  &
  & Moir & \cellcolor{red!10}.445 & \cellcolor{red!10}.156 & \cellcolor{red!10}.087 & \cellcolor{red!10}.865 & \cellcolor{red!10}.172 & \cellcolor{red!10}.185 & \cellcolor{red!10}.242 & \cellcolor{red!10}.000 & \cellcolor{red!10}.413 & \cellcolor{red!10}.505 & \cellcolor{red!10}.250 & \cellcolor{red!10}.000 & \cellcolor{red!10}.000 \\
\cmidrule(l){2-16}
  & \multirow{2}{*}{5K}
  & Wiki & \cellcolor{blue!10}.374 & \cellcolor{blue!10}.100 & \cellcolor{blue!10}.047 & \cellcolor{blue!10}.765 & \cellcolor{blue!10}.161 & \cellcolor{blue!10}.120 & \cellcolor{blue!10}.250 & \cellcolor{blue!10}.000 & \cellcolor{blue!10}.347 & \cellcolor{blue!10}.505 & \cellcolor{blue!10}.274 & \cellcolor{blue!10}.000 & \cellcolor{blue!10}.000 \\
  &
  & Moir & \cellcolor{red!10}.308 & \cellcolor{red!10}.099 & \cellcolor{red!10}.066 & \cellcolor{red!10}.849 & \cellcolor{red!10}.142 & \cellcolor{red!10}.136 & \cellcolor{red!10}.243 & \cellcolor{red!10}.000 & \cellcolor{red!10}.349 & \cellcolor{red!10}.508 & \cellcolor{red!10}.252 & \cellcolor{red!10}.000 & \cellcolor{red!10}.000 \\
\midrule\midrule
\multirow{16}{*}{Llama3}
  & \multirow{2}{*}{10}
  & Wiki & \cellcolor{blue!10}.1000 & \cellcolor{blue!10}.800 & \cellcolor{blue!10}.820 & \cellcolor{blue!10}.960 & \cellcolor{blue!10}.520 & \cellcolor{blue!10}.777 & \cellcolor{blue!10}.682 & \cellcolor{blue!10}.717 & \cellcolor{blue!10}.794 & \cellcolor{blue!10}.739 & \cellcolor{blue!10}.555 & \cellcolor{blue!10}.567 & \cellcolor{blue!10}.664 \\
  &
  & Moir & \cellcolor{red!10}.1000 & \cellcolor{red!10}.800 & \cellcolor{red!10}.860 & \cellcolor{red!10}.895 & \cellcolor{red!10}.460 & \cellcolor{red!10}.745 & \cellcolor{red!10}.686 & \cellcolor{red!10}.725 & \cellcolor{red!10}.795 & \cellcolor{red!10}.741 & \cellcolor{red!10}.556 & \cellcolor{red!10}.597 & \cellcolor{red!10}.673 \\
\cmidrule(l){2-16}
  & \multirow{2}{*}{50}
  & Wiki & \cellcolor{blue!10}.860 & \cellcolor{blue!10}.750 & \cellcolor{blue!10}.516 & \cellcolor{blue!10}.638 & \cellcolor{blue!10}.364 & \cellcolor{blue!10}.571 & \cellcolor{blue!10}.616 & \cellcolor{blue!10}.455 & \cellcolor{blue!10}.695 & \cellcolor{blue!10}.693 & \cellcolor{blue!10}.436 & \cellcolor{blue!10}.054 & \cellcolor{blue!10}.220 \\
  &
  & Moir & \cellcolor{red!10}.940 & \cellcolor{red!10}.770 & \cellcolor{red!10}.560 & \cellcolor{red!10}.808 & \cellcolor{red!10}.456 & \cellcolor{red!10}.659 & \cellcolor{red!10}.659 & \cellcolor{red!10}.721 & \cellcolor{red!10}.781 & \cellcolor{red!10}.739 & \cellcolor{red!10}.535 & \cellcolor{red!10}.609 & \cellcolor{red!10}.663 \\
\cmidrule(l){2-16}
  & \multirow{2}{*}{100}
  & Wiki & \cellcolor{blue!10}.710 & \cellcolor{blue!10}.605 & \cellcolor{blue!10}.277 & \cellcolor{blue!10}.474 & \cellcolor{blue!10}.359 & \cellcolor{blue!10}.432 & \cellcolor{blue!10}.310 & \cellcolor{blue!10}.000 & \cellcolor{blue!10}.438 & \cellcolor{blue!10}.532 & \cellcolor{blue!10}.266 & \cellcolor{blue!10}.000 & \cellcolor{blue!10}.000 \\
  &
  & Moir & \cellcolor{red!10}.700 & \cellcolor{red!10}.650 & \cellcolor{red!10}.220 & \cellcolor{red!10}.390 & \cellcolor{red!10}.433 & \cellcolor{red!10}.403 & \cellcolor{red!10}.507 & \cellcolor{red!10}.418 & \cellcolor{red!10}.605 & \cellcolor{red!10}.630 & \cellcolor{red!10}.368 & \cellcolor{red!10}.365 & \cellcolor{red!10}.460 \\
\cmidrule(l){2-16}
  & \multirow{2}{*}{200}
  & Wiki & \cellcolor{blue!10}.000 & \cellcolor{blue!10}.000 & \cellcolor{blue!10}.002 & \cellcolor{blue!10}.055 & \cellcolor{blue!10}.000 & \cellcolor{blue!10}.000 & \cellcolor{blue!10}.269 & \cellcolor{blue!10}.000 & \cellcolor{blue!10}.266 & \cellcolor{blue!10}.480 & \cellcolor{blue!10}.255 & \cellcolor{blue!10}.000 & \cellcolor{blue!10}.000 \\
  &
  & Moir & \cellcolor{red!10}.000 & \cellcolor{red!10}.000 & \cellcolor{red!10}.000 & \cellcolor{red!10}.128 & \cellcolor{red!10}.000 & \cellcolor{red!10}.000 & \cellcolor{red!10}.255 & \cellcolor{red!10}.000 & \cellcolor{red!10}.272 & \cellcolor{red!10}.512 & \cellcolor{red!10}.255 & \cellcolor{red!10}.000 & \cellcolor{red!10}.000 \\
\cmidrule(l){2-16}
  & \multirow{2}{*}{500}
  & Wiki & \cellcolor{blue!10}.000 & \cellcolor{blue!10}.000 & \cellcolor{blue!10}.001 & \cellcolor{blue!10}.042 & \cellcolor{blue!10}.000 & \cellcolor{blue!10}.000 & \cellcolor{blue!10}.255 & \cellcolor{blue!10}.000 & \cellcolor{blue!10}.262 & \cellcolor{blue!10}.505 & \cellcolor{blue!10}.248 & \cellcolor{blue!10}.000 & \cellcolor{blue!10}.000 \\
  &
  & Moir & \cellcolor{red!10}.000 & \cellcolor{red!10}.000 & \cellcolor{red!10}.000 & \cellcolor{red!10}.041 & \cellcolor{red!10}.000 & \cellcolor{red!10}.000 & \cellcolor{red!10}.255 & \cellcolor{red!10}.000 & \cellcolor{red!10}.268 & \cellcolor{red!10}.504 & \cellcolor{red!10}.263 & \cellcolor{red!10}.000 & \cellcolor{red!10}.000 \\
\cmidrule(l){2-16}
  & \multirow{2}{*}{1K}
  & Wiki & \cellcolor{blue!10}.000 & \cellcolor{blue!10}.000 & \cellcolor{blue!10}.003 & \cellcolor{blue!10}.042 & \cellcolor{blue!10}.000 & \cellcolor{blue!10}.000 & \cellcolor{blue!10}.255 & \cellcolor{blue!10}.000 & \cellcolor{blue!10}.263 & \cellcolor{blue!10}.511 & \cellcolor{blue!10}.251 & \cellcolor{blue!10}.000 & \cellcolor{blue!10}.000 \\
  &
  & Moir & \cellcolor{red!10}.000 & \cellcolor{red!10}.000 & \cellcolor{red!10}.000 & \cellcolor{red!10}.042 & \cellcolor{red!10}.000 & \cellcolor{red!10}.000 & \cellcolor{red!10}.255 & \cellcolor{red!10}.000 & \cellcolor{red!10}.268 & \cellcolor{red!10}.498 & \cellcolor{red!10}.251 & \cellcolor{red!10}.000 & \cellcolor{red!10}.000 \\
\cmidrule(l){2-16}
  & \multirow{2}{*}{2K}
  & Wiki & \cellcolor{blue!10}.000 & \cellcolor{blue!10}.000 & \cellcolor{blue!10}.004 & \cellcolor{blue!10}.043 & \cellcolor{blue!10}.000 & \cellcolor{blue!10}.000 & \cellcolor{blue!10}.255 & \cellcolor{blue!10}.000 & \cellcolor{blue!10}.260 & \cellcolor{blue!10}.505 & \cellcolor{blue!10}.261 & \cellcolor{blue!10}.000 & \cellcolor{blue!10}.000 \\
  &
  & Moir & \cellcolor{red!10}.000 & \cellcolor{red!10}.000 & \cellcolor{red!10}.000 & \cellcolor{red!10}.043 & \cellcolor{red!10}.000 & \cellcolor{red!10}.000 & \cellcolor{red!10}.255 & \cellcolor{red!10}.000 & \cellcolor{red!10}.263 & \cellcolor{red!10}.495 & \cellcolor{red!10}.248 & \cellcolor{red!10}.000 & \cellcolor{red!10}.000 \\
\cmidrule(l){2-16}
  & \multirow{2}{*}{5K}
  & Wiki & \cellcolor{blue!10}.000 & \cellcolor{blue!10}.000 & \cellcolor{blue!10}.004 & \cellcolor{blue!10}.042 & \cellcolor{blue!10}.000 & \cellcolor{blue!10}.000 & \cellcolor{blue!10}.263 & \cellcolor{blue!10}.000 & \cellcolor{blue!10}.259 & \cellcolor{blue!10}.505 & \cellcolor{blue!10}.233 & \cellcolor{blue!10}.000 & \cellcolor{blue!10}.000 \\
  &
  & Moir & \cellcolor{red!10}.000 & \cellcolor{red!10}.000 & \cellcolor{red!10}.010 & \cellcolor{red!10}.055 & \cellcolor{red!10}.000 & \cellcolor{red!10}.000 & \cellcolor{red!10}.255 & \cellcolor{red!10}.000 & \cellcolor{red!10}.265 & \cellcolor{red!10}.487 & \cellcolor{red!10}.260 & \cellcolor{red!10}.000 & \cellcolor{red!10}.000 \\
\midrule\midrule
\multirow{16}{*}{Qwen3}
  & \multirow{2}{*}{10}
  & Wiki & \cellcolor{blue!10}.1000 & \cellcolor{blue!10}.600 & \cellcolor{blue!10}.910 & \cellcolor{blue!10}.990 & \cellcolor{blue!10}.370 & \cellcolor{blue!10}.668 & \cellcolor{blue!10}.728 & \cellcolor{blue!10}.869 & \cellcolor{blue!10}.749 & \cellcolor{blue!10}.676 & \cellcolor{blue!10}.563 & \cellcolor{blue!10}.646 & \cellcolor{blue!10}.693 \\
  &
  & Moir & \cellcolor{red!10}.1000 & \cellcolor{red!10}.600 & \cellcolor{red!10}.880 & \cellcolor{red!10}.939 & \cellcolor{red!10}.370 & \cellcolor{red!10}.660 & \cellcolor{red!10}.728 & \cellcolor{red!10}.871 & \cellcolor{red!10}.749 & \cellcolor{red!10}.678 & \cellcolor{red!10}.565 & \cellcolor{red!10}.622 & \cellcolor{red!10}.689 \\
\cmidrule(l){2-16}
  & \multirow{2}{*}{50}
  & Wiki & \cellcolor{blue!10}.1000 & \cellcolor{blue!10}.600 & \cellcolor{blue!10}.820 & \cellcolor{blue!10}.972 & \cellcolor{blue!10}.340 & \cellcolor{blue!10}.636 & \cellcolor{blue!10}.727 & \cellcolor{blue!10}.872 & \cellcolor{blue!10}.749 & \cellcolor{blue!10}.674 & \cellcolor{blue!10}.567 & \cellcolor{blue!10}.622 & \cellcolor{blue!10}.688 \\
  &
  & Moir & \cellcolor{red!10}.1000 & \cellcolor{red!10}.670 & \cellcolor{red!10}.736 & \cellcolor{red!10}.959 & \cellcolor{red!10}.373 & \cellcolor{red!10}.660 & \cellcolor{red!10}.730 & \cellcolor{red!10}.868 & \cellcolor{red!10}.749 & \cellcolor{red!10}.683 & \cellcolor{red!10}.568 & \cellcolor{red!10}.622 & \cellcolor{red!10}.690 \\
\cmidrule(l){2-16}
  & \multirow{2}{*}{100}
  & Wiki & \cellcolor{blue!10}.1000 & \cellcolor{blue!10}.600 & \cellcolor{blue!10}.799 & \cellcolor{blue!10}.955 & \cellcolor{blue!10}.325 & \cellcolor{blue!10}.621 & \cellcolor{blue!10}.729 & \cellcolor{blue!10}.879 & \cellcolor{blue!10}.749 & \cellcolor{blue!10}.677 & \cellcolor{blue!10}.566 & \cellcolor{blue!10}.585 & \cellcolor{blue!10}.682 \\
  &
  & Moir & \cellcolor{red!10}.1000 & \cellcolor{red!10}.665 & \cellcolor{red!10}.665 & \cellcolor{red!10}.947 & \cellcolor{red!10}.362 & \cellcolor{red!10}.639 & \cellcolor{red!10}.729 & \cellcolor{red!10}.872 & \cellcolor{red!10}.748 & \cellcolor{red!10}.677 & \cellcolor{red!10}.570 & \cellcolor{red!10}.640 & \cellcolor{red!10}.694 \\
\cmidrule(l){2-16}
  & \multirow{2}{*}{200}
  & Wiki & \cellcolor{blue!10}.990 & \cellcolor{blue!10}.580 & \cellcolor{blue!10}.737 & \cellcolor{blue!10}.922 & \cellcolor{blue!10}.323 & \cellcolor{blue!10}.604 & \cellcolor{blue!10}.727 & \cellcolor{blue!10}.869 & \cellcolor{blue!10}.747 & \cellcolor{blue!10}.685 & \cellcolor{blue!10}.567 & \cellcolor{blue!10}.603 & \cellcolor{blue!10}.686 \\
  &
  & Moir & \cellcolor{red!10}.985 & \cellcolor{red!10}.632 & \cellcolor{red!10}.575 & \cellcolor{red!10}.908 & \cellcolor{red!10}.366 & \cellcolor{red!10}.612 & \cellcolor{red!10}.727 & \cellcolor{red!10}.873 & \cellcolor{red!10}.748 & \cellcolor{red!10}.679 & \cellcolor{red!10}.564 & \cellcolor{red!10}.615 & \cellcolor{red!10}.688 \\
\cmidrule(l){2-16}
  & \multirow{2}{*}{500}
  & Wiki & \cellcolor{blue!10}.986 & \cellcolor{blue!10}.536 & \cellcolor{blue!10}.598 & \cellcolor{blue!10}.918 & \cellcolor{blue!10}.323 & \cellcolor{blue!10}.572 & \cellcolor{blue!10}.724 & \cellcolor{blue!10}.865 & \cellcolor{blue!10}.744 & \cellcolor{blue!10}.678 & \cellcolor{blue!10}.562 & \cellcolor{blue!10}.573 & \cellcolor{blue!10}.676 \\
  &
  & Moir & \cellcolor{red!10}.979 & \cellcolor{red!10}.618 & \cellcolor{red!10}.460 & \cellcolor{red!10}.857 & \cellcolor{red!10}.354 & \cellcolor{red!10}.568 & \cellcolor{red!10}.724 & \cellcolor{red!10}.870 & \cellcolor{red!10}.747 & \cellcolor{red!10}.657 & \cellcolor{red!10}.558 & \cellcolor{red!10}.609 & \cellcolor{red!10}.680 \\
\cmidrule(l){2-16}
  & \multirow{2}{*}{1K}
  & Wiki & \cellcolor{blue!10}.985 & \cellcolor{blue!10}.553 & \cellcolor{blue!10}.477 & \cellcolor{blue!10}.894 & \cellcolor{blue!10}.377 & \cellcolor{blue!10}.575 & \cellcolor{blue!10}.722 & \cellcolor{blue!10}.871 & \cellcolor{blue!10}.737 & \cellcolor{blue!10}.665 & \cellcolor{blue!10}.548 & \cellcolor{blue!10}.597 & \cellcolor{blue!10}.675 \\
  &
  & Moir & \cellcolor{red!10}.976 & \cellcolor{red!10}.630 & \cellcolor{red!10}.381 & \cellcolor{red!10}.842 & \cellcolor{red!10}.378 & \cellcolor{red!10}.551 & \cellcolor{red!10}.723 & \cellcolor{red!10}.878 & \cellcolor{red!10}.743 & \cellcolor{red!10}.661 & \cellcolor{red!10}.563 & \cellcolor{red!10}.634 & \cellcolor{red!10}.687 \\
\cmidrule(l){2-16}
  & \multirow{2}{*}{2K}
  & Wiki & \cellcolor{blue!10}.968 & \cellcolor{blue!10}.489 & \cellcolor{blue!10}.377 & \cellcolor{blue!10}.876 & \cellcolor{blue!10}.388 & \cellcolor{blue!10}.529 & \cellcolor{blue!10}.702 & \cellcolor{blue!10}.853 & \cellcolor{blue!10}.728 & \cellcolor{blue!10}.640 & \cellcolor{blue!10}.537 & \cellcolor{blue!10}.585 & \cellcolor{blue!10}.659 \\
  &
  & Moir & \cellcolor{red!10}.942 & \cellcolor{red!10}.578 & \cellcolor{red!10}.299 & \cellcolor{red!10}.842 & \cellcolor{red!10}.406 & \cellcolor{red!10}.511 & \cellcolor{red!10}.720 & \cellcolor{red!10}.868 & \cellcolor{red!10}.739 & \cellcolor{red!10}.664 & \cellcolor{red!10}.552 & \cellcolor{red!10}.597 & \cellcolor{red!10}.675 \\
\cmidrule(l){2-16}
  & \multirow{2}{*}{5K}
  & Wiki & \cellcolor{blue!10}.943 & \cellcolor{blue!10}.452 & \cellcolor{blue!10}.288 & \cellcolor{blue!10}.851 & \cellcolor{blue!10}.396 & \cellcolor{blue!10}.479 & \cellcolor{blue!10}.653 & \cellcolor{blue!10}.741 & \cellcolor{blue!10}.699 & \cellcolor{blue!10}.628 & \cellcolor{blue!10}.433 & \cellcolor{blue!10}.353 & \cellcolor{blue!10}.543 \\
  &
  & Moir & \cellcolor{red!10}.912 & \cellcolor{red!10}.502 & \cellcolor{red!10}.226 & \cellcolor{red!10}.828 & \cellcolor{red!10}.415 & \cellcolor{red!10}.449 & \cellcolor{red!10}.708 & \cellcolor{red!10}.863 & \cellcolor{red!10}.728 & \cellcolor{red!10}.656 & \cellcolor{red!10}.511 & \cellcolor{red!10}.585 & \cellcolor{red!10}.657 \\
\bottomrule
\end{tabular}%
}
\end{table*}

\clearpage
\subsection{AlphaEdit Editing Results}
\label{app:results_alphaedit}

Table~\ref{tab:ae_batch} and Table~\ref{tab:ae_optimal_seq} report the detailed metrics for AlphaEdit under the batch (5K--20K edits) and sequential (100--5K edits) regimes, respectively.

\begin{table*}[h]
\centering
\caption{AlphaEdit batch editing with per-model optimal null-space threshold. $C_\text{Wiki}$ (\colorbox{blue!10}{blue}) vs $C_{\method}$ (\colorbox{red!10}{red}).}
\label{tab:ae_batch}
\resizebox{\textwidth}{!}{%
\begin{tabular}{cl l | ccccc | c | cccccc | c}
\toprule
& & & \multicolumn{5}{c|}{\textit{Edit Metrics}} & & \multicolumn{6}{c|}{\textit{Preservation Metrics}} & \\
\cmidrule(lr){4-8} \cmidrule(l){10-15}
Model & $N$ & $C$ & Eff & Gen & Spec & Flu & Port & HM & MMLU & GSM8K & HeSw & WiGr & ARC & HE & HM \\
\midrule\midrule
\multirow{6}{*}{OLMo2}
  & \multirow{2}{*}{5K}
  & Wiki & \cellcolor{blue!10}.458 & \cellcolor{blue!10}.139 & \cellcolor{blue!10}.773 & \cellcolor{blue!10}.963 & \cellcolor{blue!10}.105 & \cellcolor{blue!10}.236 & \cellcolor{blue!10}.586 & \cellcolor{blue!10}.754 & \cellcolor{blue!10}.829 & \cellcolor{blue!10}.712 & \cellcolor{blue!10}.575 & \cellcolor{blue!10}.390 & \cellcolor{blue!10}.603 \\
  &
  & Moir & \cellcolor{red!10}.462 & \cellcolor{red!10}.144 & \cellcolor{red!10}.742 & \cellcolor{red!10}.971 & \cellcolor{red!10}.109 & \cellcolor{red!10}.242 & \cellcolor{red!10}.584 & \cellcolor{red!10}.755 & \cellcolor{red!10}.831 & \cellcolor{red!10}.718 & \cellcolor{red!10}.583 & \cellcolor{red!10}.378 & \cellcolor{red!10}.600 \\
\cmidrule(l){2-16}
  & \multirow{2}{*}{10K}
  & Wiki & \cellcolor{blue!10}.360 & \cellcolor{blue!10}.109 & \cellcolor{blue!10}.748 & \cellcolor{blue!10}.972 & \cellcolor{blue!10}.084 & \cellcolor{blue!10}.191 & \cellcolor{blue!10}.584 & \cellcolor{blue!10}.761 & \cellcolor{blue!10}.827 & \cellcolor{blue!10}.718 & \cellcolor{blue!10}.592 & \cellcolor{blue!10}.402 & \cellcolor{blue!10}.612 \\
  &
  & Moir & \cellcolor{red!10}.364 & \cellcolor{red!10}.110 & \cellcolor{red!10}.718 & \cellcolor{red!10}.979 & \cellcolor{red!10}.087 & \cellcolor{red!10}.194 & \cellcolor{red!10}.585 & \cellcolor{red!10}.763 & \cellcolor{red!10}.830 & \cellcolor{red!10}.717 & \cellcolor{red!10}.578 & \cellcolor{red!10}.414 & \cellcolor{red!10}.614 \\
\cmidrule(l){2-16}
  & \multirow{2}{*}{20K}
  & Wiki & \cellcolor{blue!10}.249 & \cellcolor{blue!10}.086 & \cellcolor{blue!10}.720 & \cellcolor{blue!10}.970 & \cellcolor{blue!10}.061 & \cellcolor{blue!10}.145 & \cellcolor{blue!10}.581 & \cellcolor{blue!10}.745 & \cellcolor{blue!10}.824 & \cellcolor{blue!10}.726 & \cellcolor{blue!10}.575 & \cellcolor{blue!10}.365 & \cellcolor{blue!10}.592 \\
  &
  & Moir & \cellcolor{red!10}.254 & \cellcolor{red!10}.084 & \cellcolor{red!10}.688 & \cellcolor{red!10}.974 & \cellcolor{red!10}.065 & \cellcolor{red!10}.149 & \cellcolor{red!10}.584 & \cellcolor{red!10}.755 & \cellcolor{red!10}.825 & \cellcolor{red!10}.713 & \cellcolor{red!10}.595 & \cellcolor{red!10}.414 & \cellcolor{red!10}.616 \\
\midrule\midrule
\multirow{6}{*}{Llama3}
  & \multirow{2}{*}{5K}
  & Wiki & \cellcolor{blue!10}.520 & \cellcolor{blue!10}.253 & \cellcolor{blue!10}.788 & \cellcolor{blue!10}.988 & \cellcolor{blue!10}.139 & \cellcolor{blue!10}.326 & \cellcolor{blue!10}.674 & \cellcolor{blue!10}.732 & \cellcolor{blue!10}.792 & \cellcolor{blue!10}.735 & \cellcolor{blue!10}.553 & \cellcolor{blue!10}.591 & \cellcolor{blue!10}.669 \\
  &
  & Moir & \cellcolor{red!10}.461 & \cellcolor{red!10}.207 & \cellcolor{red!10}.787 & \cellcolor{red!10}.991 & \cellcolor{red!10}.123 & \cellcolor{red!10}.287 & \cellcolor{red!10}.678 & \cellcolor{red!10}.714 & \cellcolor{red!10}.792 & \cellcolor{red!10}.734 & \cellcolor{red!10}.546 & \cellcolor{red!10}.579 & \cellcolor{red!10}.662 \\
\cmidrule(l){2-16}
  & \multirow{2}{*}{10K}
  & Wiki & \cellcolor{blue!10}.459 & \cellcolor{blue!10}.234 & \cellcolor{blue!10}.741 & \cellcolor{blue!10}.987 & \cellcolor{blue!10}.125 & \cellcolor{blue!10}.297 & \cellcolor{blue!10}.667 & \cellcolor{blue!10}.744 & \cellcolor{blue!10}.789 & \cellcolor{blue!10}.737 & \cellcolor{blue!10}.560 & \cellcolor{blue!10}.579 & \cellcolor{blue!10}.668 \\
  &
  & Moir & \cellcolor{red!10}.415 & \cellcolor{red!10}.188 & \cellcolor{red!10}.730 & \cellcolor{red!10}.989 & \cellcolor{red!10}.113 & \cellcolor{red!10}.264 & \cellcolor{red!10}.673 & \cellcolor{red!10}.708 & \cellcolor{red!10}.790 & \cellcolor{red!10}.734 & \cellcolor{red!10}.554 & \cellcolor{red!10}.615 & \cellcolor{red!10}.670 \\
\cmidrule(l){2-16}
  & \multirow{2}{*}{20K}
  & Wiki & \cellcolor{blue!10}.386 & \cellcolor{blue!10}.178 & \cellcolor{blue!10}.686 & \cellcolor{blue!10}.986 & \cellcolor{blue!10}.110 & \cellcolor{blue!10}.253 & \cellcolor{blue!10}.657 & \cellcolor{blue!10}.718 & \cellcolor{blue!10}.789 & \cellcolor{blue!10}.729 & \cellcolor{blue!10}.549 & \cellcolor{blue!10}.573 & \cellcolor{blue!10}.658 \\
  &
  & Moir & \cellcolor{red!10}.319 & \cellcolor{red!10}.143 & \cellcolor{red!10}.684 & \cellcolor{red!10}.990 & \cellcolor{red!10}.092 & \cellcolor{red!10}.213 & \cellcolor{red!10}.667 & \cellcolor{red!10}.716 & \cellcolor{red!10}.790 & \cellcolor{red!10}.739 & \cellcolor{red!10}.552 & \cellcolor{red!10}.622 & \cellcolor{red!10}.671 \\
\midrule\midrule
\multirow{6}{*}{Qwen3}
  & \multirow{2}{*}{5K}
  & Wiki & \cellcolor{blue!10}.809 & \cellcolor{blue!10}.228 & \cellcolor{blue!10}.631 & \cellcolor{blue!10}.935 & \cellcolor{blue!10}.199 & \cellcolor{blue!10}.376 & \cellcolor{blue!10}.534 & \cellcolor{blue!10}.124 & \cellcolor{blue!10}.638 & \cellcolor{blue!10}.579 & \cellcolor{blue!10}.356 & \cellcolor{blue!10}.073 & \cellcolor{blue!10}.202 \\
  &
  & Moir & \cellcolor{red!10}.728 & \cellcolor{red!10}.219 & \cellcolor{red!10}.625 & \cellcolor{red!10}.932 & \cellcolor{red!10}.180 & \cellcolor{red!10}.353 & \cellcolor{red!10}.653 & \cellcolor{red!10}.746 & \cellcolor{red!10}.692 & \cellcolor{red!10}.629 & \cellcolor{red!10}.463 & \cellcolor{red!10}.372 & \cellcolor{red!10}.558 \\
\cmidrule(l){2-16}
  & \multirow{2}{*}{10K}
  & Wiki & \cellcolor{blue!10}.620 & \cellcolor{blue!10}.168 & \cellcolor{blue!10}.591 & \cellcolor{blue!10}.934 & \cellcolor{blue!10}.149 & \cellcolor{blue!10}.294 & \cellcolor{blue!10}.489 & \cellcolor{blue!10}.026 & \cellcolor{blue!10}.617 & \cellcolor{blue!10}.577 & \cellcolor{blue!10}.352 & \cellcolor{blue!10}.091 & \cellcolor{blue!10}.105 \\
  &
  & Moir & \cellcolor{red!10}.474 & \cellcolor{red!10}.141 & \cellcolor{red!10}.619 & \cellcolor{red!10}.940 & \cellcolor{red!10}.117 & \cellcolor{red!10}.245 & \cellcolor{red!10}.661 & \cellcolor{red!10}.759 & \cellcolor{red!10}.692 & \cellcolor{red!10}.636 & \cellcolor{red!10}.484 & \cellcolor{red!10}.500 & \cellcolor{red!10}.605 \\
\cmidrule(l){2-16}
  & \multirow{2}{*}{20K}
  & Wiki & \cellcolor{blue!10}.388 & \cellcolor{blue!10}.115 & \cellcolor{blue!10}.588 & \cellcolor{blue!10}.937 & \cellcolor{blue!10}.099 & \cellcolor{blue!10}.207 & \cellcolor{blue!10}.535 & \cellcolor{blue!10}.109 & \cellcolor{blue!10}.636 & \cellcolor{blue!10}.583 & \cellcolor{blue!10}.353 & \cellcolor{blue!10}.067 & \cellcolor{blue!10}.187 \\
  &
  & Moir & \cellcolor{red!10}.244 & \cellcolor{red!10}.090 & \cellcolor{red!10}.635 & \cellcolor{red!10}.957 & \cellcolor{red!10}.065 & \cellcolor{red!10}.150 & \cellcolor{red!10}.677 & \cellcolor{red!10}.799 & \cellcolor{red!10}.709 & \cellcolor{red!10}.644 & \cellcolor{red!10}.514 & \cellcolor{red!10}.500 & \cellcolor{red!10}.623 \\
\bottomrule
\end{tabular}%
}
\end{table*}

\clearpage
\begin{table*}[h]
\centering
\caption{AlphaEdit sequential editing with per-model optimal null-space threshold. $C_\text{Wiki}$ (\colorbox{blue!10}{blue}) vs $C_{\method}$ (\colorbox{red!10}{red}).}
\label{tab:ae_optimal_seq}
\resizebox{\textwidth}{!}{%
\begin{tabular}{cl l | ccccc | c | cccccc | c}
\toprule
& & & \multicolumn{5}{c|}{\textit{Edit Metrics}} & & \multicolumn{6}{c|}{\textit{Preservation Metrics}} & \\
\cmidrule(lr){4-8} \cmidrule(l){10-15}
Model & $N$ & $C$ & Eff & Gen & Spec & Flu & Port & HM & MMLU & GSM8K & HeSw & WiGr & ARC & HE & HM \\
\midrule\midrule
\multirow{10}{*}{OLMo2}
  & \multirow{2}{*}{100}
  & Wiki & \cellcolor{blue!10}.555 & \cellcolor{blue!10}.125 & \cellcolor{blue!10}.941 & \cellcolor{blue!10}.956 & \cellcolor{blue!10}.145 & \cellcolor{blue!10}.265 & \cellcolor{blue!10}.590 & \cellcolor{blue!10}.762 & \cellcolor{blue!10}.833 & \cellcolor{blue!10}.714 & \cellcolor{blue!10}.586 & \cellcolor{blue!10}.414 & \cellcolor{blue!10}.617 \\
  &
  & Moir & \cellcolor{red!10}.785 & \cellcolor{red!10}.255 & \cellcolor{red!10}.891 & \cellcolor{red!10}.964 & \cellcolor{red!10}.199 & \cellcolor{red!10}.403 & \cellcolor{red!10}.590 & \cellcolor{red!10}.768 & \cellcolor{red!10}.833 & \cellcolor{red!10}.715 & \cellcolor{red!10}.587 & \cellcolor{red!10}.420 & \cellcolor{red!10}.620 \\
\cmidrule(l){2-16}
  & \multirow{2}{*}{500}
  & Wiki & \cellcolor{blue!10}.665 & \cellcolor{blue!10}.240 & \cellcolor{blue!10}.858 & \cellcolor{blue!10}.968 & \cellcolor{blue!10}.156 & \cellcolor{blue!10}.350 & \cellcolor{blue!10}.590 & \cellcolor{blue!10}.773 & \cellcolor{blue!10}.832 & \cellcolor{blue!10}.711 & \cellcolor{blue!10}.586 & \cellcolor{blue!10}.390 & \cellcolor{blue!10}.608 \\
  &
  & Moir & \cellcolor{red!10}.838 & \cellcolor{red!10}.313 & \cellcolor{red!10}.755 & \cellcolor{red!10}.963 & \cellcolor{red!10}.207 & \cellcolor{red!10}.432 & \cellcolor{red!10}.589 & \cellcolor{red!10}.760 & \cellcolor{red!10}.833 & \cellcolor{red!10}.715 & \cellcolor{red!10}.587 & \cellcolor{red!10}.396 & \cellcolor{red!10}.610 \\
\cmidrule(l){2-16}
  & \multirow{2}{*}{1K}
  & Wiki & \cellcolor{blue!10}.718 & \cellcolor{blue!10}.303 & \cellcolor{blue!10}.781 & \cellcolor{blue!10}.964 & \cellcolor{blue!10}.185 & \cellcolor{blue!10}.403 & \cellcolor{blue!10}.590 & \cellcolor{blue!10}.760 & \cellcolor{blue!10}.831 & \cellcolor{blue!10}.723 & \cellcolor{blue!10}.585 & \cellcolor{blue!10}.384 & \cellcolor{blue!10}.605 \\
  &
  & Moir & \cellcolor{red!10}.855 & \cellcolor{red!10}.351 & \cellcolor{red!10}.652 & \cellcolor{red!10}.965 & \cellcolor{red!10}.236 & \cellcolor{red!10}.462 & \cellcolor{red!10}.591 & \cellcolor{red!10}.765 & \cellcolor{red!10}.832 & \cellcolor{red!10}.720 & \cellcolor{red!10}.587 & \cellcolor{red!10}.378 & \cellcolor{red!10}.603 \\
\cmidrule(l){2-16}
  & \multirow{2}{*}{2K}
  & Wiki & \cellcolor{blue!10}.723 & \cellcolor{blue!10}.317 & \cellcolor{blue!10}.714 & \cellcolor{blue!10}.953 & \cellcolor{blue!10}.187 & \cellcolor{blue!10}.405 & \cellcolor{blue!10}.588 & \cellcolor{blue!10}.748 & \cellcolor{blue!10}.831 & \cellcolor{blue!10}.719 & \cellcolor{blue!10}.577 & \cellcolor{blue!10}.372 & \cellcolor{blue!10}.597 \\
  &
  & Moir & \cellcolor{red!10}.877 & \cellcolor{red!10}.372 & \cellcolor{red!10}.553 & \cellcolor{red!10}.960 & \cellcolor{red!10}.239 & \cellcolor{red!10}.460 & \cellcolor{red!10}.590 & \cellcolor{red!10}.764 & \cellcolor{red!10}.832 & \cellcolor{red!10}.717 & \cellcolor{red!10}.581 & \cellcolor{red!10}.402 & \cellcolor{red!10}.612 \\
\cmidrule(l){2-16}
  & \multirow{2}{*}{5K}
  & Wiki & \cellcolor{blue!10}.725 & \cellcolor{blue!10}.324 & \cellcolor{blue!10}.605 & \cellcolor{blue!10}.962 & \cellcolor{blue!10}.193 & \cellcolor{blue!10}.405 & \cellcolor{blue!10}.587 & \cellcolor{blue!10}.739 & \cellcolor{blue!10}.826 & \cellcolor{blue!10}.722 & \cellcolor{blue!10}.570 & \cellcolor{blue!10}.396 & \cellcolor{blue!10}.604 \\
  &
  & Moir & \cellcolor{red!10}.852 & \cellcolor{red!10}.376 & \cellcolor{red!10}.451 & \cellcolor{red!10}.970 & \cellcolor{red!10}.240 & \cellcolor{red!10}.445 & \cellcolor{red!10}.587 & \cellcolor{red!10}.756 & \cellcolor{red!10}.829 & \cellcolor{red!10}.716 & \cellcolor{red!10}.577 & \cellcolor{red!10}.396 & \cellcolor{red!10}.607 \\
\midrule\midrule
\multirow{10}{*}{Llama3}
  & \multirow{2}{*}{100}
  & Wiki & \cellcolor{blue!10}.465 & \cellcolor{blue!10}.230 & \cellcolor{blue!10}.961 & \cellcolor{blue!10}.990 & \cellcolor{blue!10}.146 & \cellcolor{blue!10}.325 & \cellcolor{blue!10}.684 & \cellcolor{blue!10}.721 & \cellcolor{blue!10}.794 & \cellcolor{blue!10}.738 & \cellcolor{blue!10}.561 & \cellcolor{blue!10}.561 & \cellcolor{blue!10}.665 \\
  &
  & Moir & \cellcolor{red!10}.465 & \cellcolor{red!10}.175 & \cellcolor{red!10}.947 & \cellcolor{red!10}.997 & \cellcolor{red!10}.129 & \cellcolor{red!10}.283 & \cellcolor{red!10}.684 & \cellcolor{red!10}.710 & \cellcolor{red!10}.795 & \cellcolor{red!10}.738 & \cellcolor{red!10}.558 & \cellcolor{red!10}.597 & \cellcolor{red!10}.670 \\
\cmidrule(l){2-16}
  & \multirow{2}{*}{500}
  & Wiki & \cellcolor{blue!10}.632 & \cellcolor{blue!10}.302 & \cellcolor{blue!10}.902 & \cellcolor{blue!10}.993 & \cellcolor{blue!10}.166 & \cellcolor{blue!10}.384 & \cellcolor{blue!10}.683 & \cellcolor{blue!10}.726 & \cellcolor{blue!10}.795 & \cellcolor{blue!10}.743 & \cellcolor{blue!10}.562 & \cellcolor{blue!10}.603 & \cellcolor{blue!10}.675 \\
  &
  & Moir & \cellcolor{red!10}.604 & \cellcolor{red!10}.256 & \cellcolor{red!10}.873 & \cellcolor{red!10}.993 & \cellcolor{red!10}.156 & \cellcolor{red!10}.354 & \cellcolor{red!10}.684 & \cellcolor{red!10}.711 & \cellcolor{red!10}.794 & \cellcolor{red!10}.737 & \cellcolor{red!10}.558 & \cellcolor{red!10}.597 & \cellcolor{red!10}.670 \\
\cmidrule(l){2-16}
  & \multirow{2}{*}{1K}
  & Wiki & \cellcolor{blue!10}.690 & \cellcolor{blue!10}.349 & \cellcolor{blue!10}.838 & \cellcolor{blue!10}.989 & \cellcolor{blue!10}.191 & \cellcolor{blue!10}.426 & \cellcolor{blue!10}.682 & \cellcolor{blue!10}.723 & \cellcolor{blue!10}.794 & \cellcolor{blue!10}.741 & \cellcolor{blue!10}.560 & \cellcolor{blue!10}.615 & \cellcolor{blue!10}.676 \\
  &
  & Moir & \cellcolor{red!10}.690 & \cellcolor{red!10}.308 & \cellcolor{red!10}.777 & \cellcolor{red!10}.986 & \cellcolor{red!10}.202 & \cellcolor{red!10}.418 & \cellcolor{red!10}.684 & \cellcolor{red!10}.694 & \cellcolor{red!10}.794 & \cellcolor{red!10}.734 & \cellcolor{red!10}.554 & \cellcolor{red!10}.603 & \cellcolor{red!10}.668 \\
\cmidrule(l){2-16}
  & \multirow{2}{*}{2K}
  & Wiki & \cellcolor{blue!10}.757 & \cellcolor{blue!10}.416 & \cellcolor{blue!10}.749 & \cellcolor{blue!10}.978 & \cellcolor{blue!10}.210 & \cellcolor{blue!10}.461 & \cellcolor{blue!10}.677 & \cellcolor{blue!10}.725 & \cellcolor{blue!10}.793 & \cellcolor{blue!10}.734 & \cellcolor{blue!10}.559 & \cellcolor{blue!10}.603 & \cellcolor{blue!10}.672 \\
  &
  & Moir & \cellcolor{red!10}.751 & \cellcolor{red!10}.384 & \cellcolor{red!10}.650 & \cellcolor{red!10}.975 & \cellcolor{red!10}.219 & \cellcolor{red!10}.452 & \cellcolor{red!10}.680 & \cellcolor{red!10}.711 & \cellcolor{red!10}.794 & \cellcolor{red!10}.729 & \cellcolor{red!10}.547 & \cellcolor{red!10}.573 & \cellcolor{red!10}.661 \\
\cmidrule(l){2-16}
  & \multirow{2}{*}{5K}
  & Wiki & \cellcolor{blue!10}.829 & \cellcolor{blue!10}.501 & \cellcolor{blue!10}.536 & \cellcolor{blue!10}.955 & \cellcolor{blue!10}.255 & \cellcolor{blue!10}.499 & \cellcolor{blue!10}.666 & \cellcolor{blue!10}.664 & \cellcolor{blue!10}.790 & \cellcolor{blue!10}.731 & \cellcolor{blue!10}.547 & \cellcolor{blue!10}.573 & \cellcolor{blue!10}.651 \\
  &
  & Moir & \cellcolor{red!10}.784 & \cellcolor{red!10}.470 & \cellcolor{red!10}.433 & \cellcolor{red!10}.947 & \cellcolor{red!10}.283 & \cellcolor{red!10}.485 & \cellcolor{red!10}.670 & \cellcolor{red!10}.689 & \cellcolor{red!10}.789 & \cellcolor{red!10}.727 & \cellcolor{red!10}.550 & \cellcolor{red!10}.597 & \cellcolor{red!10}.661 \\
\midrule\midrule
\multirow{10}{*}{Qwen3}
  & \multirow{2}{*}{100}
  & Wiki & \cellcolor{blue!10}.1000 & \cellcolor{blue!10}.470 & \cellcolor{blue!10}.889 & \cellcolor{blue!10}.960 & \cellcolor{blue!10}.290 & \cellcolor{blue!10}.572 & \cellcolor{blue!10}.725 & \cellcolor{blue!10}.875 & \cellcolor{blue!10}.747 & \cellcolor{blue!10}.675 & \cellcolor{blue!10}.557 & \cellcolor{blue!10}.670 & \cellcolor{blue!10}.695 \\
  &
  & Moir & \cellcolor{red!10}.1000 & \cellcolor{red!10}.550 & \cellcolor{red!10}.840 & \cellcolor{red!10}.967 & \cellcolor{red!10}.307 & \cellcolor{red!10}.602 & \cellcolor{red!10}.728 & \cellcolor{red!10}.875 & \cellcolor{red!10}.748 & \cellcolor{red!10}.689 & \cellcolor{red!10}.554 & \cellcolor{red!10}.591 & \cellcolor{red!10}.682 \\
\cmidrule(l){2-16}
  & \multirow{2}{*}{500}
  & Wiki & \cellcolor{blue!10}.988 & \cellcolor{blue!10}.381 & \cellcolor{blue!10}.810 & \cellcolor{blue!10}.952 & \cellcolor{blue!10}.244 & \cellcolor{blue!10}.499 & \cellcolor{blue!10}.712 & \cellcolor{blue!10}.859 & \cellcolor{blue!10}.735 & \cellcolor{blue!10}.683 & \cellcolor{blue!10}.514 & \cellcolor{blue!10}.597 & \cellcolor{blue!10}.666 \\
  &
  & Moir & \cellcolor{red!10}.994 & \cellcolor{red!10}.482 & \cellcolor{red!10}.676 & \cellcolor{red!10}.912 & \cellcolor{red!10}.278 & \cellcolor{red!10}.540 & \cellcolor{red!10}.719 & \cellcolor{red!10}.864 & \cellcolor{red!10}.738 & \cellcolor{red!10}.691 & \cellcolor{red!10}.552 & \cellcolor{red!10}.585 & \cellcolor{red!10}.676 \\
\cmidrule(l){2-16}
  & \multirow{2}{*}{1K}
  & Wiki & \cellcolor{blue!10}.981 & \cellcolor{blue!10}.388 & \cellcolor{blue!10}.750 & \cellcolor{blue!10}.938 & \cellcolor{blue!10}.242 & \cellcolor{blue!10}.494 & \cellcolor{blue!10}.691 & \cellcolor{blue!10}.821 & \cellcolor{blue!10}.722 & \cellcolor{blue!10}.660 & \cellcolor{blue!10}.501 & \cellcolor{blue!10}.457 & \cellcolor{blue!10}.615 \\
  &
  & Moir & \cellcolor{red!10}.987 & \cellcolor{red!10}.495 & \cellcolor{red!10}.595 & \cellcolor{red!10}.898 & \cellcolor{red!10}.313 & \cellcolor{red!10}.554 & \cellcolor{red!10}.710 & \cellcolor{red!10}.856 & \cellcolor{red!10}.728 & \cellcolor{red!10}.663 & \cellcolor{red!10}.536 & \cellcolor{red!10}.567 & \cellcolor{red!10}.660 \\
\cmidrule(l){2-16}
  & \multirow{2}{*}{2K}
  & Wiki & \cellcolor{blue!10}.956 & \cellcolor{blue!10}.327 & \cellcolor{blue!10}.679 & \cellcolor{blue!10}.905 & \cellcolor{blue!10}.232 & \cellcolor{blue!10}.455 & \cellcolor{blue!10}.637 & \cellcolor{blue!10}.554 & \cellcolor{blue!10}.689 & \cellcolor{blue!10}.644 & \cellcolor{blue!10}.440 & \cellcolor{blue!10}.304 & \cellcolor{blue!10}.503 \\
  &
  & Moir & \cellcolor{red!10}.947 & \cellcolor{red!10}.431 & \cellcolor{red!10}.498 & \cellcolor{red!10}.860 & \cellcolor{red!10}.306 & \cellcolor{red!10}.510 & \cellcolor{red!10}.684 & \cellcolor{red!10}.806 & \cellcolor{red!10}.705 & \cellcolor{red!10}.634 & \cellcolor{red!10}.492 & \cellcolor{red!10}.512 & \cellcolor{red!10}.620 \\
\cmidrule(l){2-16}
  & \multirow{2}{*}{5K}
  & Wiki & \cellcolor{blue!10}.814 & \cellcolor{blue!10}.212 & \cellcolor{blue!10}.573 & \cellcolor{blue!10}.847 & \cellcolor{blue!10}.199 & \cellcolor{blue!10}.360 & \cellcolor{blue!10}.402 & \cellcolor{blue!10}.002 & \cellcolor{blue!10}.595 & \cellcolor{blue!10}.543 & \cellcolor{blue!10}.320 & \cellcolor{blue!10}.024 & \cellcolor{blue!10}.012 \\
  &
  & Moir & \cellcolor{red!10}.853 & \cellcolor{red!10}.306 & \cellcolor{red!10}.369 & \cellcolor{red!10}.780 & \cellcolor{red!10}.271 & \cellcolor{red!10}.412 & \cellcolor{red!10}.609 & \cellcolor{red!10}.565 & \cellcolor{red!10}.647 & \cellcolor{red!10}.580 & \cellcolor{red!10}.410 & \cellcolor{red!10}.237 & \cellcolor{red!10}.450 \\
\bottomrule
\end{tabular}%
}
\end{table*}

\clearpage
\paragraph{AlphaEdit Graphical Summary}
Figure~\ref{fig:hm_alphaedit} and Figure~\ref{fig:pres_alphaedit} summarize the aggregate and per-task preservation trajectories for AlphaEdit. 
\begin{figure*}[h]
    \centering
    \vspace{-15pt} 
    \includegraphics[width=0.6\textwidth]{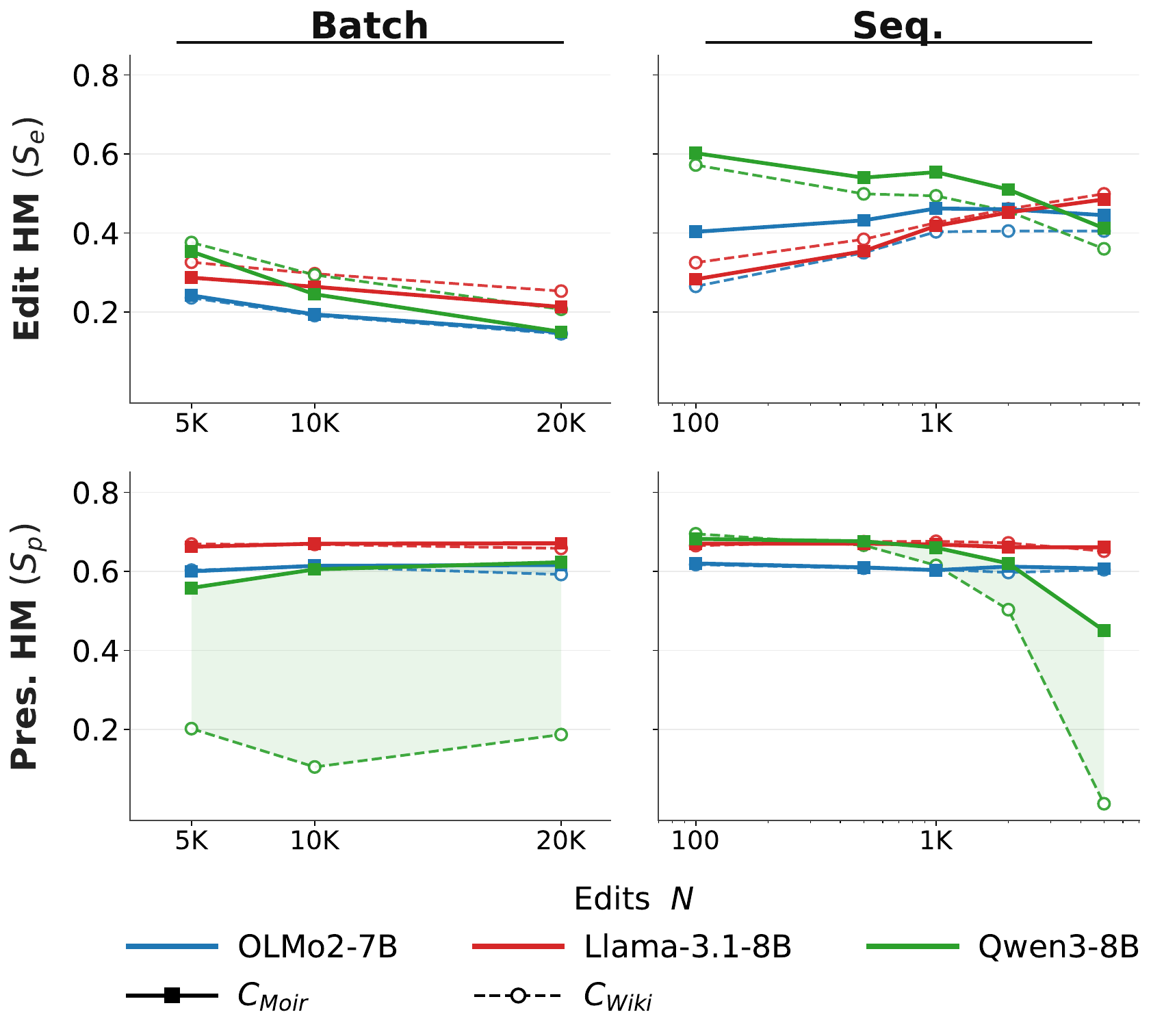}
    \vspace{-5pt} 
    \caption{\textbf{AlphaEdit: aggregate Edit~HM and Pres.~HM.} Pres.~HM is flat for OLMo2
and Llama3, but Qwen3 with WikiText $C_\text{Wiki}$ holds at $\sim\!0.20$ vs
$\sim\!0.60$ with $C_\method$---a gap that persists across all
edit budgets and exposes AlphaEdit's covariance dependence.}
    \label{fig:hm_alphaedit}
    \vspace{-10pt} 
\end{figure*}
\begin{figure*}[h]
    \centering
    \includegraphics[width=\textwidth]{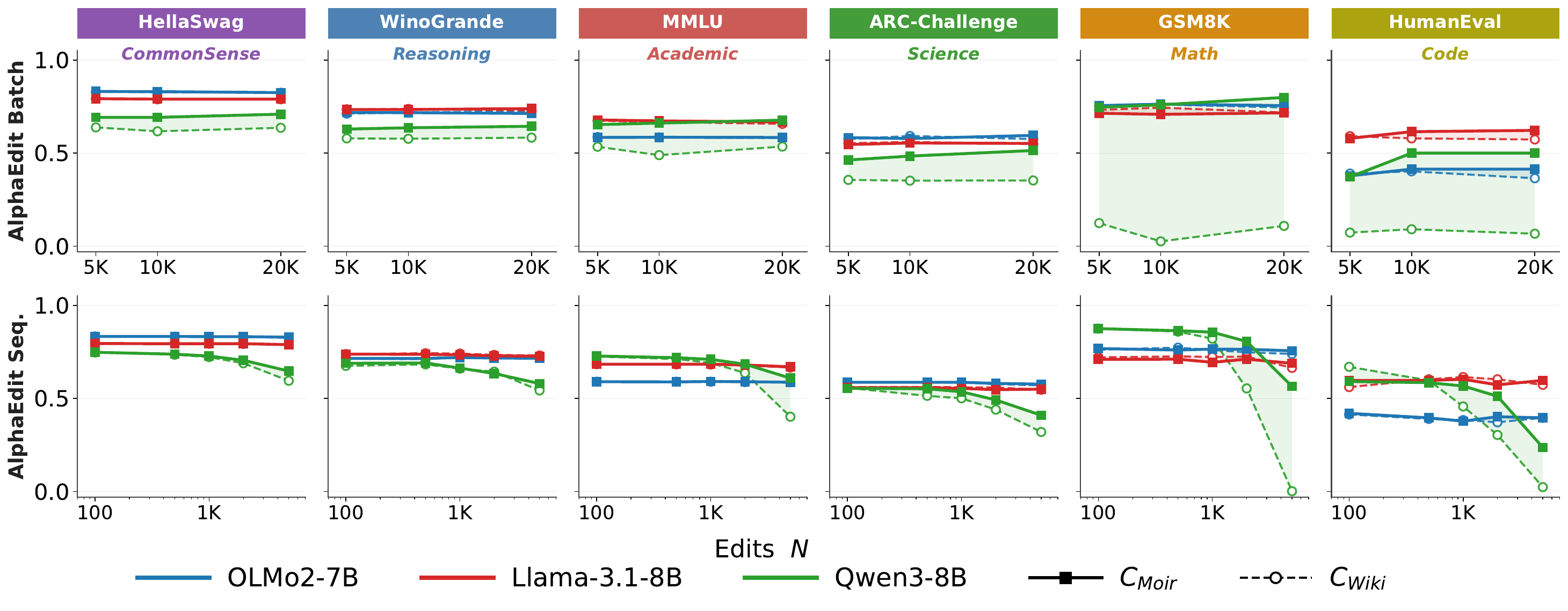}
    \caption{\textbf{Per-task preservation under AlphaEdit.}
 AlphaEdit is stable on OLMo2 and Llama3 regardless of $C_0$, but
on Qwen3, $C_\text{Wiki}$ collapses GSM8K and HumanEval to near zero
while $C_\method$ holds them at near-baseline levels. Batch $N{=}5\text{K}{-}20\text{K}$ and sequential $N{=}100{-}5\text{K}$.}
    \label{fig:pres_alphaedit}
\end{figure*}
\clearpage



\newpage

\end{document}